\documentclass[preprint,3p,12pt]{elsarticle}

\usepackage{color}
\usepackage{amsmath}
\usepackage{amsfonts}
\usepackage{amssymb}
\usepackage{latexsym}
\usepackage{latexsym}
\usepackage{array}
\usepackage{amsthm}
\usepackage{graphics}
\usepackage{comment}
\usepackage{epsf,graphicx,subfigure}
\usepackage{epsfig}
\usepackage{epstopdf}
\usepackage{tikz}
\usetikzlibrary{arrows,positioning,shapes.geometric}

\usepackage{longtable}
\usepackage{rotating}

\newtheorem{Theorem}{Theorem}

\newtheorem{Lemma}{Lemma}
\newtheorem{Corollary}{Corollary}
\newtheorem{Assumption}{Assumption}
\newtheorem{Definition}{Definition}
\newtheorem{Remark}{Remark}

\begin{document}

\begin{frontmatter}





\title{Central and Non-central Limit Theorems arising from the Scattering Transform and its Neural Activation Generalization}

\author[GRLiu]{Gi-Ren Liu}

\author[YCSheu]{Yuan-Chung Sheu}
\author[HTWu,HTWu2]{Hau-Tieng Wu\corref{mycorrespondingauthor}}
\cortext[mycorrespondingauthor]{Hau-Tieng Wu}
\ead{hauwu@math.duke.edu}

\address[GRLiu]{Department of Mathematics, National Chen-Kung University, Tainan, Taiwan}
\address[YCSheu]{Department of Applied Mathematics, National Chiao Tung University, Hsinchu, Taiwan}
\address[HTWu]{Department of Mathematics and Department of Statistical Science, Duke University, Durham, NC, USA}
\address[HTWu2]{Mathematics Division, National Center for Theoretical Sciences, Taipei, Taiwan}

\begin{abstract}
Motivated by analyzing complicated and non-stationary time series, we study a generalization of the scattering transform (ST) that includes broad neural activation functions, which is called {\em neural activation ST} (NAST).
On the whole, NAST is a transform that comprises a sequence of ``neural processing units'', each of which applies a high pass filter to the input from the previous layer followed by a composition with a nonlinear function as the output to the next neuron. Here, the nonlinear function models how a neuron gets excited by the input signal.
In addition to showing properties like non-expansion, horizontal translational invariability and insensitivity to local deformation,
the statistical properties of the second order NAST of a Gaussian process with various dependence and (non-)stationarity structure
and its interaction with the chosen high pass filters and activation functions are explored and central limit theorem (CLT) and non-CLT results are provided. Numerical simulations are also provided.
The results explain how NAST processes complicated and non-stationary time series, and pave a way towards statistical inference based on NAST under the non-null case.
\end{abstract}


\end{frontmatter}


\section{Introduction}\label{sec:introduction}

The scattering transform (ST) \cite{mallat2012group,anden2014deep} is motivated by establishing a mathematical foundation of the convolutional neural network (CNN), particularly the feature extraction part.
ST has been applied to various applications since its appearance. It was applied to audio classification  \cite{anden2011multiscale}, fetal heart rate analysis \cite{chudavcek2013scattering}, music genre classification \cite{chen2013music}, heart sound classification \cite{li2019heart}, sleep stage classification \cite{wu2014assess}, hyperspectral image classification \cite{tang2014hyperspectral}, quantum chemical energies analysis \cite{hirn2017wavelet}, etc.
We mention that in our recent work, we have shown that ST can efficiently quantify nonlinear broad spectral features of EEG. Indeed, with these features, we can develop an accurate automatic sleep stage annotation system for clinical usage \cite{liu2020diffuse,liu2020save}. Recently, this automatic annotation system was further applied to establish an artificial intelligent system to help evaluate and improve sleep center quality \cite{liu2020hospitals}.

By and large, the $K$-th order ST depends on a sequential interlacing convolution and the modulus operator (taking absolute value). In practice, to extract stable features,  a suitable low pass filter is applied to the $K$-th order ST to obtain the {\em $K$-th order ST coefficients}. Here, the modulus operator captures the neural activation. The $K$-th order ST captures the feature extraction and the activation of the neuron in the $K$-th level CNN, which builds informative representations for the signals of interest.
Its theoretical properties, including {\em non-expansive on $L^2(\mathbb{R}^n)$} and {\em translation invariance} and {\em local diffeomorphisms}, have been studied and reported in, for example, \cite{mallat2012group,anden2014deep}. Recently it has been actively extended to Gabor systems \cite{czaja2019analysis,czaja2020rotationally} or to the non-Euclidean setup, like graph \cite{gama2019diffusion,gama2019stability,zou2020graph}, and to the deep Haar scattering \cite{cheng2016deep}.
It is natural to ask if we could better model CNN via generalizing the modulus operator in ST, since activation functions play a crucial role in discriminative capabilities of artificial neural networks.
In practice, different activation functions other than the modulus operator could be selected for each hidden layer or each neuron; see, for example \cite{nanni2020stochastic,ertuugrul2018novel}. From this standpoint, we may obtain a better understanding of CNN by generalizing ST to include more general activation functions.
Take a mother wavelet $\psi$, and denote $\psi_{j}(x):=2^{-j}\psi(2^{-j}x)$ for $j\in \mathbb{Z}$. For $K\in \mathbb{N}$,
the $K$-th order {\em neural activation ST} (NAST) we consider in this paper is
\[
U^{A_{1},A_{2},\ldots,A_K}[j_{1},j_{2},\ldots,j_K]X(t):=A_K((\ldots A_{2}(A_{1}(X\star\psi_{j_{1}})\star\psi_{j_{2}})\ldots)\star \psi_{j_K}(t)),
\]
where $t\in\mathbb{R}$, $X$ is any reasonable input function or random process, $j_{k}$ is the scale parameter of the continuous wavelet transform (CWT) in the $k$-th layers, and $A_{k}$ is a nonlinear function modeling various neural activations; for example,  Rectifier Linear Unit (ReLU), arctangent, and several others \cite{pedamonti2018comparison}. We may apply a suitable low pass filter to the $K$-th order NAST and obtain the {\em $K$-th order NAST coefficients}.
Here, when $A_k(x)=|x|$, we recover the traditional ST.
In \cite{wiatowski2017mathematical}, such idea has been implemented with several theoretical properties explored.
Note that in \cite{wiatowski2017mathematical}, the activation function is followed by an extra pooling layer, which we call the {\em pooling NAST} and it is not considered in this paper.
The authors in \cite{wiatowski2017mathematical} showed that the pooling NAST exhibits {\em vertical translation invariance} in the sense that the features becoming more translation-invariant with the increasing network depth. Note that the translation invariance of ST shown in \cite{mallat2012group,anden2014deep} is the {\em horizontal} one, in the sense of the features becoming more translation invariant if the wavelet scale parameter $J$ increases. For input functions satisfying the {\em deformation-insensitive property} \cite[Definition 5]{wiatowski2017mathematical},
a deformation sensitivity bound is also derived for the pooling NAST. However, the authors of \cite{wiatowski2017mathematical} only focused on deterministic functions.

While existing theoretical results have covered a wide range of signals, in practice we may encounter more challenging ones. It is well known that electroencephalogram (EEG), and many other biomedical time series have complicated statistical behavior. For example, it was discovered that there is a long-range temporal correlations and scaling behavior in human brain oscillations \cite{linkenkaer2001long}, heart rate possessed a strong $1/f$ component over a wide frequency \cite{kobayashi19821}, or oven a multifractal structure \cite{ivanov1999multifractality},
and the human coordination is a long-memory process \cite{chen1997long}. More importantly, these signals are usually non-stationary. Due to its scientific and clinical impacts, a lot of efforts have been invested in analyzing such signals. However, it remains a signal processing challenge. A natural question to ask is how does NAST behave on such complicated time series.
%
%

In this paper, 
we ask how NAST, and hence ST as a special case, behaves on random processes of various structures. There have been some work discussing the behavior of the traditional ST on random process.
In \cite{bruna2015intermittent}, the first- and second-order ST coefficients are used to characterize more detailed properties of random processes with stationary increments,
including the intermittency and self-similarity, where the Poisson processes, fractional Brownian motions were analyzed.
The limiting distributions of the second-order ST of the fractional Brownian motions
have been proved to be folded Gaussian when the vanishing moment of the wavelet is greater than or equal to two. However, the distribution of a random process after applying ST, particularly the covariance structure of the limiting random process, was not discussed. For the statistical inference purpose, it is important to obtain such structure, or at least find a way to approximate it.
The empirical results for L$\acute{\textup{e}}$vy processes and multifractal cascade processes were also provided.
%
Moreover, the general long-range dependent Gaussian process is not discussed in \cite{bruna2015intermittent}, and the impact of the chosen mother wavelet is not yet explored. To our knowledge, the statistical behavior of NAST on random processes is left open. More theoretical results can be found in Section \ref{sec:Existing}.

\subsection{Our contribution}

%



%


We provide several theoretical results about NAST.
First, we show that if the NAST coefficients are obtained with a proper low pass filter, specifically a companion father wavelet, the NAST coefficients exhibit {\em horizontal} translation invariance, which is different from the vertical translation invariance shown in \cite{wiatowski2017mathematical}. 
We also provide a deformation sensitivity bound for the NAST coefficients.

Second, to handle biomedical signals that exhibit long range dependence structure, like EEG, we establish central and non-central limit theorems for stationary Gaussian processes that might exhibit long-range dependence and fractional Brownian motions.
The appearance of Gaussian and non-Gaussian scenarios depends on the behavior of the Fourier transform
of the mother wavelet near the origin.
On the other hand, the structure of the large-scale limit depends not only on the Hermite rank of the nonlinear function $A_{1}$ used in the first layer but also on the differentiability of the nonlinear function used in the last layer.
In short, these diverse central and non-central type results come from exploring how NAST depends on the vanishing moment conditions of the chosen wavelet.
The developed theory can be directly applied to ST when the activation functions are $|x|$.

%





In the context henceforth, the notation $\mathbb{E}$ means the expectation and $\overset{d}{\Rightarrow}$ denotes the
convergence of random processes
in the sense of finite-dimensional
distributions. We use $\star$ to denote the convolution operator.
We use the notation $f^{\star\ell}$, where $\ell\geq 2$, to denote the {\em $\ell$-fold convolution} of the function $f$ when it makes sense; that is,
\begin{align}\notag
f^{\star\ell}(\eta) := \int_{\mathbb{R}^{\ell-1}}
f(\eta_{1})f(\eta_{2}-\eta_{1})\cdots f(\eta-\eta_{\ell-1})
\ d\eta_{1}\ d\eta_{2}\cdots d\eta_{\ell-1}.
\end{align}

The rest of the paper is organized as follows.
In Section \ref{sec:preliminary}, we  present some preliminaries. The formal definition of NAST is given in Section \ref{sec:NASTdefinition}. We state our main
results in Section \ref{sec:mainresult}, and their proofs are given
in Section \ref{sec:proof}. More results and auxillary lemmas and their proofs are given in the appendix.

\vskip 20 pt

\section{Preliminaries}\label{sec:preliminary}

Let $\psi\in L^{1}(\mathbb{R})\cap L^{2}(\mathbb{R})$ be the mother wavelet, which satisfies $\int_{\mathbb{R}} \psi(t) dt = 0$.
A family of real-valued functions $\{\psi_{j}(t)\mid j\in \mathbb{Z}, t \in \mathbb{R}\}$
is called a wavelet family if it is generated from $\psi$ through the dilation procedure
\begin{equation}
\psi_{j}(t) = 2^{-j}\psi(2^{-j}t).
\end{equation}
Denote the Fourier transform of $\psi$ by $\Psi$, which is defined as
\begin{equation}\label{df:Fourier}
\Psi(\lambda) = \int_{\mathbb{R}}e^{-i\lambda t}\psi(t) dt.
\end{equation}
Because $\psi\in L^{1}(\mathbb{R})\cap L^{2}(\mathbb{R})$ and $\int_{\mathbb{R}} \psi(t) dt = 0$,
$\Psi$ is a continuous and square-integrable function with $\Psi(0) = 0$. This observation motivates us to make the following assumption.

\begin{Assumption}\label{Assumption:1:wavelet}
For the Fourier transform of $\psi$, we assume that there exists a bounded and continuous function $C_{\Psi}$, which is positive at the origin (i.e., $C_{\Psi}(0)>0$),
such that the square-integrable function $\Psi$ can be expressed as
\begin{equation}\label{eq:Psi}
\Psi(\lambda) = C_{\Psi}(\lambda)|\lambda|^{\alpha}
\end{equation}
for some $\alpha>0$.
\end{Assumption}


By the Taylor expansion, we know that the parameter $\alpha$ in Assumption \ref{Assumption:1:wavelet} is greater than or equal to the vanishing moment $N$ of $\psi$, where
\begin{equation}\label{df:vanish_moment}
N = \max\left\{n\in \mathbb{N}\,\, \Big|\, \int_{\mathbb{R}}t^{\ell}\psi(t)dt=0\ \textup{for}\ \ell = 0,1,\ldots,n-1,\ \textup{and}\  \int_{\mathbb{R}}t^{n}\psi(t)dt\neq0\right\},
\end{equation}
Assumption \ref{Assumption:1:wavelet} with $\alpha\geq 1$ holds for commonly used wavelets, including the real part of complex Morlet wavelet ($\alpha=1$), the Mexican hat wavelet ($\alpha=2$), and the $K$-th order Daubechies wavelet ($\alpha=K/2$), where $K = 2,4,\ldots,20.$
Assumption \ref{Assumption:1:wavelet} with $\alpha\in(0,1)$ is realized by
the real part of Morse wavelets \cite{olhede2002generalized}, which is a family of analytic wavelets parametrized by two parameters in the Fourier domain via
\begin{equation}\label{eq:morse}
\Psi(\lambda) = \left\{\begin{array}{ll}K_{\alpha,\gamma}\mathcal{H}(\lambda)\lambda^{\alpha}e^{-\lambda^{\gamma}}\ & \textup{for}\ \lambda\geq 0,\\
 0 \ & \textup{for}\ \lambda< 0,\end{array}
 \right.
\end{equation}
where $\alpha>0$, $\gamma>0$, $K_{\alpha,\gamma}$ is a normalizing constant, and $\mathcal{H}$ is the Heaviside unit step function.
For the wavelet function, whose Fourier transform $\Psi$ vanishes in a neighborhood of the origin, e.g., the Mayer wavelet,
please see the paragraph below Theorem 1.

Given a bounded function $X: \mathbb{R}\rightarrow \mathbb{R}$, the CWT of $X$ is defined as
\begin{equation}\label{df:Wavelet}
     X\star \psi_{j} (t) = 2^{-j}\int_{\mathbb{R}}X(s)\psi(\frac{t-s}{2^{j}})ds\, ,
\end{equation}
where $j\in \mathbb{Z}$ indicates the scale and $t\in \mathbb{R}$ indicates time.
Compared with the windowed Fourier transform, the CWT is better in capturing short-lived high frequency phenomena, e.g., the singularities and transient structures, because the time-width of $2^{-j_{1}}\psi(2^{-j_{1}}t)$ is adapted to the scale variable $2^{j_{1}}$ \cite{daubechies1992ten}.

\subsection{Stationary Gaussian processes and Wiener-It$\hat{o}$ integrals}

Let $(\Omega, \mathcal{F}, P)$  be an underlying
probability space such that all random elements appeared in this article are defined on it.
Let $X$ be a mean-square continuous and
stationary real Gaussian random process with the constant mean $\mathbb E[X(s)]=\mu\in \mathbb{R}$ and the
covariance function $R_{X}$:
\begin{equation}
R_{X}(t_{1}-t_{2}) =\mathbb E[X(t_{1})X(t_{2})]-\mathbb E[X(t_{1})]\mathbb E[X(t_{2})],\ t_{1},t_{2}\in \mathbb{R}.
\end{equation}
Because $\int_{\mathbb{R}}\psi(s) ds = 0 $,
\begin{align*}
X\star\psi_{j}(t) =& \int_{\mathbb{R}} X(s) 2^{-j}\psi(\frac{t-s}{2^{j}})ds
= \int_{\mathbb{R}} \left\{X(s)-\mu\right\} 2^{-j}\psi(\frac{t-s}{2^{j}})ds
= [X-\mu]\star\psi_{j}(t).
\end{align*}
Hence, we can assume that $\mathbb E[X(s)]=0$ without loss of generality.
By the Bochner-Khinchin theorem \cite[Chapter 4]{krylov2002introduction},
there exists a unique nonnegative measure $F_{X}$ on $\mathbb{R}$ such that $F_{X}(\mathbb{R}) = R_{X}(0)$
and
\begin{equation}\label{eq:BK}
R_{X}(t) = \int_{\mathbb{R}}e^{i\lambda t}F_{X}(d\lambda),\ t\in \mathbb{R}.
\end{equation}
The measure $F_{X}$ is called the spectral measure of the covariance function $R_{X}$.

\begin{Assumption}\label{Assumption:2:spectral}
The spectral measure $F_{X}$ has the density
$f_{X}$ and
\begin{equation}\label{condition_f}
f_{X}(\lambda) = \frac{C_{X}(\lambda)}{|\lambda|^{1-\beta}},
\end{equation}
where $\beta\in(0,1)$ is the {\em Hurst index} of long-range dependent and $C_{X}$ is a bounded and continuous function from $\mathbb{R}$ to $[0,\infty)$
{such that the decay of $C_{X}(\lambda)$ as $|\lambda|\rightarrow\infty$ is faster than that of $|\lambda|^{-\beta-\varepsilon}$ for some $\varepsilon>0$}.
\end{Assumption}

Note that the spectral density function $f_{X}$ is nonnegative and integrable with $\int_{\mathbb{R}} f_{X}(\lambda) d\lambda = R_{X}(0)$.
If the function $C_{X}$ in Assumption \ref{Assumption:2:spectral} satisfies $C_{X}(0)>0$, then $X$ is a long-range dependent process \cite{doukhan2002theory,pipiras2017long} because
$f_{X}$ has a singularity at the origin.

\begin{Remark}
In \cite{gao2004modelling}, the long-range dependent Gaussian processes having spectral densities of the form
\begin{equation}\label{example_spectral}
f_{X}(\lambda) = \frac{c_{1}}{|\lambda|^{1-\beta_{1}}(|\lambda|^{2}+c_{2})^{\beta_{2}}},
\end{equation}
where $\beta_{1}\in(0,1)$, $\beta_{2}\in(1/2,\infty)$ and $c_{1},c_{2}\in(0,\infty)$, were applied to model the interest rate
and the Standard and Poor 500 data.
We refer readers with interest to \cite{leonenko1999limit,anh1999possible,doukhan2002theory} for more discussion about the stationary random processes with singular spectral densities.
The long-range dependence parameter $\beta$ can be obtained by regressing the logged squared wavelet
coefficients \cite{bardet2000wavelet}.
For $\beta_{1} = \beta_{2} = 1$,
(\ref{example_spectral}) is the spectral density
of the well-known {\em Ornstein-Uhlenbeck process}, which is short-range dependent.

\end{Remark}

Under Assumption \ref{Assumption:2:spectral}, (\ref{eq:BK}) can be rewritten as
\begin{equation}\label{eq:BK2}
R_{X}(t) = \int_{\mathbb{R}}e^{i\lambda t}f_{X}(\lambda)d\lambda,\ t\in \mathbb{R}.
\end{equation}
By the Karhunen Theorem, the Gaussian process $X$ has the
representation
\begin{align}\label{sample path represent}
X(t)=\int_{\mathbb{R}}
e^{i\lambda t}\sqrt{f_{X}(\lambda)}W(d\lambda)\,,
\end{align}
where $t\in \mathbb{R}$ and $W(d\lambda)$ is the standard complex-valued Gaussian random measure on $\mathbb{R}$
satisfying
\begin{align}\label{ortho}
W(\Delta_{1})=\overline{W(-\Delta_{1})}\ \ \textup{and}\ \ \mathbb{E}\left[W(\Delta_{1})
\overline{W(\Delta_{2})}\right]=\textup{Leb}(\Delta_{1}\cap\Delta_{2})
\end{align}
for any
 $\Delta_{1},\Delta_{2}\in
\mathcal{B}(\mathbb{R})$, where Leb is the Lebesgue measure on $\mathbb{R}$.
See, for example, \cite[Theorem 1.1.3]{leonenko1999limit} for the above facts.
By substituting (\ref{sample path represent}) into (\ref{df:Wavelet}),
the CWT of $X$ can be represented as a Wiener-It$\hat{\textup{o}}$ integral:
\begin{align}\label{lemma:scatter1}
X\star \psi_{j}(t) =\,& 2^{-j}\int_{\mathbb{R}}\left[\int_{\mathbb{R}}
e^{i\lambda s}\sqrt{f_{X}(\lambda)}W(d\lambda)\right]\psi(\frac{t-s}{2^{j}})ds
\\=\,&\notag2^{-j}\int_{\mathbb{R}}\left[\int_{\mathbb{R}}
e^{i\lambda s}\psi(\frac{t-s}{2^{j}})ds\right]\sqrt{f_{X}(\lambda)}W(d\lambda)
=\int_{\mathbb{R}}
e^{i\lambda t}\Psi(2^{j}\lambda)\sqrt{f_{X}(\lambda)}W(d\lambda)\,.
\end{align}

From (\ref{lemma:scatter1}), we know that $X\star \psi_{j}$ is a zero-mean stationary Gaussian process
with the spectral density $f_{X\star \psi_{j}}(\lambda) = |\Psi(2^{j}\lambda)|^{2}f_{X}(\lambda).$
By the orthogonal property (\ref{ortho}), we have
\begin{equation}\label{def:sigmaj1}
\sigma_{j}^{2}:=\mathbb E[\left(X\star \psi_{j}\right)^{2}] = \int_{\mathbb{R}}|\Psi(2^{j}\lambda)|^{2}f_{X}(\lambda) d\lambda,\ j\in \mathbb{Z}.
\end{equation}

\subsection{Fractional Brownian motions}\label{sec:fbm}
The representation (\ref{lemma:scatter1}) can be extended to the case of Gaussian processes with stationary increments,
such as the two-sided fractional Brownian motion $B_{H}(t)$ proposed in \cite{mandelbrot1968fractional},
where $t\in \mathbb{R}$ and $H\in (0,1)$.
The fractional Brownian motions have been widely used in financial models \cite{oksendal2003fractional}
and engineering applications.
From \cite[(19)]{reed1995spectral}, $B_{H}$, where $0<H<1$, has the spectral representation
\begin{equation}\label{eq:fbm}
B_{H}(t) = \frac{1}{2\pi} \int_{\mathbb{R}}(e^{i\lambda t}-1)\left(\frac{1}{i\lambda}\right)^{H+\frac{1}{2}}d\beta(\lambda),\ t\in \mathbb{R},
\end{equation}
where $\beta(\lambda) = -\overline{\beta(\lambda)}$ is a complex-valued Wiener process in frequency of orthogonal increments, that is,
if $(\lambda_{1},\lambda_{2})\cap (\lambda_{3},\lambda_{4}) = \phi$,
\begin{equation}\label{fbm_prop1}
\mathbb{E}\left\{[\overline{\beta(\lambda_{1})-\beta(\lambda_{2})}]
\left[\beta(\lambda_{3})-\beta(\lambda_{4})\right]\right\} = 0
\end{equation}
and
\begin{equation}\label{fbm_prop2}
\mathbb{E}\left\{\overline{d\beta(\lambda^{'})}
d\beta(\lambda)\right\} = 2\pi \delta(\lambda^{'}-\lambda)d\lambda.
\end{equation}
By a direct computation (see also \cite[Theorem 2]{reed1995spectral}), (\ref{eq:fbm}) implies that the covariance function of $B_{H}$ has the well-known form
\begin{equation*}
\textup{Cov}(B_{H}(s),B_{H}(t)) = \frac{V_{H}}{2}\left(|s|^{2H}+|t|^{2H}-|s-t|^{2H}\right),\ s,t\in \mathbb{R},
\end{equation*}
where
\begin{equation*}
V_{H} = \frac{1}{[\Gamma(H+\frac{1}{2})]^{1/2}}\left\{\frac{1}{2H}+\int_{1}^{\infty}\left[u^{H-\frac{1}{2}}-(u-1)^{H-\frac{1}{2}}\right]^{2}du\right\}.
\end{equation*}
By (\ref{eq:fbm}) and $\int_{\mathbb{R}}\psi(s)ds = 0$,
\begin{align}\notag
B_{H}\star \psi_{j_{1}}(t) =& \frac{1}{2\pi} \int_{\mathbb{R}}
\left[\int_{\mathbb{R}}(e^{i\lambda s}-1)\psi_{j_{1}}(t-s)ds\right]\left(\frac{1}{i\lambda}\right)^{H+\frac{1}{2}}d\beta(\lambda)
\\\label{fbm_psi}=&
\frac{1}{2\pi} \int_{\mathbb{R}}
e^{i\lambda t}\left(\frac{1}{i\lambda}\right)^{H+\frac{1}{2}}\Psi(2^{j_{1}}\lambda)d\beta(\lambda).
\end{align}

By the orthogonal properties (\ref{fbm_prop1}) and (\ref{fbm_prop2}),
the covariance function of $B_{H}\star \psi_{j_{1}}$ has the spectral representation
\begin{align}\label{fbm_psi_cov}
\textup{Cov}\left(B_{H}\star \psi_{j_{1}}(t_{1}),B_{H}\star \psi_{j_{1}}(t_{2})\right) =
\frac{1}{2\pi} \int_{\mathbb{R}}
e^{i\lambda (t_{1}-t_{2})}\frac{1}{|\lambda|^{2H+1}}|\Psi(2^{j_{1}}\lambda)|^{2}d\lambda.
\end{align}
From (\ref{fbm_psi}) and (\ref{fbm_psi_cov}), we know that if $\alpha>H$,
$B_{H}\star \psi_{j_{1}}$ is a stationary Gaussian process with spectral density
\begin{equation}\label{X_fbm}
f_{B_{H}\star \psi_{j_{1}}}(\lambda) = (2\pi)^{-1}|\lambda|^{-(2H+1)}|\Psi(2^{j_{1}}\lambda)|^{2},
\end{equation}
which is absolute integrable under the condition $\alpha>H$.
The condition $\alpha>H$ is satisfied by the real part of complex Morlet wavelet ($\alpha=1$), the Mexican hat wavelet ($\alpha=2$), and the $K$-th order Daubechies wavelet ($\alpha=K/2$), where $K = 2,4,\ldots,20$.
Furthermore, (\ref{fbm_psi_cov}) implies that the finite dimensional distribution of $B_{H}\star \psi_{j_{1}}$
is the same as that of $X\star \psi_{j_{1}}$, where $X$ is a generalized random process (also called $1/f$ noise) with spectral density
$f_{X}(\lambda) = (2\pi)^{-1}|\lambda|^{-(1-(-2H))}$, i.e., (\ref{condition_f}) with $\beta = -2H$ and $C_{X}(\lambda)=(2\pi)^{-1}$ for all $\lambda\in \mathbb{R}$.
Under the condition $\alpha>H$, the second moment of $X\star \psi_{j}$ (or $B_{H}\star \psi_{j_{1}}$)
has the following scaling relation:
\begin{equation}\label{def:sigmaj1_fbm}
\sigma_{j}^{2}=\mathbb E[\left(X\star \psi_{j}\right)^{2}] = \frac{1}{2\pi}\int_{\mathbb{R}}|\Psi(2^{j}\lambda)|^{2}
\frac{1}{|\lambda|^{2H+1}} d\lambda = 2^{2Hj}\sigma_{0}^{2},\ j\in \mathbb{Z}.
\end{equation}
In Section \ref{sec:mainresult}, we will see that our main result for the case of stationary Gaussian processes
can be extended to the case of fractional Brownian motions easily by the relationship above.


\subsection{Hermite polynomial expansion}

We consider the Hermite polynomial expansion to represent nonlinear functions of interest. We need the following assumption.

\begin{Assumption}\label{Assumption:3:Hermite}
Fix $k,j\in \mathbb{N}$. Assume that $A_{k}(\sigma_{j}\cdot)$ belongs to the Gaussian-Hilbert space $L^{2}(\mathbb{R},(2\pi)^{-1/2}e^{-\frac{y^{2}}{2}}dy)$;
that is,
$\int_{\mathbb{R}}\left[A_{k}(\sigma_{j} y)\right]^2 \frac{1}{\sqrt{2\pi}}e^{-\frac{y^{2}}{2}}dy<\infty$.
\end{Assumption}

\begin{Definition}[Hermite polynomials and Hermite rank]
The Hermite
polynomials, $\{H_{\ell}(y)\}_{\ell=0,1,2,\ldots}$, are defined as
\[
H_{\ell}(y)=(-1)^{\ell}e^{\frac{y^{2}}{2}}\frac{d^{\ell}}{dy^{\ell}}e^{-\frac{y^{2}}{2}}\,,
\]
where $\ell\in \{0,1,2,\ldots\}$. The Hermite rank of $A_{1}(\sigma_{j_{1}}\ \cdot)$ is defined by
\begin{align*}
\textup{rank}\left(A_{1}(\sigma_{j_{1}}\ \cdot)\right) = \textup{inf}\Big\{\ell\geq 1:\ C_{\sigma_{j_{1}},\ell}\neq 0\Big\},
\end{align*}
where
\begin{align}\label{hermitecoeff}
C_{\sigma_{j_{1}},\ell}=\int_{\mathbb{R}}A_{1}(\sigma_{j_{1}}y)\frac{H_{\ell}(y)}{\sqrt{\ell!}}\frac{1}{\sqrt{2\pi}}e^{-\frac{y^{2}}{2}}dy,
\ \ell\in\mathbb{N}\cup\{0\}.
\end{align}
\end{Definition}
For the Hermite polynomials, by simple calculation, $H_{0}(y) = 1$, $H_{1}(y) = y$, and $H_{2}(y) = y^{2}-1$. For the Hermite rank, some common examples are $\textup{rank}\left(\textup{ReLU}(\cdot)\right) =1$ and  $\textup{rank}\left(|\cdot|\right) =2$.
Recall the following well-known facts about Hermite polynomials (see, for example,
\cite[Corollary 5.5 and p. 30]{major1981lecture}, \cite{ito1951multiple} or \cite{houdre1994chaos}):
\begin{align}\label{expectionhermite}
\mathbb{E}[H_{\ell_{1}}(Y_{j_{1}}(t_{1}))H_{\ell_{2}}(Y_{j_{1}}(t_{2}))]=
\delta^{\ell_{1}}_{\ell_{2}} \ell_{1}!
R_{Y_{j_{1}}}^{\ell_{1}}(t_{1}-t_{2})\,,
\end{align}
where $Y_{j_{1}} = \frac{1}{\sigma_{j_{1}}}X\star\psi_{j_{1}}$, $\ell_1,\ell_2\in \{0,1,2,\ldots\}$, and $\delta^{\ell_{1}}_{\ell_{2}}$ is the Kronecker symbol.
Moreover,
\begin{align}\label{itoformula}
H_{\ell}(Y_{j_{1}}(t))=\sigma_{j_{1}}^{-\ell}
\int^{'}_{\mathbb{R}^{\ell}}e^{i(\lambda_{1}+\ldots+\lambda_{\ell})t}\left[\prod_{k=1}^{\ell}{\Psi(2^{j_{1}}\lambda_{k})\sqrt{f_{X}(\lambda_{k})}}\right]
W(d\lambda_{1})\cdots W(d\lambda_{\ell}),
\end{align}
where
$\int^{'}$ means that the integral excludes the diagonal hyperplanes
$\lambda_{k}=\mp \lambda_{k^{'}}$ for $k, k^{'}\in\{1,\ldots,\ell\}$ and  $k\neq k^{'}$.

\section{Neural activation scattering transform}\label{sec:NASTdefinition}

We now define NAST that generalizes ST considered in \cite{mallat2012group}.

\begin{Definition}[Neural activation scattering transform]
Fix $n\in \mathbb{N}$ and take $n$ Lipschitz continuous functions $A_{1},\ldots,A_{n}:\mathbb{R}\to\mathbb{R}$.
The $n$-th order NAST depending on $(A_{1},\ldots,A_n)$ is defined by
\begin{align}
\{U^{A_{1:m} }[j_{1},j_{2},\ldots,j_{n}]X(t)\}_{j_{1},j_{2},\ldots,j_{n}\in \mathbb{Z}}\,,\label{def:generalizedST:all}
\end{align}
where $A_{1:n} := (A_{1},\ldots,A_{n})$ and
\begin{align}
U^{A_{1:n}}[j_{1},\ldots,j_{n}]X(t)
&= A_{n}\left(U^{A_{1:n-1}}[j_{1},\ldots,j_{n-1}]X\star \psi_{j_{n}}(t)\right) \nonumber\\
&=U^{A_{n}}[j_{n}]\left(\cdots\left(U^{A_{2}}[j_{2}]\left(U^{A_{1}}[j_{1}]X\right)\right)\cdots\right)\nonumber\\
&=A_n\left(\left(\ldots A_{2}\left(A_1\left(X\star\psi_{j_{1}}\right)\star \psi_{j_{2}}\right)\ldots\right)\star \psi_{j_{n}}(t)\right)\,.
\end{align}
\end{Definition}

%


Let $\phi$ be the father wavelet whose Fourier transform $\Phi$, along with $\Psi$, satisfies the Littlewood-Paley condition
\begin{equation}\label{Littlewood_condition}
|\Phi(2^{J}\lambda)|^{2}+ \underset{j<J}{\sum}|\Psi(2^{j}\lambda)|^{2} \leq 1,\ \lambda\in \mathbb{R},
\end{equation}
for all $J\in \mathbb{Z}$ and $\phi\in L^{1}(\mathbb{R})$ with $\int_{\mathbb{R}} \phi(x)dx = 1$.
For $J\in \mathbb{N}$, define
\begin{equation*}
\phi_{J}(t) = 2^{-J}\phi(\frac{t}{2^{J}}),\ t\in \mathbb{R},
\end{equation*}
and
\begin{equation*}
\mathcal P_{J}X(t) = X\star \phi_{J}(t) = \int_{\mathbb{R}}2^{-J}\phi(\frac{t-s}{2^{J}})X(s)ds,\ t\in \mathbb{R}\,.
\end{equation*}

\begin{Definition}[The set of $n$-th order NAST coefficients]
Let $\Lambda_{J} = \{j\mid j\in \mathbb{Z}\ \textup{and}\ j<J\}$. The set of {\em $n$-th order NAST coefficients} is
\begin{equation*}
S^{A_{1:n}}[\Lambda_{J}^{n}]X := \left\{\mathcal P_J[U^{A_{n}}[j_{n}]\left(\cdots\left(U^{A_{2}}[j_{2}]\left(U^{A_{1}}[j_{1}]X\right)\right)\cdots\right)]\right\}_{(j_{1},\ldots,j_{n})\in \Lambda_{J}^{n}}\,.
\end{equation*}
where $\Lambda_{J}^{n} = \{(j_{1},\ldots,j_{n})\mid j_{k}\in \Lambda_{J}\ \textup{for}\ k=1,2,\ldots,n\}$. Also denote
\begin{equation*}
U^{A_{1:n}}[\Lambda_{J}^{n}]X := \left\{U^{A_{n}}[j_{n}]\left(\cdots\left(U^{A_{2}}[j_{2}]\left(U^{A_{1}}[j_{1}]X\right)\right)\cdots\right)\right\}_{(j_{1},\ldots,j_{n})\in \Lambda_{J}^{n}}
\end{equation*}
and
\begin{equation}\label{def:cellST}
U_{J}^{A_{1:n}}[\Lambda_{J}^{n}]X := \left\{\mathcal{P}_{J}U^{A_{1:n-1}}[\Lambda_{J}^{n-1}]X, U^{A_{1:n}}[\Lambda_{J}^{n}]X\right\}.
\end{equation}
We set $U^{A_{1:0}}[\Lambda_{J}^{0}]X = X$ and $S^{A_{1:0}}[\Lambda_{J}^{0}]X=\mathcal P_{J}X.$
\end{Definition}
We follow the symbols used in \cite{mallat2012group}.
The curly bracket in (\ref{def:cellST}) contains the set of $(n-1)$-th order NAST coefficients
and the $n$-th order NAST of $X$.
If the absolute value function is used as the nonlinear operator, i.e., $A_{1}(\cdot)= A_{2}(\cdot) = \cdots = |\cdot|$, (\ref{def:generalizedST:all}) is reduced to the classical ST of $X$ considered in \cite{mallat2012group,anden2014deep,bruna2015intermittent}.
We mention that this generalization is very close to that considered in \cite{wiatowski2017mathematical},
where every activation function is followed by a pooling layer. 
A similar definition holds for higher dimensional Euclidean space, but we focus on $\mathbb{R}$ in this work.

\subsection{From the neural network perspective}

The computational structure of ST is similar to fully-connected convolutional neural
networks (CNN) \cite{lecun2010convolutional,hinton2012deep}, and its convolutional structure inspired the proposal of ST in \cite{mallat2012group}.
However, in the classical ST, $A_1(\cdot)=A_2(\cdot)=\ldots=|\cdot|$, where the activation functions are not fully modeled.
The proposed NAST is richer than the traditional ST in the sense that more flexible activation functions are considered, and a random processes can be converted into different random processes by NAST.
Recall that the $n$-th order NAST includes $n$ convolution layers with nonlinear operators inserted between convolutional layers:
\begin{align}
U^{A_{1},\ldots,A_n}[j_{1},\ldots,j_{n}]X(t)
=A_n\left(\left(\ldots A_{2}\left(A_1\left(X\star\psi_{j_{1}}\right)\star \psi_{j_{2}}\right)\ldots\right)\star \psi_{j_{n}}(t)\right),\nonumber
\end{align}
where $A_{1}\ldots A_{n}$ are selected activation functions by users.
Here, $A_{1},\ldots, A_{n}$ are flexible to capture different types of activations associated with different types of neurons. Also, it could capture the practical fact that activations are always shifted by a bias parameter, which may be critical for successful learning \cite{wang2019bias}.
However, note that while ideally the activation functions $A_{1},\ldots, A_{n}$ for each hidden layer, or even for each neuron,  should be obtained by a chosen training process from the given dataset \cite{ertuugrul2018novel,nanni2020stochastic} for the purpose of improving the discriminative capabilities, in the traditional ST framework and current NAST framework, the training/learning process is not involved. 
%

%
%

\subsection{More existing Theoretical Results on random processes}\label{sec:Existing}
In \cite[Section 4]{mallat2012group}, the relationship between the second-order moments
\begin{equation}
\{\mathbb E[\left(U[j_{1},j_{2},\ldots,j_{n}]X\right)^{2}]\mid j_{1},j_{2},\ldots,j_{n}\in \mathbb{Z}\},\ n\in \mathbb{N},
\end{equation}
and $\mathbb E[|X|^{2}]$
is discussed when $X$ is a stationary process with finite second-order moments.
Moreover, the ST ended with a convolution with a father wavelet is shown to be nearly Lipschitz-continuous under the time warping.
In \ref{main_section_stability}, we show that NAST ended with a pooling layer
also has this important property.
For square integrable and non-random function $X$, the exponential decay of $\|U[j_{1},j_{2},\ldots,j_{n}]X\|_{2}$ along $n$
has been discussed in \cite{waldspurger2017exponential}.
From an application point of view, the first and second order ST coefficients are sufficient for musical genre and phone classification \cite{anden2014deep}. We also find that the first and second order ST coefficients are sufficient for EEG signal analysis \cite{liu2020diffuse,liu2020save}.

Also, recall a closely related quantity, the {\em normalized second-order scattering moments}, defined by
\begin{align}\label{def:second_scattering_moment}
\widetilde{S}X(j_{1},j_{2}) = \frac{\mathbb E\left[U[j_{1},j_{2}]X\right]}{\mathbb E\left[U[j_{1}]X\right]}= \frac{\mathbb E\left||X\star \psi_{j_{1}}|\star \psi_{j_{2}}\right|}{\mathbb E|X\star \psi_{j_{1}}|}\,,
\end{align}
where $j_{1},j_{2}\in \mathbb{Z}$.
The empirical version of these moments have been used as features for musical genre and phone classification \cite[Section VI]{anden2014deep}
because the normalization increases the desired invariance properties for machine learning models.
Because $|X\star \psi_{j_{1}}|$ provides the information about the occurrence time of burst activity at the scale $2^{j_{1}}$,
$\mathbb E\big||X\star \psi_{j_{1}}|\star \psi_{j_{2}}\big|$ measures the time variability of the burst activity over time scales
$2^{j_{2}}\geq 2^{j_{1}}$. Hence, $\widetilde{S}X(j_{1},j_{2})$ gives multiscale measurements of intermittency.

\section{Main Results}\label{sec:mainresult}

%


As mentioned in Section \ref{sec:introduction}, we will present two main results: (1) the horizontal translation invariance and insensitivity to deformation of NAST coefficients for stationary random processes;
(2) large-scale limiting theorems about the second-order NAST of a stationary Gaussian process and the fractional Brownian motions.

\subsection{Horizontal translation invariance and insensitivity to deformation}

%

We show that 
for any stationary random process $X$ with finite second moment, NAST exhibits horizontal translation invariance when the scale parameter $J$ increases. Different from the vertical translation invariance \cite{wiatowski2017mathematical}, the horizontal translation invariance means that the asymptotic translation invariance in every network layer is guaranteed as long as $J\rightarrow\infty$.
Second, we consider a randomly deformed process $\{X(t-\tau(t))\}_{t\in \mathbb{R}}$,
where $X$ and $\tau$ are independent stationary random processes, and provide a deformation sensitivity bound for the NAST coefficients.

\subsubsection{Non-expansiveness property}

\begin{Lemma}\label{thm:nonexpansive}
If the Littlewood-Paley condition (\ref{Littlewood_condition}) holds and the nonlinear functions $A_{1},A_{2},...,A_{n}$ are Lipschitz continuous with Lipschitz constant $L_{A_{i}}\leq 1$ and $A_{i}(0)=0$ for $i\in\{1,2,...,n\}$, then for any strictly stationary process $X$ with $\mathbb{E}[|X|^{2}]<\infty$,
\begin{align*}
\mathbb{E}\left[\|U^{A_{1:n}}[\Lambda_{J}^{n}]X\|^{2}\right] := \underset{(j_{1},...,j_{n})\in \Lambda_{J}^{n}}{\sum}\mathbb{E}\left[|U^{A_{1},...,A_{n}}[j_{1},...,j_{n}]X|^{2}\right]
\end{align*}
satisfies
\begin{align}\label{nonincreasing_U}
\mathbb{E}\left[|X|^{2}\right]\geq
\mathbb{E}\left[\|U^{A_{1}}[\Lambda_{J}]X\|^{2}\right]\geq\cdots\geq
\mathbb{E}\left[\|U^{A_{1:n-1}}[\Lambda_{J}^{n-1}]X\|^{2}\right]\geq\mathbb{E}\left[\|U^{A_{1:n}}[\Lambda_{J}^{n}]X\|^{2}\right]
\end{align}
for $n\in \mathbb{N}\cup\{0\}.$
For the NAST coefficients, we have
\begin{align}\label{nonexpansive_S}
\mathbb{E}\left[\overset{n}{\underset{m=0}{\sum}}\|S^{A_{1:m}}_{J}[\Lambda_{J}^{m}]X\|^{2}\right]
\leq \mathbb{E}\left[|X|^{2}\right],\ n\in \mathbb{N}\cup\{0\}.
\end{align}

\end{Lemma}


\ref{sec:proof_nonexpansive_UJA} provides the proof of Lemma \ref{thm:nonexpansive}.
Note that the modulus function $A(x) = |x|$, the rectified linear unit $A(x) = \max\{0,x\}$, the hyperbolic tangent $A(x) = \textup{tanh}(x)$,
and the shifted logistic sigmoid $A(x) = \frac{1}{1+e^{-x}}-\frac{1}{2}$ are Lipschitz continuous with Lipschitz constant $L_{A}\leq 1$ and $A(0)=0$.
The inequality (\ref{nonincreasing_U}) implies that $\mathbb{E}\left[\|U^{A_{1:n}}[\Lambda_{J}^{n}]X\|^{2}\right]$
is decreasing along with the order $n$ of NAST.
The inequality (\ref{nonexpansive_S}) shows that the total energy of the NAST coefficients over all layers
is bounded by $\mathbb{E}\left[|X|^{2}\right]$ for each fixed $J\in \mathbb{Z}$.


\subsubsection{Insensitivity to deformation and translation}

Let $\tau$ be a stationary random process.
We consider the randomly deformed process $L_{\tau}X(t) = X(t-\tau(t))$, where $t\in \mathbb{R}$,
and investigate the impact of random deformation on the NAST coefficients.
Denote the set of the NAST coefficients within the first $n$ layers by
\begin{equation}\label{def:pooled_NAST_first_m}
S^{A_{1:n}}_{J}[\Lambda_{J}^{0:n}]X = \left\{S^{A_{1:m}}_{J}[\Lambda_{J}^{m}]X\right\}_{0\leq m\leq n}=\{P_{J}U^{A_{1:m}}[\Lambda_{J}^{m}]X\}_{0\leq m\leq n}.
\end{equation}

\begin{Lemma}\label{thm:deformation}
Assume the same conditions as in Lemma \ref{thm:nonexpansive}.
If $\tau$ is statistically independent of $X$, $\tau\in \mathbf{C}^{2}(\mathbb{R})$ and $\|\tau^{'}\|_{\infty}\leq\frac{1}{2}$ almost surely, then for $n\in\mathbb{N}$,
\begin{align}\notag
\mathbb{E}\left[\|S_{J}^{A_{1:n}}[\Lambda_{J}^{0:n}]L_{\tau}X-S_{J}^{A_{1:n}}[\Lambda_{J}^{0:n}]X\|^{2}\right]
\leq
C\left\{(n+1)^{2}K(\tau)+2^{-2J}(n+1)\mathbb{E}\left[\|\tau\|_{\infty}^{2}\right]\right\}\mathbb{E}\left[|X|^{2}\right]
\end{align}
for some constant $C>0$, where
\begin{equation*}
K(\tau) =
\left\{
\begin{array}{ll}
\mathbb{E}\left\{\left[\|\tau^{'}\|_{\infty}\left(\textup{log}\frac{\|\bigtriangleup \tau\|_{\infty}}{\|\tau^{'}\|_{\infty}}\vee 1\right)+\|\tau^{''}\|_{\infty}\right]^{2}\right\}, & \|\tau^{'}\|_{\infty} \neq 0,
\\
0, & \|\tau^{'}\|_{\infty} = 0,
\end{array}\right.
\end{equation*}
and $\|\bigtriangleup \tau\|_{\infty}:= \underset{s,t \in \mathbb{R}}{\textup{sup}}|\tau(s)-\tau(t)|$.

\end{Lemma}

Lemma \ref{thm:deformation} means that for any fixed network depth $n$, if the strength of time warping decays, i.e., $K(\tau)\rightarrow 0$
and $\|\tau\|_{\infty}\rightarrow 0$, the NAST coefficients of the deformed process $L_{\tau}X$
will be close to that of the original process $X$.

If $\tau$ is a constant function, i.e., $\|\tau^{'}\|_{\infty} = 0$,
Lemma \ref{thm:deformation} shows that
\begin{align}\notag
\mathbb{E}\left[\|S_{J}^{A_{1:n}}[\Lambda_{J}^{0:n}]L_{\tau}X-S_{J}^{A_{1:n}}[\Lambda_{J}^{0:n}]X\|^{2}\right]
\leq
C2^{-2J}(n+1)\|\tau\|_{\infty}^{2}\mathbb{E}\left[|X|^{2}\right]
\end{align}
for some constant $C>0$.
It means that for any fixed network depth $n$, the NAST coefficients are asymptotically translation invariant when $J\rightarrow\infty$. This property is referred as the horizontal translation invariance \cite{wiatowski2017mathematical}.

\subsection{Large-scale limiting theorems}

%
We show that the rescaled random process derived from the second-order NAST
converges to a Gaussian process in the finite dimensional distribution sense if
(a) the commonly used wavelets with high vanishing moments, including the Morlet wavelet, the Mexican hat wavelet, and the Daubechies wavelet,
are used in NAST,
and
(b) the nonlinear function $A_{2}$ used in the last layer is differentiable at the origin with $A_{2}^{'}(0)\neq0$.
Under the condition (a), if $A_{2}$ is a continuous function but not differentiable at the origin or $A_{2}^{'}(0)=0$, e.g., the modulus operator $|\cdot|$ and the rectified linear Unit (ReLU) function, the limit of the rescaled coefficient process is a subordinated Gaussian process.
If the condition (a) does not hold,
the non-Gaussian scenarios may arise from the limit of the rescaled coefficient process if the Hermite rank of $A_{1}$ is greater than 1.
Proofs of all theorems and corollaries mentioned in this section are postponed to Section \ref{sec:proof}.

\begin{Remark}
The technical reason for {having such diverse results} is that if the condition (a) holds,
then the low frequency components in the stationary Gaussian processes will be filtered out significantly if wavelets with high vanishing moments are used in NAST.
After performing the nonlinear transformation by $A_{1}$, the filtered Gaussian process becomes a weakly dependent random process,
whose convolution with wavelet functions at larger scales belongs to the domain of attraction of Gaussian limits.
%
For wavelets with low vanishing moments, like the  Morse wavelet, 
%
 the long-range dependence of the initial input {will} be partially reserved after performing
the convolution with the wavelet and the subsequent nonlinear transformation by $A_{1}$.
Hence, the convolution of the {resulting process, which is usually} non-Gaussian {and of} long-range {dependence,} with
wavelet functions at larger scales belongs to the domain of attraction of non-Gaussian limits.
The proof is based on the truncation of Hermite expansions and the multiple Wiener-It$\hat{\textup{o}}$ integrals \cite{leonenko1999limit}.

\end{Remark}

%




\subsubsection{Gaussian limits arising from the second-order NAST with stationary Gaussian processes as inputs}

\begin{Theorem}\label{thm1:A_gaussian}
Let $\psi$ be a real-valued mother wavelet function satisfying Assumption \ref{Assumption:1:wavelet}, $X$ be a second-order stationary Gaussian process satisfying Assumption \ref{Assumption:2:spectral}, and $A_{1}$ be a function satisfying Assumption \ref{Assumption:3:Hermite}. Denote $\mathbf{r} = \textup{rank}(A_{1}(\sigma_{j_{1}}\cdot))$.
If the parameters $\alpha$ and $\beta$ in Assumptions 1 and 2 satisfy $2\alpha+\beta>1/\mathbf{r}$, then for each fixed $j_{1}\in \mathbb{Z}$, when $j_{2}\rightarrow \infty$,
the macroscopic rescaled random process
\begin{equation}\label{thm1:rescaled_process}
2^{j_{2}/2}U^{A_{1}}[j_{1}]X\star \psi_{j_{2}}(2^{j_{2}}t),\ t\in \mathbb{R},
\end{equation}
converges to a Gaussian process $\{V_{1}(t)\}_{t\in \mathbb{R}}$
in the finite dimensional distribution sense.
Moreover, $V_{1}$ can be represented by a Wiener-It$\hat{o}$ integral as follows
\begin{equation}\label{thm1_V1}
V_{1}(t) = \kappa\int_{\mathbb{R}}
e^{i\lambda t}
\Psi(\lambda)W(d\lambda),
\end{equation}
where $\Psi$ is the Fourier transform of $\psi$, $W(d\lambda)$ is a complex-valued Gaussian random measure on $\mathbb{R}$, and
\begin{align}\label{thm1:kappa}
\kappa=\left[\overset{\infty}{\underset{\ell=\mathbf{r}}{\sum}}
\sigma_{j_{1}}^{-2\ell}f_{X\star \psi_{j_{1}}}^{\star\ell}(0)
C^{2}_{\sigma_{j_{1}},\ell}\right]^{\frac{1}{2}}= \left[\frac{1}{2\pi}\int_{\mathbb{R}}R_{A_{1}(X\star \psi_{j_{1}})}(t)dt\right]^{\frac{1}{2}}.
\end{align}
Here, $\{C_{\sigma_{j_{1}},\ell}\}_{\ell= \mathbf{r},\mathbf{r}+1,\ldots}$ are the coefficients of the Hermite polynomials expansion of  $A_{1}(\sigma_{j_{1}}\cdot)$.
$\sigma_{j_{1}}^{2}$ and $f_{X\star \psi_{j_{1}}}$ are the variance and the spectral density function of the stationary process $X\star \psi_{j_{1}}$, respectively, and $R_{A_{1}(X\star \psi_{j_{1}})}$ is the covariance function of the stationary process $A_{1}(X\star \psi_{j_{1}})$.
\end{Theorem}

For the Mayer wavelet, whose Fourier transform $\Psi$ vanishes in a neighborhood
of the origin, $A_{1}(X\star \psi_{j_{1}})$ becomes a weakly dependent process for all stationary Gaussian process $X$.
Following the proof of Theorem \ref{thm1:A_gaussian},
we can show that the rescaled process $2^{j_{2}/2}U^{A_{1}}[j_{1}]X\star\psi_{j_{2}}(2^{j_{2}}t)$, where $t\in \mathbb{R}$,
also converges to a Gaussian process
in the finite dimensional distribution sense when $j_{2}$ tends to infinity.
A part of results in Theorem \ref{thm1:A_gaussian} is validated in Fig. \ref{fig:validation_theorem_gaussian},
in which $\psi$ is the Daubechies-4 wavelet ($\alpha = 2$), $A_{1}(\cdot) = |\cdot|$, $j_{1} = 1$, $j_{2}=10$,
and $X$ is a stationary Gaussian process with the spectral density (\ref{example_spectral}) and $(\beta_{1},\beta_{2},c_{2}) = (0.75,4,1)$.
The quantile-quantile plot in Fig. \ref{fig:validation_theorem_gaussian} shows that
the empirical distribution of the random variable $U^{A_{1}}[j_{1}]X\star \psi_{j_{2}}(0)$ normalized by the standard deviation
is very close to the Gaussian distribution.

\begin{figure}
\centering
\subfigure[][histogram]
{\includegraphics[scale=0.6]{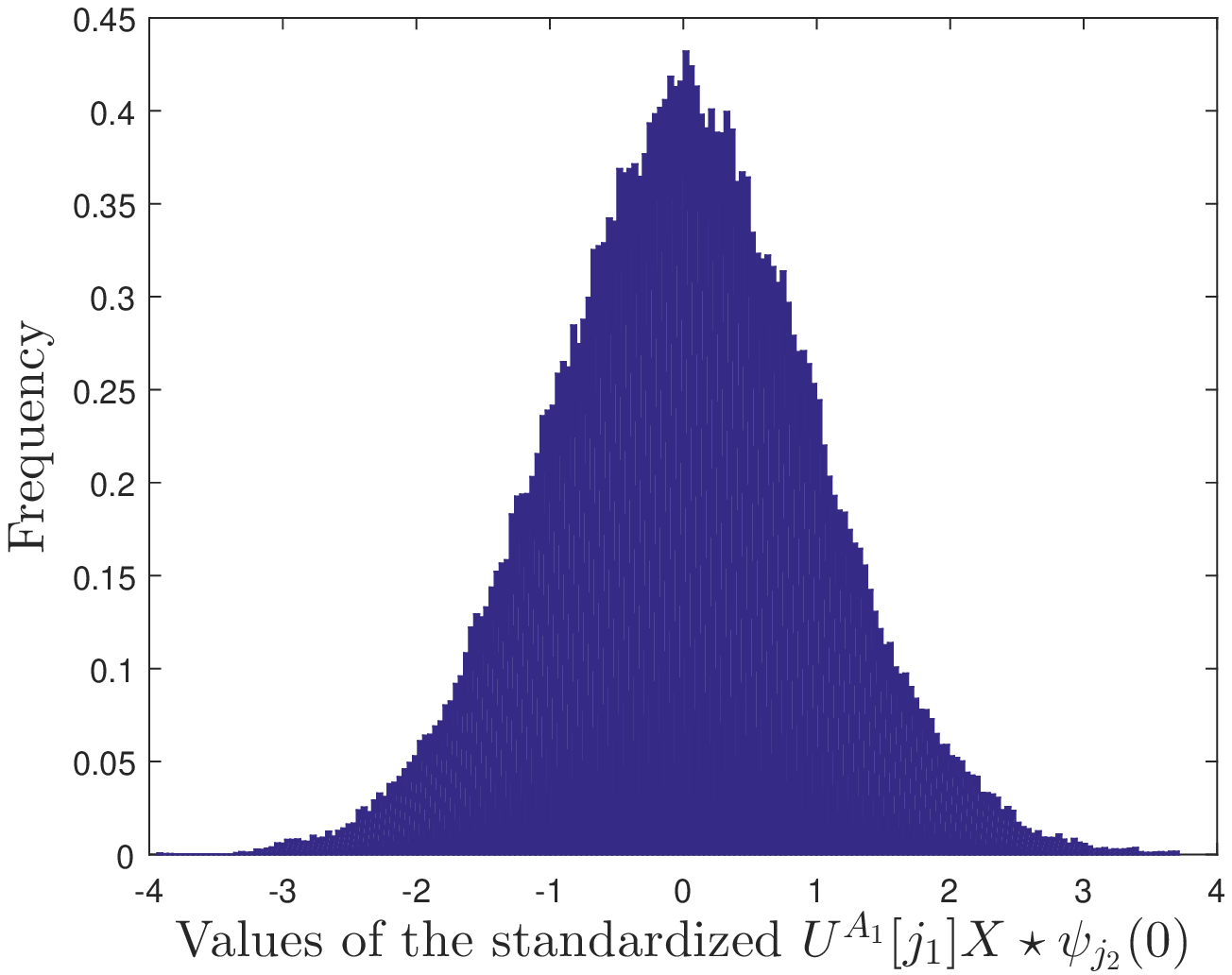}} \hspace{0.4cm}
\subfigure[][quantile-quantile plot]
{\includegraphics[scale=0.6]{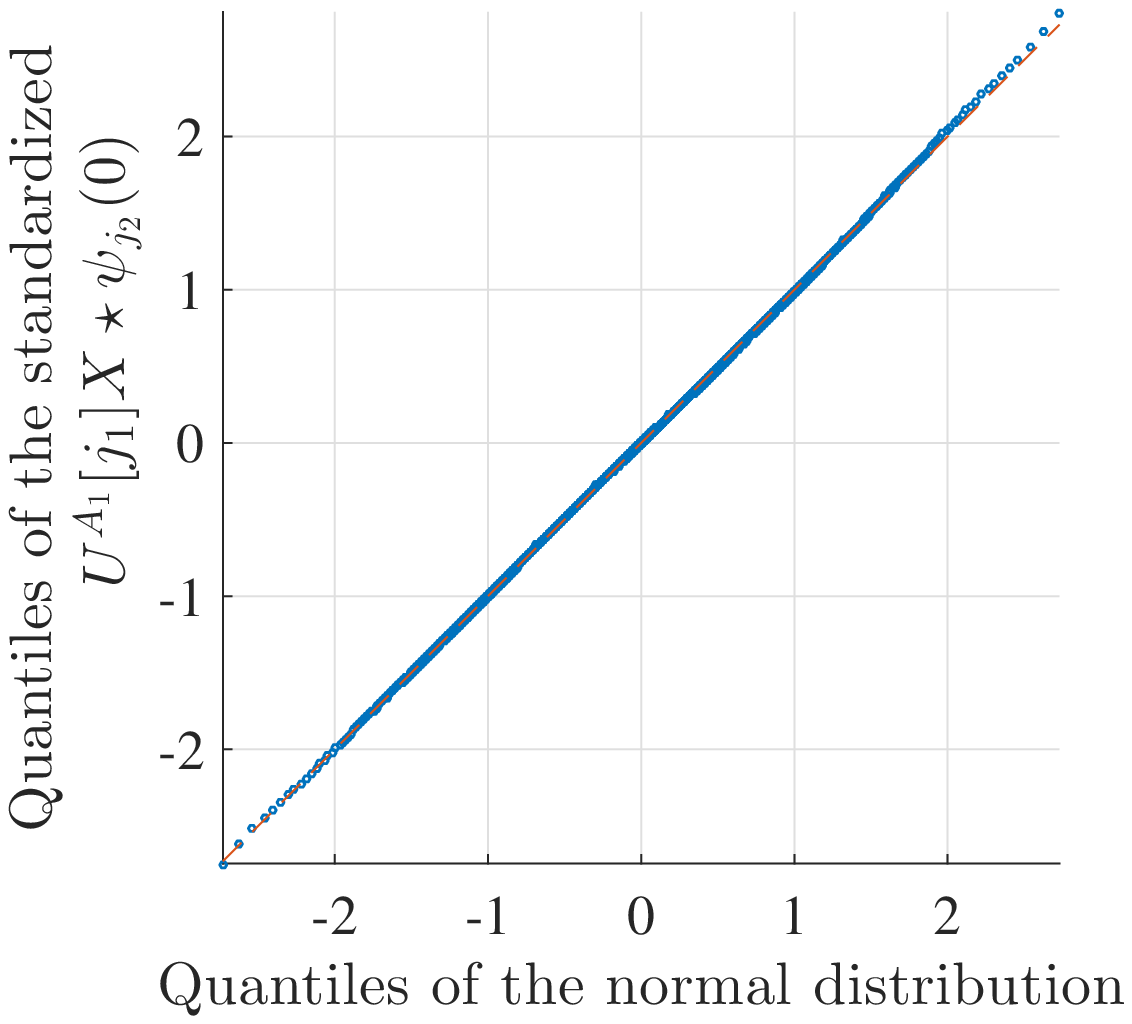}} 
\caption{Numerical validation of Theorem \ref{thm1:A_gaussian}.
Here, $X$ is a stationary Gaussian process with the spectral density (\ref{example_spectral}) and $(\beta_{1},\beta_{2},c_{2}) = (0.75,4,1)$ and
the Daubechies-4 wavelet is used for the scattering transform. We also set $A_{1}(\cdot)=|\cdot|$, $j_{1} = 1$ and $j_{2} = 10$.
}
\label{fig:validation_theorem_gaussian}
\end{figure}

For the function $A_{2}$, we assume that one of the following assumptions holds.

\begin{Assumption} \label{Assumption:4.1:A2:differentiable}
$A_{2}$ is differentiable at the origin with $A_{2}^{'}(0)\neq0$.
\end{Assumption}

\begin{Assumption} \label{Assumption:4.2:A2:Fchi}
$A_{2}$ belongs to the function space $F_{\chi}$ for some $\chi>0$, where
\begin{equation}
F_{\chi} := \{\tilde{A}\in C(\mathbb{R})\mid \tilde{A}(cx) = c^{\chi}\tilde{A}(x)\ \textup{for all}\ c>0\}.
\end{equation}
\end{Assumption}

Clearly, functions satisfying Assumption \ref{Assumption:4.1:A2:differentiable} and Assumption \ref{Assumption:4.2:A2:Fchi} are exclusive. The function space $F_{1}$ includes continuous functions of the form
\begin{equation}
A_{2}(x) = \left\{
\begin{array}{ll}
a_{+}x  &\ \ \ \textup{if}\ x\geq 0
\\
a_{-}x  &\ \ \ \textup{if}\ x< 0\,,
\end{array}\right.
\end{equation}
where $a_{+},a_{-}\in \mathbb{R}$. Typical examples include $A_{2}(\cdot) =|\cdot|$ and ReLU($\cdot$).

If $A_{2}$ satisfies Assumption \ref{Assumption:4.1:A2:differentiable}, we can apply the delta method
summarized in \ref{Lemma:Delta_method} (see also \cite[Theorem 3.7]{dasgupta2008asymptotic} and \cite[Section 8.3.1]{resnick2003probability}) to find the finite dimensional distribution and the covariance structure
of the scaling limit of $2^{j_{2}/2}U^{A_{1},A_{2}}[j_{1},j_{2}]X(2^{j_{2}}t)$ from the result of Theorem \ref{thm1:A_gaussian}.
Specifically, by the delta method, we immediately 
show that the covariance structure of the limit of $2^{j_{2}/2}U^{A_{1},A_{2}}[j_{1},j_{2}]X(2^{j_{2}}\cdot)$
is the same as that of $V_{1}$ up to a multiplicative constant.
If $A_{2}$ satisfies Assumption \ref{Assumption:4.2:A2:Fchi}, we have
\begin{align}
2^{\chi j_{2}/2}U^{A_{1},A_{2}}[j_{1},j_{2}]X(2^{j_{2}}t)
\label{bridge_first_second}
=A_{2}\Big(2^{j_{2}/2}U^{A_{1}}[j_{1}]X\star \psi_{j_{2}}(2^{j_{2}}t)\Big)
\end{align}
and the distribution comes immediately from the result of Theorem \ref{thm1:A_gaussian} and applying the continuous mapping theorem \cite{shao2006mathematical} to the right hand side of (\ref{bridge_first_second}). 

\begin{Corollary}\label{corollary_delta_generalA}
Assume the same conditions as in Theorem \ref{thm1:A_gaussian}. If Assumption \ref{Assumption:4.1:A2:differentiable} holds,
then for each fixed $j_{1}\in \mathbb{Z}$,
\begin{equation}\label{thm1.2.1}
2^{j_{2}/2}\Big\{U^{A_{1},A_{2}}[j_{1},j_{2}]X(2^{j_{2}}t)-A_{2}(0)\Big\}\overset{d}{\Rightarrow} A_{2}^{'}(0)V_{1}(t)
\end{equation}
when $j_{2}\rightarrow \infty$ in the finite dimensional distribution sense.
If Assumption \ref{Assumption:4.2:A2:Fchi} holds,
then for each fixed $j_{1}\in \mathbb{Z}$,
\begin{equation}\label{thm1.2.1}
2^{\chi j_{2}/2}U^{A_{1},A_{2}}[j_{1},j_{2}]X(2^{j_{2}}t)\overset{d}{\Rightarrow} A_{2}(V_{1}(t))
\end{equation}
when $j_{2}\rightarrow \infty$ in the finite dimensional distribution sense.
Moreover, if $A_{1}(\cdot) = A_{2}(\cdot)=|\cdot|$ 
and $C_{X}(0)>0$, for all $j_{1}\in \mathbb{Z}$,
\begin{equation}\label{thm1.2.1}
\underset{j\rightarrow\infty}{\lim}2^{j/2}\widetilde{S}X(j_{1},j_{1}+j) = \Theta(j_{1},\alpha,\beta),
\end{equation}
where
\begin{equation}\label{def:Theta}
\Theta(j_{1},\alpha,\beta) = 2^{-\frac{j_{1}}{2}}\kappa\sigma_{j_{1}}^{-1}\|\Psi\|_{L^{2}}
\end{equation}
and $\kappa$ is defined in (\ref{thm1:kappa}) with $A_{1}(\cdot)=|\cdot|$. Moreover,
\begin{equation}\label{thm1.2.2}
\underset{j_{1}\rightarrow\infty}{\textup{lim}}\Theta(j_{1},\alpha,\beta) = \|\Psi\|_{L^{2}}\left[\overset{\infty}{\underset{\ell=2}{\sum}}
\|Q\|_{L^{1}}^{-\ell}
Q^{\star \ell}(0)
C^{2}_{\ell}\right]^{\frac{1}{2}}<\infty,
\end{equation}
where $\{C_{\ell}\}_{\ell=2}^{\infty}$ are the coefficients of the Hermite polynomials expansion of  $|\cdot|$ and
\begin{equation}\label{corollary:Q}
Q(\lambda) = \frac{|\Psi(\lambda)|^{2}}{|\lambda|^{1-\beta}}.
\end{equation}

\end{Corollary}

\begin{Remark} For the case $A_{2}(\cdot) = |\cdot|\in F_{1}$ and $t_{1},t_{2}\in \mathbb{R}$, $(A_{2}(V_{1}(t_{1})),A_{2}(V_{1}(t_{2}))$
has the bivariate folded normal distribution.
An integral representation for the covariance between $|V_{1}(t_{1})|$ and $|V_{1}(t_{2})|$ is derived in \cite{murthy2015note}.
\end{Remark}

\subsubsection{Gaussian limits arising from the second-order NAST with the fractional Brownian motions as inputs}

By (\ref{fbm_prop1})-(\ref{fbm_psi_cov}),
the finite dimensional distributions of $B_{H}\star \psi_{j_{1}}$ and $X\star \psi_{j_{1}}$ coincide with each other when $\alpha>H$,
where $X$ is a Gaussian $1/f$ noise with spectral density $f_{X}(\lambda) = (2\pi)^{-1}|\lambda|^{-(2H+1)}$
(i.e., $C_{X}(\cdot) = (2\pi)^{-1}$ and $\beta = -2H$).
The self-similar structure of $f_{X}$ not only leads to the scaling relation $\sigma_{j}^{2}=2^{2Hj}\sigma_{0}^{2}$
for the variance of $X\star \psi_{j_{1}}$, which is derived from
(\ref{def:sigmaj1_fbm}), but also makes
\begin{equation}
f_{X\star \psi_{j}}^{\star\ell}(0) = 2^{j(2H\ell+1)}f^{\star \ell}_{X\star \psi_{0}}(0),\ j\in \mathbb{Z},\ \ell\in \mathbb{N}.
\end{equation}
Hence, we have the results for fractional Brownian motion parallel to Theorem \ref{thm1:A_gaussian} and Corollary \ref{corollary_delta_generalA} as follows.

\begin{Theorem}\label{thm1:A_gaussian_fbm}
Let $\psi$ be a real-valued mother wavelet function satisfying Assumption \ref{Assumption:1:wavelet}, $B_{H}$ be a two-sided fractional Brownian motion with Hurst index $H\in(0,1)$, and $A_{1}$ be a function satisfying Assumption \ref{Assumption:3:Hermite}. Denote $\mathbf{r} = \textup{rank}(A_{1}(\sigma_{j_{1}}\cdot))$.
If the parameter $\alpha$ in Assumption 1 satisfies $2(\alpha-H)>1/\mathbf{r}$, then for each fixed $j_{1}\in \mathbb{Z}$, when $j_{2}\rightarrow \infty$,
the macroscopic rescaled random process
\begin{equation}\label{thm1:rescaled_process_fbm}
2^{j_{2}/2}U^{A_{1}}[j_{1}]B_{H}\star \psi_{j_{2}}(2^{j_{2}}t),\ t\in \mathbb{R},
\end{equation}
converges to a Gaussian process $\{V_{1}(t)\}_{t\in \mathbb{R}}$
in the finite dimensional distribution sense.
Moreover, $V_{1}$ can be represented by a Wiener-It$\hat{\textup{o}}$ integral as follows
\begin{equation*}
V_{1}(t) = \kappa\int_{\mathbb{R}}
e^{i\lambda t}
\Psi(\lambda)W(d\lambda),
\end{equation*}
where
\begin{align}\label{thm1:kappa_fbm}
\kappa=\left[2^{j_{1}}\overset{\infty}{\underset{\ell=\mathbf{r}}{\sum}}
\sigma_{0}^{-2\ell}f_{X\star \psi_{0}}^{\star\ell}(0)
C^{2}_{\sigma_{j_{1}},\ell}\right]^{\frac{1}{2}}= \left[\frac{1}{2\pi}\int_{\mathbb{R}}R_{A_{1}(X\star \psi_{j_{1}})}(t)dt\right]^{\frac{1}{2}}.
\end{align}
Here, $X$ is the Gaussian 1/$f$ noise derived from $B_{H}$ through (\ref{X_fbm}).
The definitions about $\Psi$, $W(d\lambda)$, $\{C_{\sigma_{j_{1}},\ell}\}_{\ell= \mathbf{r},\mathbf{r}+1,\ldots}$, $\sigma_{0}^{2}$, $f_{X\star \psi_{0}}$, and $R_{A_{1}(X\star \psi_{j_{1}})}$ are the same as that in Theorem \ref{thm1:A_gaussian}.
\end{Theorem}

\begin{Corollary}\label{corollary_delta_generalA_fbm}
Assume the same conditions as in Theorem \ref{thm1:A_gaussian_fbm}. If Assumption \ref{Assumption:4.1:A2:differentiable} holds,
then for each fixed $j_{1}\in \mathbb{Z}$,
\begin{equation*}
2^{j_{2}/2}\Big\{U^{A_{1},A_{2}}[j_{1},j_{2}]B_{H}(2^{j_{2}}t)-A_{2}(0)\Big\}\overset{d}{\Rightarrow} A_{2}^{'}(0)V_{1}(t)
\end{equation*}
when $j_{2}\rightarrow \infty$ in the finite dimensional distribution sense.
If Assumption \ref{Assumption:4.2:A2:Fchi} holds,
then for each fixed $j_{1}\in \mathbb{Z}$,
\begin{equation*}
2^{\chi j_{2}/2}U^{A_{1},A_{2}}[j_{1},j_{2}]B_{H}(2^{j_{2}}t)\overset{d}{\Rightarrow} A_{2}(V_{1}(t))
\end{equation*}
when $j_{2}\rightarrow \infty$ in the finite dimensional distribution sense.
Moreover, if $A_{1}(\cdot) = A_{2}(\cdot)=|\cdot|$, 
for all $j_{1}\in \mathbb{Z}$,
\begin{equation}\label{thm1.2.1_fbm}
\underset{j\rightarrow\infty}{\lim}2^{j/2}\widetilde{S}B_{H}(j_{1},j_{1}+j) = \Theta(\alpha,-2H),
\end{equation}
where
\begin{equation}\label{def:Theta_fbm}
\Theta(\alpha,-2H) = \sigma_{0}^{-1}\|\Psi\|_{L^{2}}\left[\overset{\infty}{\underset{\ell=2}{\sum}}
\sigma_{0}^{-2\ell}f_{X\star \psi_{0}}^{\star\ell}(0)
C^{2}_{\ell}\right]^{\frac{1}{2}},
\end{equation}
which is independent to $j_{1}$.
Here, $\{C_{\ell}\}_{\ell= 2,3,\ldots}$ are the coefficients of the Hermite polynomials expansion of  $|\cdot|$.

\end{Corollary}

Theorem \ref{thm1:A_gaussian_fbm} is comparable with \cite[Lemma 3.3]{bruna2015intermittent}, where the authors proved that
if the wavelet $\psi$ have compact support and two vanishing moments, for any fixed $t\in \mathbb{R}$, when $j\rightarrow\infty$,
\begin{align}
2^{j/2}|B_{H}\star \psi| \star \psi_{j}(t)
\end{align}
is asymptotically normal with mean zero and variance $\|\psi\|^{2}_{L^{2}}\int_{\mathbb{R}}R_{|B_{H}\star \psi|}(\tau)d\tau$.
Theorem \ref{thm1:A_gaussian_fbm}
further provides the covariance function of the limiting process, which is not discussed in \cite{bruna2015intermittent}.
Corollary \ref{corollary_delta_generalA_fbm} shows that the limit of $\textup{log}_{2}\widetilde{S}B_{H}(j_{1},j_{1}+j)$
be irrelevant to $j_{1}$ and have a slope of $-1/2$. This part of result coincides with the result in \cite[Theorem 3.2]{bruna2015intermittent}, while Corollary \ref{corollary_delta_generalA_fbm} further provides a series representation for the limit $\underset{j\rightarrow\infty}{\lim}2^{j/2}\widetilde{S}B_{H}(j_{1},j_{1}+j)$.

Different from the case of fractional Brownian motion, Corollary \ref{corollary_delta_generalA} shows that the limit of $\textup{log}_{2}\widetilde{S}X(j_{1},j_{1}+j)$ indeed depends on $j_{1}$
when the stationary long-range dependent Gaussian process $X$ and the selected wavelet $\psi$ satisfy the conditions mentioned in Theorem \ref{thm1:A_gaussian}.
Basically, this dependence is caused by the fact that the random process $X$ is not self-similar.
The results of Corollary \ref{corollary_delta_generalA} are demonstrated by Fig. \ref{fig:demo_corollary1}, in which
we simulate a stationary Gaussian process with the spectral density (\ref{example_spectral}) with $(\beta_{1},\beta_{2},c_{2}) = (0.5,1,1)$ and
the Daubechies-4 wavelet is used for the scattering transform.
In \cite[Figs. 4-6]{bruna2015intermittent}, the behavior of $\textup{log}_{2}\widetilde{S}X(j_{1},j_{1}+j)$ for other random processes, including the $\alpha$-stable L$\acute{\textup{e}}$vy processes, stochastic self-similar processes and log-infinitely divisible
multifractal random processes, was explored mainly through the simulation approach.

\begin{figure}
\centering
\subfigure[][$\widetilde{S}X(j_{1})=\mathbb{E}|X\star \psi_{j_{1}}|/\mathbb{E}|X\star \psi|$]
{\includegraphics[scale=0.45]{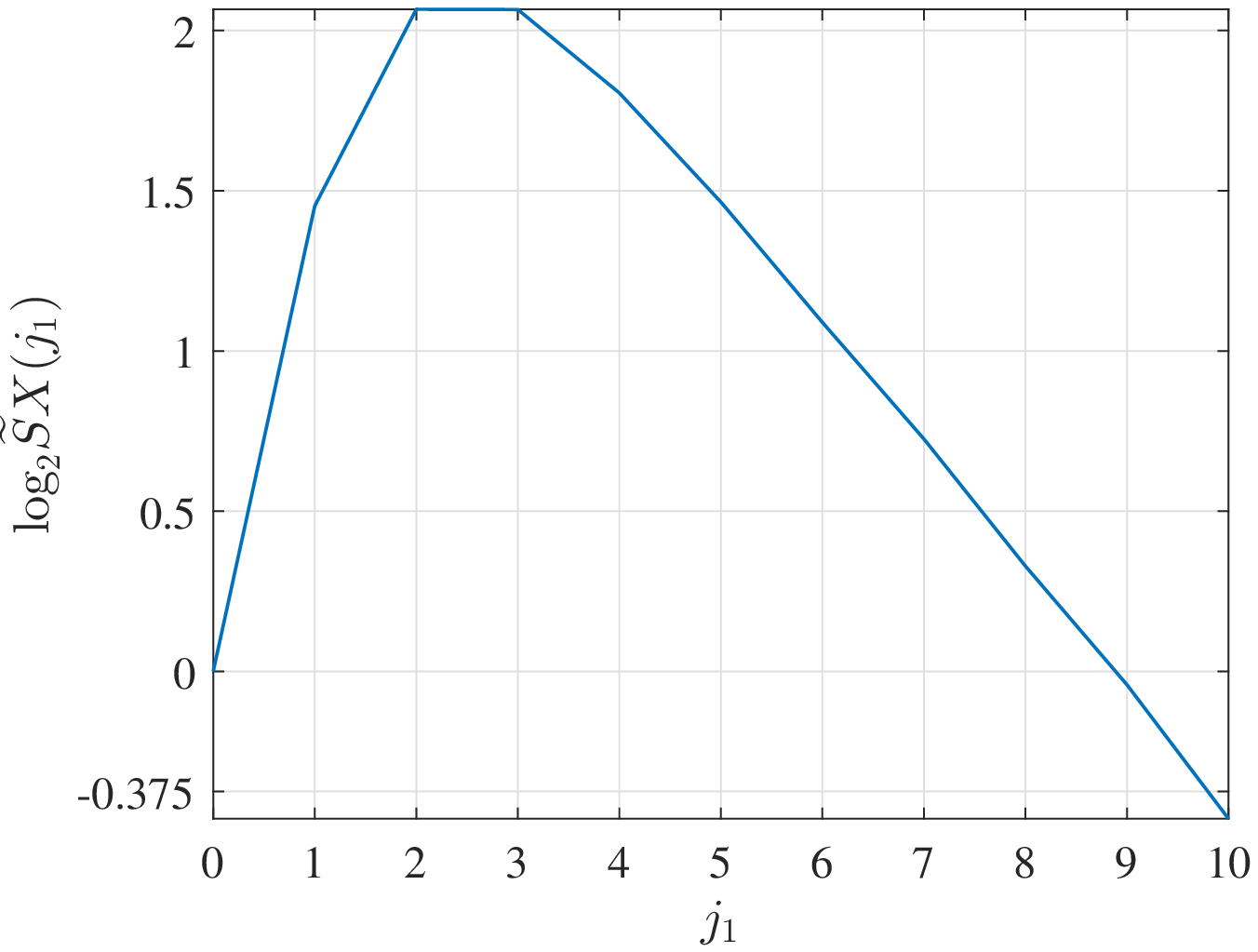}} 
\subfigure[][$\widetilde{S}X(j_{1},j_{1}+j)$]
{\includegraphics[scale=0.70]{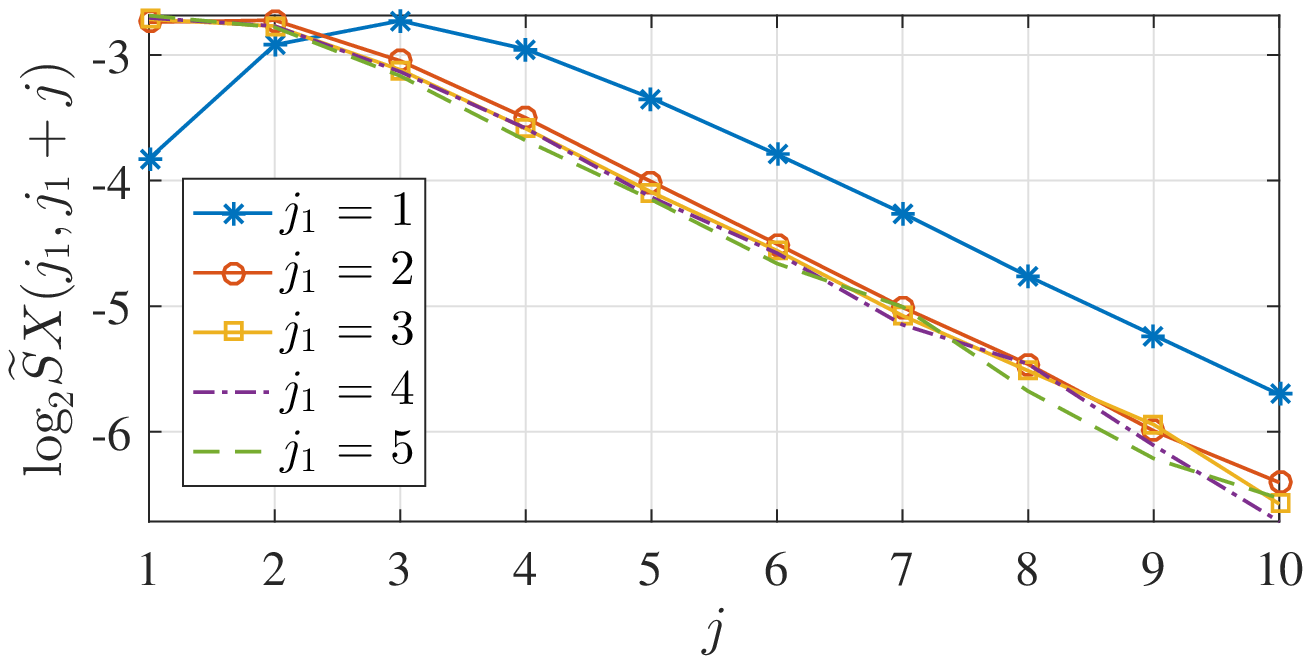}} 
\caption{Demonstration of Corollary \ref{corollary_delta_generalA} when $A_{1}(\cdot)=A_{2}(\cdot)=|\cdot|$. (a) is the $\textup{log}_{2}$ of the normalized first order scattering moment $\widetilde{S}X(j_{1})= \mathbb{E}\left[|X\star \psi_{j_{1}}|\right]/\mathbb{E}\left[|X\star \psi|\right]$. (b) is the $\textup{log}_{2}$ of the normalized second order scattering moment $\widetilde{S}X(j_{1},j_{1}+j)$ defined in (\ref{def:second_scattering_moment}). Here, $X$ is a stationary Gaussian process with the spectral density (\ref{example_spectral}) and $(\beta_{1},\beta_{2},c_{2}) = (0.75,4,1)$ and
the Daubechies-4 wavelet is used for the ST.
The slope of the curve in (a) is approximately equal to $-\beta_{1}/2=-0.375$ when $j_{1}$
is large enough.
For each $j_{1}\in\{1,2,\ldots,6\}$, the slope of the curve in (b) is approximately equal to $-0.5$ when $j$
is large enough.}
\label{fig:demo_corollary1}
\end{figure}

\subsubsection{Non-Gaussian limits arising from the second-order NAST with the stationary Gaussian processes as inputs}

\begin{Theorem}\label{thm:nonGaussian}
Let $\psi$ be a real-valued mother wavelet function satisfying Assumption \ref{Assumption:1:wavelet}, $X$ be a second-order stationary Gaussian process satisfying Assumption \ref{Assumption:2:spectral}, and $A_{1}$ be a function satisfying Assumption \ref{Assumption:3:Hermite} with $\mathbf{r}:=\textup{rank}(A_{1}(\sigma_{j_{1}}\cdot))$.
If the parameters $\alpha$ and $\beta$ in Assumptions 1 and 2 satisfy $2\alpha+\beta<1/\mathbf{r}$ and $C_{X}(0)>0$, then for each fixed $j_{1}\in \mathbb{Z}$, when $j_{2}\rightarrow\infty$,
the macroscopic rescaled random process
\begin{equation}
2^{j_{2}(2\alpha+\beta)\mathbf{r}/2}U^{A_{1}}[j_{1}]X\star\psi_{j_{2}}(2^{j_{2}}t)\,,
\end{equation}
where $t\in \mathbb{R}$,
converges to a random process $\{V_{2}(t)\}_{t\in \mathbb{R}}$
in the finite dimensional distribution sense.
Moreover, $V_{2}$ can be represented by a $\mathbf{r}$-fold Wiener-It$\hat{\textup{o}}$ integrals as follows
\begin{equation}\label{thm:limitprocess}
V_{2}(t) = \nu
\int^{'}_{\mathbb{R}^{\mathbf{r}}}e^{i(\lambda_{1}+\cdots+\lambda_{\mathbf{r}})t}
\frac{\Psi(\lambda_{1}+\cdots+\lambda_{\mathbf{r}})}{|\lambda_{1}\cdots\lambda_{\mathbf{r}}|^{\frac{1-2\alpha-\beta}{2}}}W(d\lambda_{1})\cdots W(d\lambda_{\mathbf{r}}),
\end{equation}
where $\Psi$ is the Fourier transform of $\psi$, $W(d\lambda)$ is a complex-valued Gaussian random measure on $\mathbb{R}$, and
\begin{align*}
\nu = \frac{C_{\sigma_{j_{1}},\mathbf{r}}}{\sqrt{\mathbf{r}!}}2^{j_{1}\alpha \mathbf{r}}\sigma_{j_{1}}^{-\mathbf{r}}C_{\Psi}(0)^{\mathbf{r}}C_{X}(0)^{\frac{\mathbf{r}}{2}}.
\end{align*}
\end{Theorem}

From (\ref{thm:limitprocess}), we know that $V_{2}$ is non-Gaussian if $\mathbf{r}\geq2$.
For the case $\mathbf{r}=1$, $V_{2}$ is Gaussian, but it is long-range dependent because its spectral density is singular at the origin.
A part of results in Theorem \ref{thm:nonGaussian} is validated in Fig. \ref{fig:validation_theorem_nongaussian},
in which $\psi$ is the Cauchy wavelet of order $\alpha=0.05$, $A_{1}(\cdot) = |\cdot|$, $j_{1} = 1$, $j_{2}=10$,
and $X$ is a stationary Gaussian process with the spectral density (\ref{example_spectral}) and $(\beta_{1},\beta_{2},c_{2}) = (0.2,4,1)$.
Under these conditions, the limiting process $V_{2}$ in (\ref{thm:limitprocess})
can be represented as a double Wiener integrals as follows:
\begin{align}\notag
V_{2}(t) =& \nu
\int^{'}_{\mathbb{R}^{\mathbf{2}}}e^{i(\lambda_{1}+\lambda_{2})t}
\frac{\Psi(\lambda_{1}+\lambda_{2})}{|\lambda_{1}\lambda_{2}|^{\frac{1-2\alpha-\beta}{2}}}W(d\lambda_{1})W(d\lambda_{2})
=\nu \sigma^{2} H_{2}\Big(\frac{1}{\sigma}\int_{\mathbb{R}}e^{i\lambda t}
\frac{\Psi(\lambda)}{|\lambda|^{\frac{1-2\alpha-\beta}{2}}}W(d\lambda)\Big),
\end{align}
where $\sigma$ is a constant such that
\begin{equation}\notag
\frac{1}{\sigma}\int_{\mathbb{R}}e^{i\lambda t}
\frac{\Psi(\lambda)}{|\lambda|^{\frac{1-2\alpha-\beta}{2}}}W(d\lambda)
\end{equation}
follows the standard normal distribution.
Because $\mathbb{E}[V_{2}(0)] = 0$ and $H_{2}(x) = x^{2}-1$, the distribution of the standardized $V_{2}(0)$, i.e., $(V_{2}(0)-\mathbb{E}[V_{2}(0)])/\sqrt{\textup{Var}(V_{2}(0))}$,
has the standardized $\chi^{2}$ distribution.
The quantile-quantile plot in Fig. \ref{fig:validation_theorem_nongaussian} shows that
the empirical distribution of the random variable $U^{A_{1}}[j_{1}]X\star \psi_{j_{2}}(0)$ normalized by the standard deviation
is very close to the standardized $\chi^{2}$ distribution.
\begin{figure}
\centering
\subfigure[][histogram]
{\includegraphics[scale=0.44]{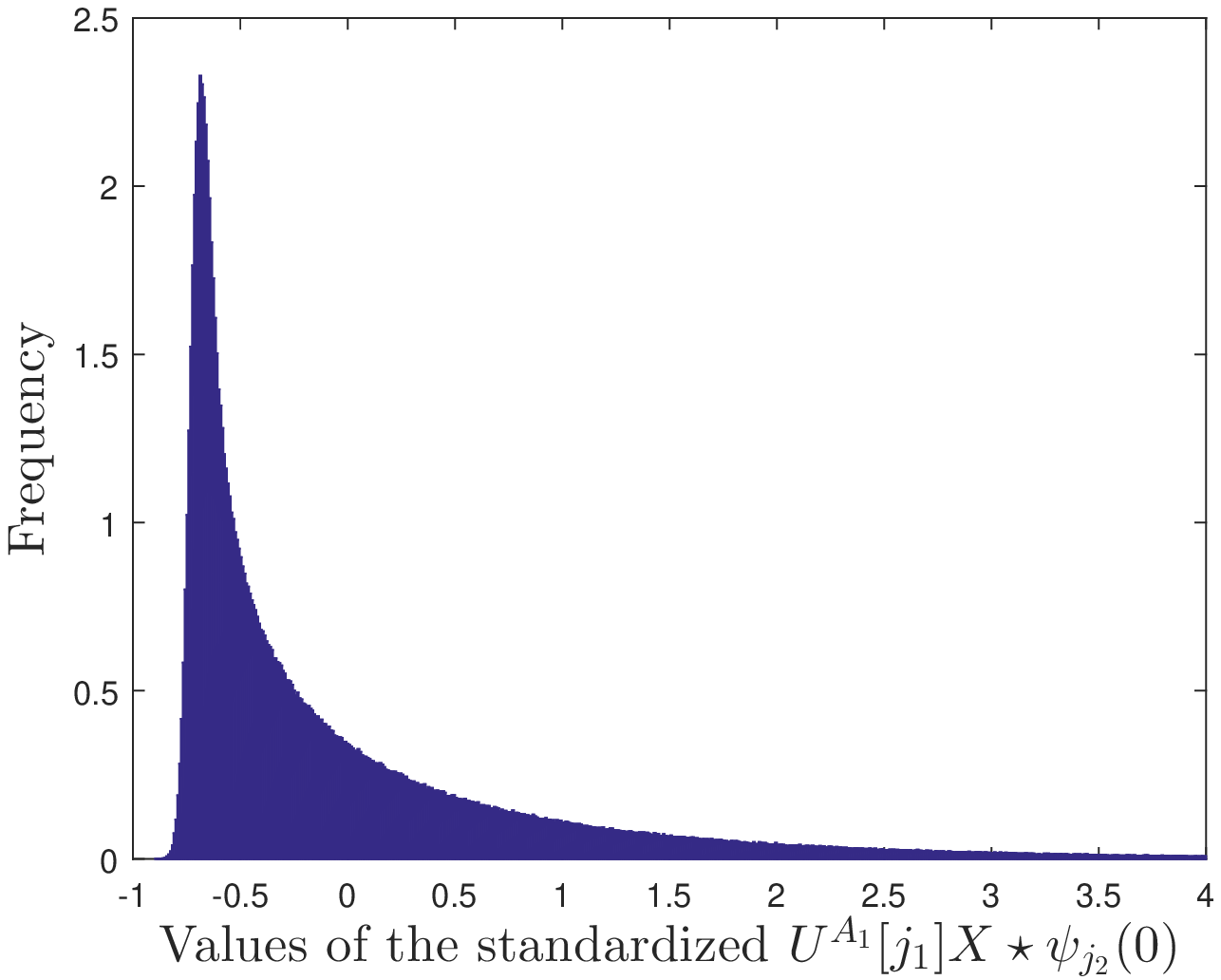}}
\subfigure[][quantile-quantile plot]
{\includegraphics[scale=0.44]{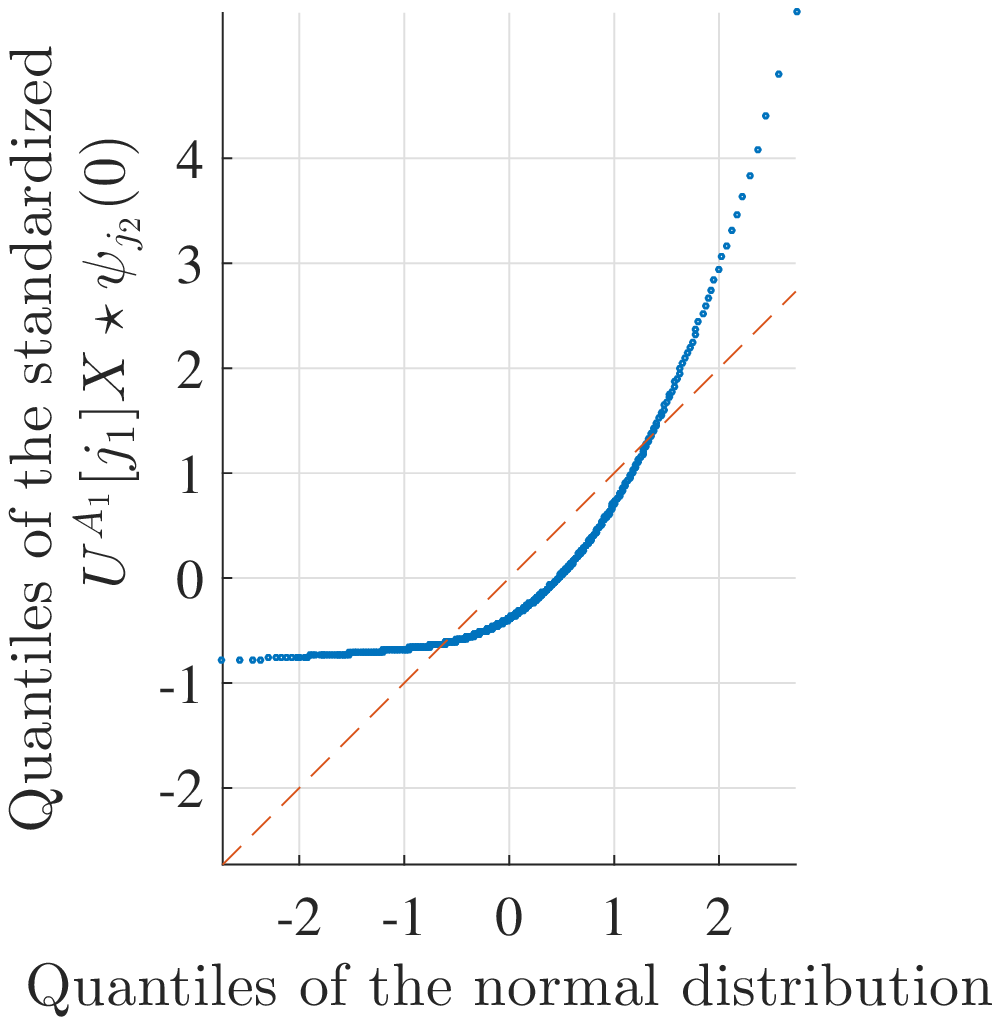}}
\subfigure[][quantile-quantile plot]
{\includegraphics[scale=0.44]{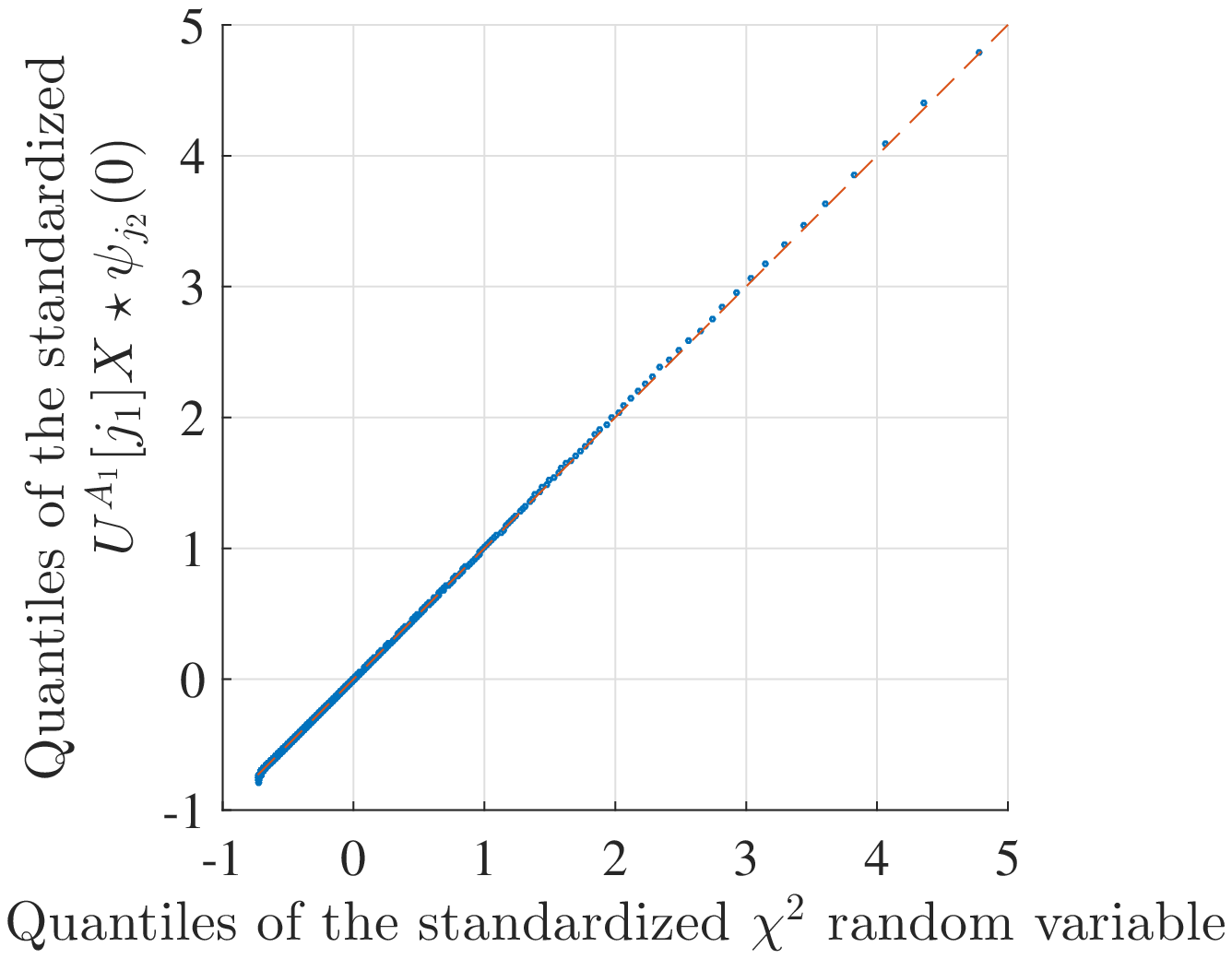}}
\caption{Numerical validation of Theorem \ref{thm:nonGaussian}.
Here, $X$ is a stationary Gaussian process with the spectral density (\ref{example_spectral}) and $(\beta_{1},\beta_{2},c_{2}) = (0.2,4,1)$, and
the real part of the Cauchy wavelet of order $\alpha=0.05$, whose Fourier transform is given by $\Psi(\lambda) = |\lambda|^{\alpha}e^{-|\lambda|}$, is used for the scattering transform.
We also set $A_{1}(\cdot)=|\cdot|$, $j_{1} = 1$ and $j_{2} = 10$.}
\label{fig:validation_theorem_nongaussian}
\end{figure}

If $A_{2}$ satisfies Assumption \ref{Assumption:4.1:A2:differentiable}, we can apply the delta method in \ref{Lemma:Delta_method} to find the finite dimensional distribution and the covariance structure
of the scaling limit of $2^{j_{2}(2\alpha+\beta)\mathbf{r}/2}U^{A_{1},A_{2}}[j_{1},j_{2}]X(2^{j_{2}}t)$ from the result of Theorem \ref{thm:nonGaussian}.
If $A_{2}$ satisfies Assumption \ref{Assumption:4.2:A2:Fchi}, then
\begin{align}\label{bridge_first_second_v2}
2^{\chi j_{2}(2\alpha+\beta)\mathbf{r}/2}U^{A_{1},A_{2}}[j_{1},j_{2}]X(2^{j_{2}}t) = A_{2}\Big(2^{ j_{2}(2\alpha+\beta)\mathbf{r}/2}U^{A_{1}}[j_{1}]X\star \psi_{j_{2}}(2^{j_{2}}t)\Big)\,,
\end{align}
where the distribution comes immediately from the result of Theorem \ref{thm:nonGaussian} and applying the continuous mapping theorem \cite{shao2006mathematical} to the right hand side of (\ref{bridge_first_second_v2}). 

\begin{Corollary}\label{corollary_nonGaussian_v1}
Assume the same conditions as in Theorem \ref{thm:nonGaussian}. If  $A_{2}$ satisfies Assumption \ref{Assumption:4.1:A2:differentiable},
then for each fixed $j_{1}\in \mathbb{Z}$,
\begin{equation}\label{nonGaussian_v1_delta}
2^{j_{2}(2\alpha+\beta)\mathbf{r}/2}\Big\{U^{A_{1},A_{2}}[j_{1},j_{2}]X(2^{j_{2}}t)-A_{2}(0)\Big\}\overset{d}{\Rightarrow} A_{2}^{'}(0)V_{2}(t)
\end{equation}
when $j_{2}\rightarrow \infty$ in the finite dimensional distribution sense.
(\ref{nonGaussian_v1_delta}) implies that the covariance function of the limiting process is the same as that of the process $V_{2}$
up to a multiplicative constant.
If Assumption \ref{Assumption:4.2:A2:Fchi} holds,
then for each fixed $j_{1}\in \mathbb{Z}$,
\begin{equation}
2^{\chi j_{2}(2\alpha+\beta)\mathbf{r}/2}U^{A_{1},A_{2}}[j_{1},j_{2}]X(2^{j_{2}}t)\overset{d}{\Rightarrow} A_{2}(V_{2}(t))
\end{equation}
when $j_{2}\rightarrow\infty$ in the finite dimensional distribution sense.
Moreover, if $A_{1}(\cdot)=A_{2}(\cdot)=|\cdot|$, for all $j_{1}\in \mathbb{Z}$,
\begin{equation}\label{thm2.2.1}
\underset{j\rightarrow\infty}{\lim}2^{j(2\alpha+\beta)}\widetilde{S}X(j_{1},j_{1}+j) = \Theta_{2}(j_{1},\alpha,\beta),
\end{equation}
where
\begin{equation}\label{def:Theta2}
\Theta_{2}(j_{1},\alpha,\beta) = 2^{-j_{1}\beta}\sigma_{j_{1}}^{-2}\frac{1}{2^{3/2}\sqrt{\pi}} C_{\Psi}(0)^{2}C_{X}(0)\mathbb E\Big|\int^{'}_{\mathbb{R}^{2}}
\frac{\Psi(\lambda_{1}+\lambda_{2})}{|\lambda_{1}\lambda_{2}|^{\frac{1-2\alpha-\beta}{2}}}W(d\lambda_{1})W(d\lambda_{2})\Big|
\end{equation}
and $C_{\Psi}$, $C_{X}$, and $Q$ are defined in (\ref{eq:Psi}), (\ref{condition_f}), and (\ref{corollary:Q}), respectively. Consequently,
\begin{equation}\label{thm2.2.2}
\underset{j_{1}\rightarrow\infty}{\textup{lim}}\Theta_{2}(j_{1},\alpha,\beta) = \frac{1}{2^{3/2}\sqrt{\pi}}  C_{\Psi}(0)^{2}C_{X}(0)^2 \|Q\|_{L^{1}}\mathbb E\Big|\int^{'}_{\mathbb{R}^{2}}
\frac{\Psi(\lambda_{1}+\lambda_{2})}{|\lambda_{1}\lambda_{2}|^{\frac{1-2\alpha-\beta}{2}}}W(d\lambda_{1})W(d\lambda_{2})\Big|.
\end{equation}

\end{Corollary}

\begin{figure}
\centering
\subfigure[][$\widetilde{S}X(j_{1})=E|X\star \psi_{j_{1}}|/|X\star \psi|$]
{\includegraphics[scale=0.49]{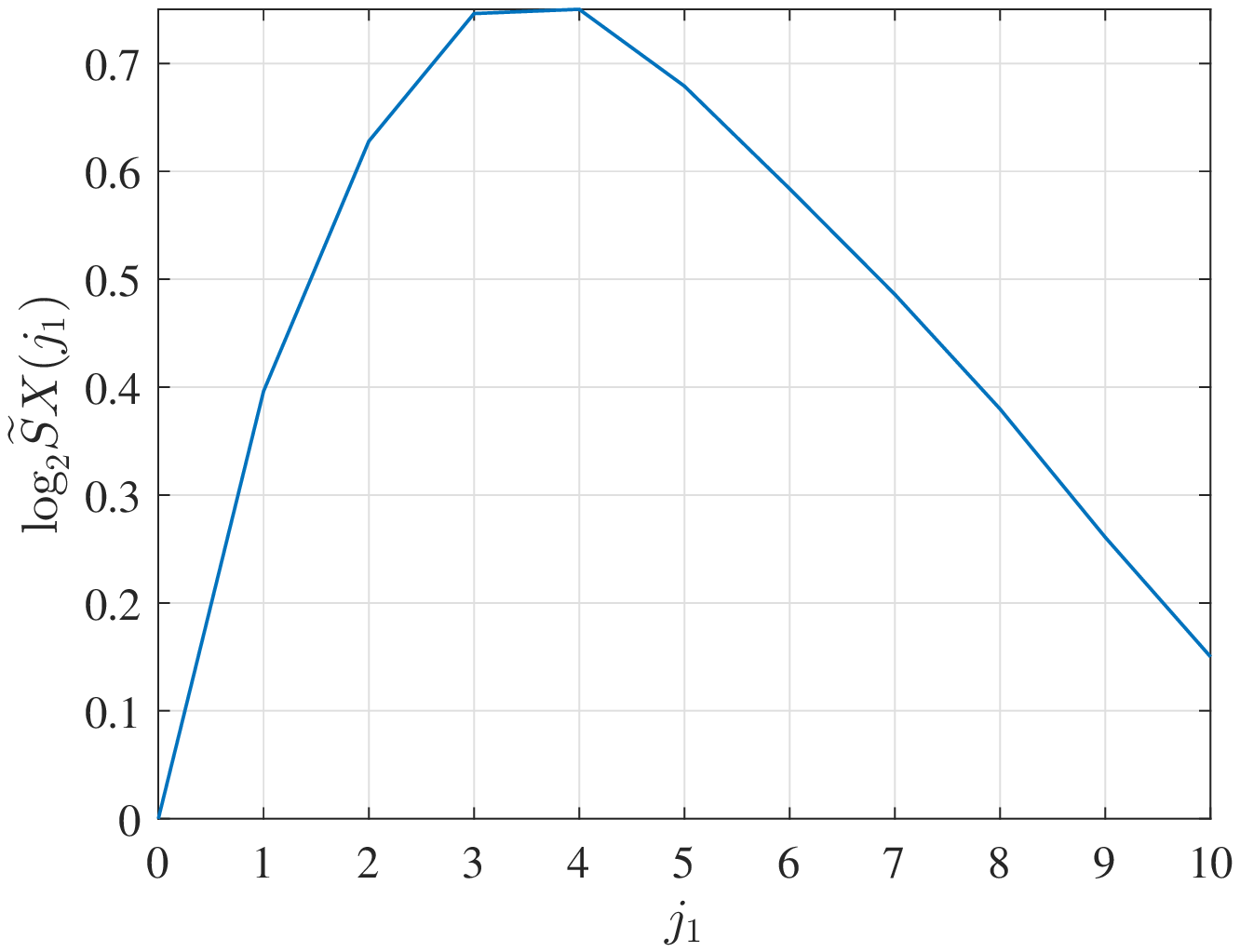}}
\subfigure[][$\widetilde{S}X(j_{1},j_{1}+j)$]
{\includegraphics[scale=0.49]{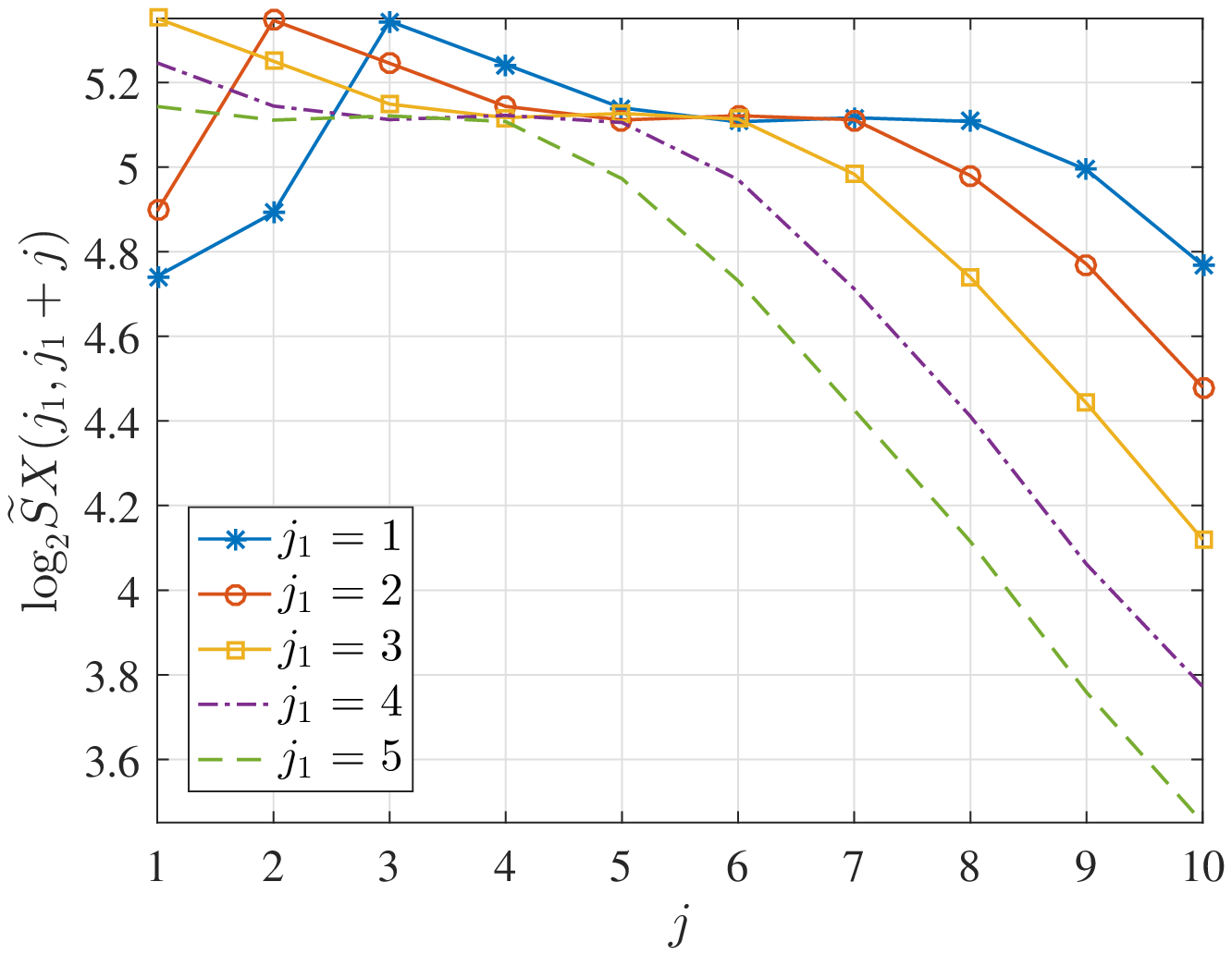}}
\caption{Numerical results for Corollary \ref{corollary_nonGaussian_v1} when $A_{1}(\cdot)=A_{2}(\cdot)=|\cdot|$. (a) is the $\textup{log}_{2}$ of the normalized first order scattering moment $\widetilde{S}X(j_{1})=E\left[|X\star \psi_{j_{1}}|/|X\star \psi|\right]$. (b) is the $\textup{log}_{2}$ of the normalized second order scattering moment $\widetilde{S}X(j_{1},j_{1}+j)$ defined in (\ref{def:second_scattering_moment}). Here, $X$ is a stationary Gaussian process with the spectral density (\ref{example_spectral}) with $(\beta_{1},\beta_{2},c_{2}) = (0.2,4,1)$ and
the real part of the Cauchy wavelet of order $\alpha=0.05$, whose Fourier transform is given by $\Psi(\lambda) = |\lambda|^{\alpha}e^{-|\lambda|}$, is used for the scattering transform.
The slope of the curve in (a) is very small ($\approx -\beta_{1}/2=-0.1$) when $j_{1}$
is large enough.
For each $j_{1}\in\{1,2,\ldots,5\}$, the slope of the curve in (b) is approximately equal to $-(2\alpha+\beta)=-0.3$ when $j$
is large enough.
}
\label{fig:demo_corollary2}
\end{figure}

When $A_{1}(\cdot)=A_{2}(\cdot)=|\cdot|$, under the condition $2\alpha+\beta>1/2$, Corollary \ref{corollary_delta_generalA} shows that the slope of $\textup{log}_{2}\widetilde{S}X(j_{1},j_{1}+j)$
with respect to $j$ is approximately equal to $-1/2$ when $j$ and $j_{1}$ are large enough.
Different from Corollary \ref{corollary_delta_generalA},
Corollary \ref{corollary_nonGaussian_v1} shows that under the condition $2\alpha+\beta<1/2$, the slope of $\textup{log}_{2}\widetilde{S}X(j_{1},j_{1}+j)$
with respect to $j$ is approximately equal to $-(2\alpha+\beta)$ when $j$ is large enough.
(\ref{thm2.2.2}) is demonstrated in Fig. \ref{fig:demo_corollary2}(b).

\subsubsection{Non-Gaussian limits arising from the second-order NAST with the fractional Brownian motions as inputs}

Similarly, 
we have parallel results of fractional Brownian motion for Theorem \ref{thm:nonGaussian} and Corollary \ref{corollary_nonGaussian_v1} as follows.


\begin{Theorem}\label{thm:nonGaussian_fbm}
Let $\psi$ be a real-valued mother wavelet function satisfying Assumption \ref{Assumption:1:wavelet}, $B_{H}$ be a two-sided fractional Brownian motion stationary with Hurst index $H\in(0,1)$, and $A_{1}$ be a function satisfying Assumption \ref{Assumption:3:Hermite} with $\mathbf{r}:=\textup{rank}(A_{1}(\sigma_{j_{1}}\cdot))$.
If the parameters $\alpha$ in Assumption 1 satisfies $0<2(\alpha-H)<1/\mathbf{r}$, then for each fixed $j_{1}\in \mathbb{Z}$, when $j_{2}\rightarrow\infty$,
the macroscopic rescaled random process
\begin{equation}
2^{j_{2}(\alpha-H)\mathbf{r}}U^{A_{1}}[j_{1}]B_{H}\star\psi_{j_{2}}(2^{j_{2}}t)\,,
\end{equation}
where $t\in \mathbb{R}$,
converges to a random process $\{V_{2}(t)\}_{t\in \mathbb{R}}$
in the finite dimensional distribution sense.
Moreover, $V_{2}$ can be represented by a $\mathbf{r}$-fold Wiener-It$\hat{\textup{o}}$ integrals as follows
\begin{equation}\label{thm:limitprocess_fbm}
V_{2}(t) = \nu
\int^{'}_{\mathbb{R}^{\mathbf{r}}}e^{i(\lambda_{1}+\cdots+\lambda_{\mathbf{r}})t}
\frac{\Psi(\lambda_{1}+\cdots+\lambda_{\mathbf{r}})}{|\lambda_{1}\cdots\lambda_{\mathbf{r}}|^{\frac{1-2\alpha+2H}{2}}}W(d\lambda_{1})\cdots W(d\lambda_{\mathbf{r}}),
\end{equation}
where 
\begin{align*}
\nu = \frac{C_{\sigma_{j_{1}},\mathbf{r}}}{\sqrt{\mathbf{r}!}}2^{j_{1}(\alpha-H) \mathbf{r}}\sigma_{0}^{-\mathbf{r}}C_{\Psi}(0)^{\mathbf{r}}(2\pi)^{-\frac{\mathbf{r}}{2}}.
\end{align*}
\end{Theorem}

\begin{Corollary}\label{corollary_nonGaussian_v1_fbm}
Assume the same conditions as in Theorem \ref{thm:nonGaussian_fbm}. If  $A_{2}$ satisfies Assumption \ref{Assumption:4.1:A2:differentiable},
then for each fixed $j_{1}\in \mathbb{Z}$,
\begin{equation*}
2^{j_{2}(\alpha-H)\mathbf{r}}\Big\{U^{A_{1},A_{2}}[j_{1},j_{2}]B_{H}(2^{j_{2}}t)-A_{2}(0)\Big\}\overset{d}{\Rightarrow} A_{2}^{'}(0)V_{2}(t)
\end{equation*}
when $j_{2}\rightarrow \infty$ in the finite dimensional distribution sense.
If Assumption \ref{Assumption:4.2:A2:Fchi} holds,
then for each fixed $j_{1}\in \mathbb{Z}$,
\begin{equation}
2^{\chi j_{2}(\alpha-H)\mathbf{r}}U^{A_{1},A_{2}}[j_{1},j_{2}]X(2^{j_{2}}t)\overset{d}{\Rightarrow} A_{2}(V_{2}(t))
\end{equation}
when $j_{2}\rightarrow\infty$ in the finite dimensional distribution sense.
Moreover, if $A_{1}(\cdot)=A_{2}(\cdot)=|\cdot|$, for all $j_{1}\in \mathbb{Z}$,
\begin{equation}\label{thm2.2.1}
\underset{j\rightarrow\infty}{\lim}2^{j(2\alpha-2H)}\widetilde{S}X(j_{1},j_{1}+j) = \Theta_{2}(\alpha,-2H),
\end{equation}
where
\begin{equation}\label{def:Theta2}
\Theta_{2}(\alpha,-2H) = \sigma_{0}^{-2}\frac{1}{2^{5/2}\pi^{3/2}} C_{\Psi}(0)^{2}\mathbb E\Big|\int^{'}_{\mathbb{R}^{2}}
\frac{\Psi(\lambda_{1}+\lambda_{2})}{|\lambda_{1}\lambda_{2}|^{\frac{1-2\alpha+2H}{2}}}W(d\lambda_{1})W(d\lambda_{2})\Big|,
\end{equation}
which is independent to $j_{1}$.
\end{Corollary}


\section{Proofs of Results in Section \ref{sec:mainresult}}\label{sec:proof}

%

%
Overall, the key technique for the proof is the {\em chaos expansion} (or Wiener chaos decomposition) \cite{taqqu1979convergence,breuer1983central,clausel2012large}.
The chaos expansion is usually used to deal with the nonlinear functionals. In \cite{clausel2012large}, it was applied to show that the large-scale limit of the wavelet coefficients only depends on the first term of the chaos expansion of the nonlinear functional, and the non-Gaussian scenarios appear when the Hermite rank of the functional is greater than or equal to 2.
%
In addition, the moment method and the Feynman-type diagrams \cite{breuer1983central} are applied to find out the limits.

$If A_{1}(\sigma_{j_{1}}\cdot)$ satisfies Assumption \ref{Assumption:3:Hermite}, then
$A_{1}(\sigma_{j_{1}}\cdot)$ has the following expansion:
\begin{align}\label{hermiteexpansion}
A_{1}(\sigma_{j_{1}}y)=
\sum_{\ell=0}^{\infty}{C_{\sigma_{j_{1}},\ell}
\frac{H_{\ell}(y)}{\sqrt{\ell!}}},
\end{align}
where $C_{\sigma_{j_{1}},\ell}$ is defined in
(\ref{hermitecoeff}).
By definition, the first-order NAST is
\begin{equation}\label{def:generalized_1st_ST}
U^{A_{1}}[j_{1}]X(t) = A_{1}\left(X\star \psi_{j_{1}}(t)\right),\ j_{1}\in \mathbb{Z}.
\end{equation}
and
the second-order NAST of $X$ is
\begin{equation}\label{df:second_order_scattering}
     U^{A_{1},A_{2}}[j_{1},j_{2}]X(t) = A_{2}\left(U^{A_{1}}[j_{1}]X\star \psi_{j_{2}}(t)\right)
     = A_{2}\left(2^{-j_{2}}\int_{\mathbb{R}} U^{A_{1}}[j_{1}]X(s)\psi(\frac{t-s}{2^{j_{2}}})ds\right).
\end{equation}
By combining (\ref{itoformula}), (\ref{hermiteexpansion}), (\ref{def:generalized_1st_ST}), and the property $\int_{\mathbb{R}}\psi(s)ds=0$,
we get
\begin{align}\label{eq:decom_2nd_scat}
U^{A_{1}}[j_{1}]X\star \psi_{j_{2}}(t)
= \int_{\mathbb{R}}\left[\sum_{\ell=\mathbf{r}}^{\infty}
\frac{C_{\sigma_{j_{1}},\ell}}{\sqrt{\ell!}}
H_{\ell}(Y_{j_{1}}(s))\right]2^{-j_{2}}\psi(\frac{t-s}{2^{j_{2}}})ds
= \sum_{\ell=\mathbf{r}}^{\infty}
\frac{C_{\sigma_{j_{1}},\ell}}{\sqrt{\ell!}}Z_{\ell}(t),
\end{align}
where $\mathbf{r}=\textup{rank}\left(A_{1}(\sigma_{j_{1}}\ \cdot)\right)$,
$Y_{j_{1}} = \frac{1}{\sigma_{j_{1}}}X\star \psi_{j_{1}}$,
and
\begin{align}
Z_{\ell}(t) =& 2^{-j_{2}}\int_{\mathbb{R}}H_{\ell}(Y_{j_{1}}(s))\psi(\frac{t-s}{2^{j_{2}}})ds\label{eq:decom_2nd_scatQ}
\\\notag=& 2^{-j_{2}}\sigma_{j_{1}}^{-\ell}
\int_{\mathbb{R}}
\left\{\int^{'}_{\mathbb{R}^{\ell}}e^{i(\lambda_{1}+\ldots+\lambda_{\ell})s}\left[\prod_{k=1}^{\ell}{\Psi(2^{j_{1}}\lambda_{k})\sqrt{f_{X}(\lambda_{k})}}\right]
W(d\lambda_{1})\cdots W(d\lambda_{\ell})\right\}\psi(\frac{t-s}{2^{j_{2}}})ds
\\\label{df:Z}=& \sigma_{j_{1}}^{-\ell}
\int^{'}_{\mathbb{R}^{\ell}}e^{i(\lambda_{1}+\cdots+\lambda_{\ell})t}\left[\prod_{k=1}^{\ell}{\Psi(2^{j_{1}}\lambda_{k})\sqrt{f_{X}(\lambda_{k})}}\right]
\Psi(2^{j_{2}}(\lambda_{1}+\cdots+\lambda_{\ell}))W(d\lambda_{1})\cdots W(d\lambda_{\ell}).\notag
\end{align}
The last equality follows from a stochastic Fubini theorem \cite[Theorem 2.1]{pipiras2010regularization} (see also \cite{lechiheb2018wiener}). By a careful expansion and analysis of \eqref{eq:decom_2nd_scatQ}, we could understand the behavior of the first order NAST with the nonlinear function $A_1$.
Next, by handling $A_2$ by the continuous mapping theorem or the delta method, we are able to find the scaling limits of the second-order NAST.

\subsection{Proof of Theorem 1}

For any $M\in \mathbb{N}$ and any set of real numbers
$\{a_{1},a_{2},\ldots,a_{M}\}$, denote
\begin{align}
\xi_{j_{2}}:=
\overset{M}{\underset{k=1}{\sum}}a_{k}2^{j_{2}/2}U^{A_{1}}[j_{1}]X\star \psi_{j_{2}}(2^{j_{2}}t_{k}),
\end{align}
where $t_{1},\cdots,t_{M}\in \mathbb{R}$
are arbitrary.
By (\ref{eq:decom_2nd_scat}) and recalling $\mathbf{r} = \textup{rank}(A_{1}(\sigma_{j_{1}}\cdot))$,
$\xi_{j_{2}}$ can be expressed as follows
\begin{equation}\label{proofthm:weaksmall}
\xi_{j_{2}} = \xi_{j_{2},\leq N}+\xi_{j_{2},>N},
\end{equation}
where $N\geq \mathbf{r}$,
\begin{equation}\label{df:smallN}
\xi_{j_{2},\leq N} = 2^{j_{2}/2}
\overset{M}{\underset{k=1}{\sum}}a_{k}
\sum_{\ell=\mathbf{r}}^{N}
\frac{C_{\sigma_{j_{1}},\ell}}{\sqrt{\ell!}}Z_{\ell}(2^{j_{2}}t_{k})\,,
\end{equation}
\begin{equation}\label{df:largeN}
\xi_{j_{2},> N} = 2^{j_{2}/2}
\overset{M}{\underset{k=1}{\sum}}a_{k}
\sum_{\ell=N+1}^{\infty}
\frac{C_{\sigma_{j_{1}},\ell}}{\sqrt{\ell!}}Z_{\ell}(2^{j_{2}}t_{k})\,,
\end{equation}
and $Z_\ell$ is defined in \eqref{eq:decom_2nd_scatQ}.
Also, define a truncated version of the limiting process $V_{1}$ as follows
\begin{align}
V_{\leq N}(t)=
\kappa_{N}\int_{\mathbb{R}}
e^{i\lambda t}
\Psi(\lambda)W(d\lambda),
\end{align}
where
\begin{align}\label{df:kappa_truncated}
\kappa_{N}=\left[\overset{N}{\underset{\ell=\mathbf{r}}{\sum}}
\sigma_{j_{1}}^{-2\ell}f_{X\star \psi_{j_{1}}}^{\star\ell}(0)
C^{2}_{\sigma_{j_{1}},\ell}\right]^{\frac{1}{2}}.
\end{align}
To get the proof of Theorem 1, we will prove that
\begin{equation}\label{proof:thm1:3claim}
\begin{array}{ll}
\textup{(a)}\ \underset{N\rightarrow\infty}{\lim}\ \underset{j_{2}\rightarrow\infty}{\lim}\mathbb E[\xi_{j_{2},>N}^2]=0;\\
\textup{(b)}\ \textup{The limit of}\ \kappa_{N}\ \textup{exists when}\ N\rightarrow\infty;\\
\textup{(c)}\ \textup{For any}\ N\geq\mathbf{r}, \xi_{j_{2},\leq N}\ \textup{converges in distribution to}\ \overset{M}{\underset{k=1}{\sum}}a_{k}V_{\leq N}(t_{k})\ \textup{as}\ j_{2}\rightarrow\infty.
\end{array}
\end{equation}
By a modification of Slustky's argument (see Lemma \ref{lemma:slustky} in the appendix), (a)-(c) imply that
$\xi_{j_{2}}$ converges in distribution to $\overset{M}{\underset{k=1}{\sum}}a_{k}V_{1}(t_{k})$ when $j_{2}\rightarrow\infty$.
From Lemma 2 in the appendix, we can see that when $N\rightarrow\infty$, the series (\ref{df:kappa_truncated}) can be expressed in terms of the integral of the covariance function of $A_{1}(X\star \psi_{j_{1}})$.
\\


\noindent{\it Proof of (a) and (b):}
In the following, we prove that $\mathbb E[\xi_{j_{2},>N}^2]$ can be made arbitrarily small as long as $j_{2}$ and $N$ are chosen sufficiently large.
Meanwhile, we will see that
$$
\overset{\infty}{\underset{\ell=N+1}{\sum}}
\sigma_{j_{1}}^{-2\ell}f_{X\star \psi_{j_{1}}}^{\star\ell}(0)
C^{2}_{\sigma_{j_{1}},\ell}\rightarrow 0
$$
when $N\rightarrow \infty$, which implies (b).

By (\ref{expectionhermite}), we know that for any $s, t \in \mathbb{R}$, $\mathbb E[Z_{\ell_{1}}(s)Z_{\ell_{2}}(t)] = 0$ if
$\ell_{1}\neq \ell_{2}$. Hence,
\begin{align}\notag
\mathbb E[\xi_{j_{2},>N}^2]=& \mathbb E\left[\left(2^{j_{2}/2}
\overset{M}{\underset{k=1}{\sum}}a_{k}
\sum_{\ell=N+1}^{\infty}
\frac{C_{\sigma_{j_{1}},\ell}}{\sqrt{\ell!}}Z_{\ell}(2^{j_{2}}t_{k})\right)^{2}
\right]
\\\label{eq:N_large}=& 2^{j_{2}}\overset{M}{\underset{k=1}{\sum}}\overset{M}{\underset{h=1}{\sum}}a_{k}a_{h}
\sum_{\ell=N+1}^{\infty}\frac{C^{2}_{\sigma_{j_{1}},\ell}}{\ell!}\mathbb E\left[Z_{\ell}(2^{j_{2}}t_{k})Z_{\ell}(2^{j_{2}}t_{h})\right].
\end{align}
By the definition of $Z_{\ell}$ in (\ref{eq:decom_2nd_scatQ}) and the orthogonal property of the Gaussian random measure (\ref{ortho}),
we know that
\begin{align}\notag
&2^{j_{2}}\mathbb E\left[Z_{\ell}(2^{j_{2}}t_{k})Z_{\ell}(2^{j_{2}}t_{h})\right]
\\\notag=&
2^{j_{2}}\sigma_{j_{1}}^{-2\ell}\ell!
\int_{\mathbb{R}^{\ell}}e^{i2^{j_{2}}(\lambda_{1}+\ldots+\lambda_{\ell})(t_{k}-t_{h})}
\left[\prod_{k=1}^{\ell}{|\Psi(2^{j_{1}}\lambda_{k})|^{2}f_{X}(\lambda_{k})}\right]
|\Psi(2^{j_{2}}(\lambda_{1}+\ldots+\lambda_{\ell}))|^{2}d\lambda_{1}\cdots d\lambda_{\ell}.
\end{align}
By considering the change of variables $\eta_j=\sum_{i=1}^j\lambda_i$ for $j=1,\ldots,\ell-1$ and $\eta=\sum_{i=1}^\ell \lambda_i$,
the expectation above can be rewritten as
\begin{align}\notag
2^{j_{2}}\mathbb E\left[Z_{\ell}(2^{j_{2}}t_{k})Z_{\ell}(2^{j_{2}}t_{h})\right]
=&
2^{j_{2}}\sigma_{j_{1}}^{-2\ell}\ell!
\int_{\mathbb{R}}e^{i2^{j_{2}}\eta(t_{k}-t_{h})}
f_{X\star\psi_{j_{1}}}^{\star\ell}(\eta)
|\Psi(2^{j_{2}}\eta)|^{2}d\eta
\\\label{eq:ZZ}
=& \sigma_{j_{1}}^{-2\ell}\ell!
\int_{\mathbb{R}}e^{i\eta(t_{k}-t_{h})}
f_{X\star\psi_{j_{1}}}^{\star\ell}(2^{-j_{2}}\eta)
|\Psi(\eta)|^{2}d\eta\,.
\end{align}
By combining (\ref{eq:N_large}) and (\ref{eq:ZZ}),
\begin{align}\notag
\mathbb E[\xi_{j_{2},>N}^2]
=& \overset{M}{\underset{k=1}{\sum}}\overset{M}{\underset{h=1}{\sum}}a_{k}a_{h}
\sum_{\ell=N+1}^{\infty}C^{2}_{\sigma_{j_{1}},\ell}\sigma_{j_{1}}^{-2\ell}
\int_{\mathbb{R}}e^{i\eta(t_{k}-t_{h})}
f_{X\star\psi_{j_{1}}}^{\star\ell}(2^{-j_{2}}\eta)
|\Psi(\eta)|^{2}d\eta
\\\notag=&\sum_{\ell=N+1}^{\infty}C^{2}_{\sigma_{j_{1}},\ell}\sigma_{j_{1}}^{-2\ell}\int_{\mathbb{R}}|\overset{M}{\underset{k=1}{\sum}}a_{k}e^{i\eta t_{k}}|^{2}
f_{X\star\psi_{j_{1}}}^{\star\ell}(2^{-j_{2}}\eta)
|\Psi(\eta)|^{2}d\eta
\\\label{eq:N_large2}=&
\int_{\mathbb{R}}|\overset{M}{\underset{k=1}{\sum}}a_{k}e^{i\eta t_{k}}|^{2}
\left[\sum_{\ell=N+1}^{\infty}C^{2}_{\sigma_{j_{1}},\ell}\sigma_{j_{1}}^{-2\ell}f_{X\star\psi_{j_{1}}}^{\star\ell}(2^{-j_{2}}\eta)\right]
|\Psi(\eta)|^{2}d\eta,
\end{align}
where the last equality follows from applying the monotone convergence theorem to change the order of summation and integration.
We use the nonnegativity of the spectral density function $f_{X\star\psi_{j_{1}}}$ to estimate the integrand as follows:
\begin{align}\notag
f_{X\star\psi_{j_{1}}}^{\star\ell}(\lambda) = \int_{\mathbb{R}}f_{X\star\psi_{j_{1}}}^{\star(\ell-1)}(\lambda-\zeta)f_{X\star\psi_{j_{1}}}(\zeta)d\zeta
\leq \left[\underset{\zeta\in \mathbb{R}}{\textup{sup}}f_{X\star\psi_{j_{1}}}^{\star(\ell-1)}(\zeta)\right] \sigma_{j_{1}}^{2}.
\end{align}
It implies that for all $\ell\geq N+1$,
\begin{align}\notag
\underset{\zeta\in \mathbb{R}}{\textup{sup}}f_{X\star\psi_{j_{1}}}^{\star\ell}(\zeta)
\leq \left[\underset{\zeta\in \mathbb{R}}{\textup{sup}}f_{X\star\psi_{j_{1}}}^{\star N}(\zeta)\right] \sigma_{j_{1}}^{2(\ell-N)}.
\end{align}
Hence, the second summation in (\ref{eq:N_large2}) can be estimated as follows
\begin{align}\notag
\sum_{\ell=N+1}^{\infty}C^{2}_{\sigma_{j_{1}},\ell}\sigma_{j_{1}}^{-2\ell}f_{X\star\psi_{j_{1}}}^{\star\ell}(2^{-j_{2}}\eta)
\leq \sigma_{j_{1}}^{-2N} \left[\underset{\zeta\in \mathbb{R}}{\textup{sup}}f_{X\star\psi_{j_{1}}}^{\star N}(\zeta)\right]\sum_{\ell=N+1}^{\infty}C^{2}_{\sigma_{j_{1}},\ell}<\infty,
\end{align}
where the finiteness of $\underset{\zeta\in \mathbb{R}}{\textup{sup}}f_{X\star\psi_{j_{1}}}^{\star N}(\zeta)$ for $N\geq \mathbf{r}$
comes form Lemma \ref{lemma:1} in the appendix.
The Lebesgue dominated convergence theorem and the continuity of $f_{X\star\psi_{j_{1}}}^{\star\ell}$ for $\ell\geq \mathbf{r}$ (Case 3 of Lemma \ref{lemma:1} in the appendix) imply that
\begin{align}\notag
\underset{j_{2}\rightarrow\infty}{\lim}\mathbb E[\xi_{j_{2},>N}^2] =& \int_{\mathbb{R}}|\overset{M}{\underset{k=1}{\sum}}a_{k}e^{i\eta t_{k}}|^{2}
\left[\sum_{\ell=N+1}^{\infty}C^{2}_{\sigma_{j_{1}},\ell}\sigma_{j_{1}}^{-2\ell}\underset{j_{2}\rightarrow\infty}{\lim}f_{X\star\psi_{j_{1}}}^{\star\ell}(2^{-j_{2}}\eta)\right]
|\Psi(\eta)|^{2}d\eta
\\\label{eq:large_N_part1}=& \int_{\mathbb{R}}|\overset{M}{\underset{k=1}{\sum}}a_{k}e^{i\eta t_{k}}|^{2}
\left[\sum_{\ell=N+1}^{\infty}C^{2}_{\sigma_{j_{1}},\ell}\sigma_{j_{1}}^{-2\ell}f_{X\star\psi_{j_{1}}}^{\star\ell}(0)\right]
|\Psi(\eta)|^{2}d\eta.
\end{align}
Observe that
\begin{align}\notag
f^{\star\ell}_{X\star\psi_{j_{1}}}(0) =& \int_{\mathbb{R}} f^{\star \mathbf{r}}_{X\star\psi_{j_{1}}}(-\lambda) f^{\star(\ell-\mathbf{r})}_{X\star\psi_{j_{1}}}(\lambda) d\lambda
\\\notag\leq&
\left[\underset{\zeta\in \mathbb{R}}{\textup{max}} f^{\star \mathbf{r}}_{X\star\psi_{j_{1}}}(\zeta)\right]
\int_{\mathbb{R}}
f^{\star(\ell-\mathbf{r})}_{X\star\psi_{j_{1}}}(\lambda) d\lambda=
\left[\underset{\zeta\in \mathbb{R}}{\textup{max}} f^{\star \mathbf{r}}_{X\star\psi_{j_{1}}}(\zeta)\right] \sigma_{j_{1}}^{2(\ell-\mathbf{r})},
\end{align}
where the last equality follows by
the convolution theorem
\begin{equation}\label{relation_power_cov_spectral}
\left[R_{X\star\psi_{j_{1}}}(t)\right]^{\ell-\mathbf{r}} = \int_{\mathbb{R}} e^{i\eta t} f^{\star(\ell-\mathbf{r})}_{X\star\psi_{j_{1}}}(\eta)d\eta\,
\end{equation}
when $\ell\geq \mathbf{r}$,
and substituting $t = 0$.
Hence, the summation in (\ref{eq:large_N_part1}) has the following estimate
\begin{align}\label{proof_b_finish_o2}
\sum_{\ell=N+1}^{\infty}C^{2}_{\sigma_{j_{1}},\ell}\sigma_{j_{1}}^{-2\ell}f_{X\star\psi_{j_{1}}}^{\star\ell}(0)
\leq \sigma_{j_{1}}^{-2\mathbf{r}}\left[\underset{\zeta\in \mathbb{R}}{\textup{max}} f^{\star \mathbf{r}}_{X\star\psi_{j_{1}}}(\zeta)\right]\sum_{\ell=N+1}^{\infty}C^{2}_{\sigma_{j_{1}},\ell},
\end{align}
which proves the claim (b) because $\overset{\infty}{\underset{\ell=\mathbf{r}}{\sum}}C^{2}_{\sigma_{j_{1}},\ell}<\infty$.
By (\ref{proof_b_finish_o2}), (\ref{eq:large_N_part1}) can be rewritten as
\begin{align}\label{proofthm:weaksmall1}
\underset{j_{2}\rightarrow\infty}{\lim}
\mathbb E[\xi_{j_{2},>N}^2]
\leq \sigma_{j_{1}}^{-2\mathbf{r}}\left[\underset{\zeta\in \mathbb{R}}{\textup{max}} f^{\star \mathbf{r}}_{X\star\psi_{j_{1}}}(\zeta)\right]\left[
\int_{\mathbb{R}}|\overset{M}{\underset{k=1}{\sum}}a_{k}e^{i\eta t_{k}}|^{2}
|\Psi(\eta)|^{2}d\eta
\right]\sum_{\ell=N+1}^{\infty}C^{2}_{\sigma_{j_{1}},\ell},
\end{align}
which implies the claim (a).\\
%
%

\noindent{\it Proof of (c):}
We apply the method of moments \cite{breuer1983central} (see also \cite[Theorem 6.5]{dasgupta2008asymptotic}) to prove
the claim (c). Because the linear combination of $V_{\leq N}(t)$ has the Gaussian distribution, which can be uniquely determined by its moments, it suffices to prove that
\begin{align}\label{Markovmethod2}
\underset{j_{2}\rightarrow \infty}{\textup{lim}}\ \mathbb{E}\xi_{j_{2},\leq N}^{p}=
\left
\{\begin{array}{lr}0 &\textup{if}\ p\ \textup{is odd},
\\ (p-1)!!
\left\{\mathbb{E}\Big[\big(\overset{M}{\underset{k=1}{\sum}}a_{k}V_{\leq N}(t_{k})\big)^{2}\Big]\right\}^{p/2}
 &\textup{if}\ p\ \textup{is even}.
\end{array}
\right.
\end{align}
From the definition of
$\xi_{j_{2},\leq
N}$ in (\ref{df:smallN}),
\begin{align}\notag
\mathbb{E}(\xi_{j_{2},\leq N})^{p}=&2^{\frac{pj_{2}}{2}}
\overset{M}{\underset{{k_{1},\cdots,k_{p}=1}}{\sum}}\
\overset{N}{\underset{\ell_{1},\cdots,\ell_{p}=\mathbf{r}}{\sum}}
\Big[\overset{p}{\underset{i=1}{\prod}}a_{k_{i}} \frac{C_{\sigma_{j_{1}},\ell_{i}}}{\sqrt{\ell_{i}!}}
\Big]
\\\label{smallN_exp}
\times&
\int_{\mathbb{R}^{p}}
\Big[
\overset{p}{\underset{i=1}{\prod}}
2^{-j_{2}}\psi(\frac{2^{j_{2}}t_{k_{i}}-s_{i}}{2^{j_{2}}})\Big]
\Big[\mathbb{E}\overset{p}{\underset{i=1}{\prod}} H_{\ell_{i}}(Y_{j_{1}}(s_{i}))\Big]
ds_{1}\cdots ds_{p}.
\end{align}

To analyze $\mathbb{E}(\xi_{j_{2},\leq N})^{p}$,
we employ the {\em diagram
method} (see \cite{breuer1983central} or  \cite[p.72]{ivanov2012statistical}). A graph $\Gamma$ with $\ell_{1}+\cdots+\ell_{p}$ vertices is
called a {\em complete diagram
of order ($\ell_{1},\cdots,\ell_{p}$)} if:\\
(a) the set of vertices $V$ of the graph $\Gamma$ is of the form $V=\overset{p}{\underset{i=1}{\cup}}W_{i}$,
where $W_{i}=\{(i,\ell):1\leq \ell\leq \ell_{i}\}$ is the $i$-th level of the graph $\Gamma$;
\\
(b) each vertex is of degree 1, that is, each vertex is just an endpoint of one edge;
\\
(c) if $((i,\ell),\ (i',\ell^{'}))$ is an edge, then $i\neq i'$, that is, the edges of the graph $\Gamma$
connect only different levels.

Let $\mathrm{T}=\mathrm{T}(\ell_{1},\cdots,\ell_{p})$ be a set of
complete diagrams of order ($\ell_{1},\cdots,\ell_{p}$).
Denote by $E(\Gamma)$ the set of edges of the graph $\Gamma\in
\mathrm{T}$.
For the edge $e=((i,\ell),\
(i',\ell^{'}))\in E(\Gamma)$ with $i<i'$, $1\leq
\ell\leq \ell_{i}$ and $1\leq \ell^{'}\leq \ell_{i'}$, we set
$d_{1}(e)=i$ and $d_{2}(e)=i'$.
We call a complete diagram $\Gamma$ {\it
regular} if its levels can be split into pairs in such a manner
that no edge connects the levels belonging to different pairs (see Figure \ref{regular_diagram}).
Denote by $\mathrm{T}^{*}=\mathrm{T}^{*}(\ell_{1},\cdots,\ell_{p})$ the
set of all regular diagrams in $\mathrm{T}$.
If
$\Gamma\in \mathrm{T}^{*}$ is a regular diagram,
then $p$ is even and $\Gamma$ can be divided into $p/2$ sub-diagrams
(denoted by
$\Gamma_{1},\cdots,\Gamma_{p/2}$), which cannot be separated
again; in this case, we naturally define $d_{1}(\Gamma_{r})\equiv
d_{1}(e)$ and $d_{2}(\Gamma_{r})\equiv d_{2}(e)$ for any $e\in
E(\Gamma_{r}),\ r=1,\ldots,p/2$. We denote $\# E(\Gamma)$
(resp. $\# E(\Gamma_{r})$) the number of edges belonged to
the specific diagram $\Gamma$ (resp. the sub-diagram
$\Gamma_{r}$).
\begin{figure}
\centering
\subfigure[][Regular complete diagram
($\# A_{1,2}=3$, $\# A_{1,3} =0$, $\# A_{1,4} =0$, $\# A_{2,3} =0$, $\# A_{2,4} =0$, $\# A_{3,4}=4$)]{\label{regular_diagram}\includegraphics[width=0.425\textwidth]{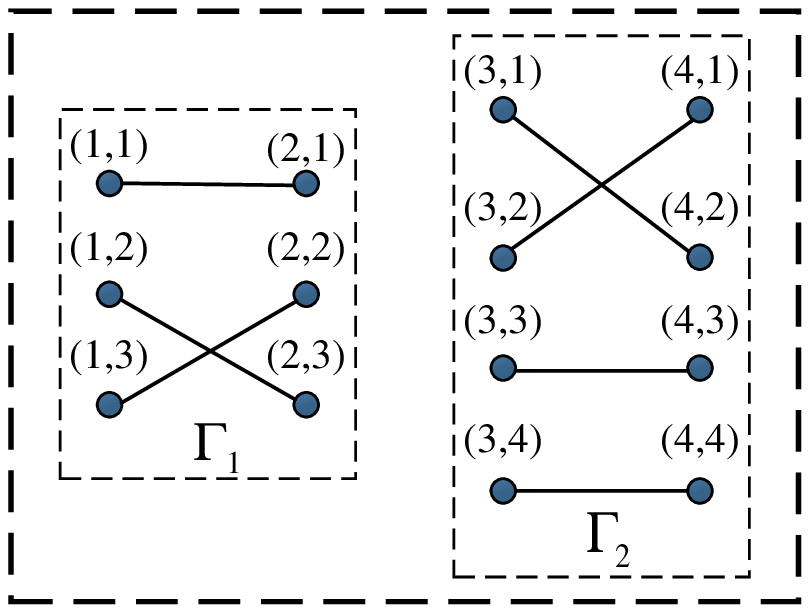}}
\hspace{0.2cm}
\subfigure[][Non-complete diagram
($\# A_{1,2}=2$, $\# A_{1,3} =0$, $\# A_{1,4} =1$, $\# A_{2,3} =1$, $\# A_{2,4} =0$, $\# A_{3,4}=3$)]{\label{nonregular_diagram}\includegraphics[width=0.402\textwidth]{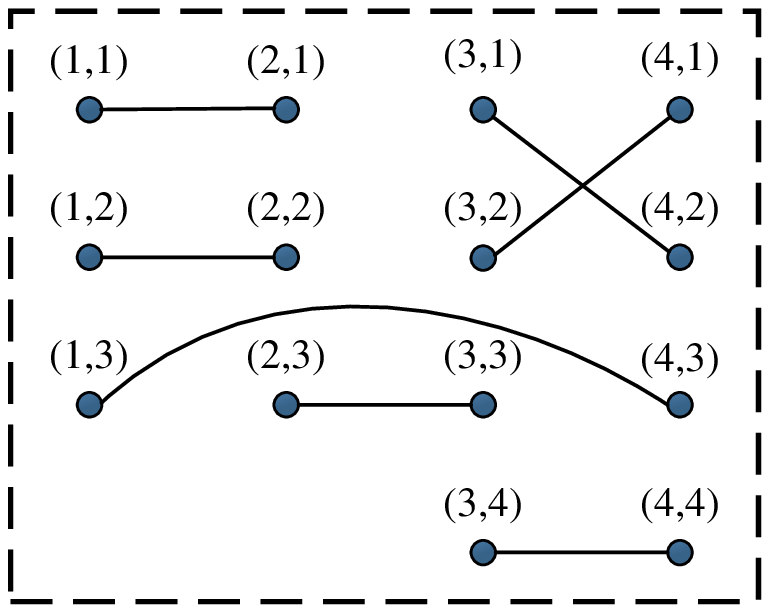}}
\caption{Illustration of complete and non-complete diagrams of order $(3,3,4,4)$. For integers $i<i'$, $\#A_{i,i'}$ means the number of edges connecting the $i$th and $i'$th levels.}
\label{diagram}
\end{figure}
Based on the above notation,
let
\begin{equation*}
D_{p}=\{(K,L):K=(k_{1},\cdots,k_{p}),1\leq k_{i}\leq M,\
L=(\ell_{1},\cdots, \ell_{p}),\ \mathbf{r}\leq \ell_{i}\leq N,\ i=1,\cdots,p\},
\end{equation*}
By \cite{major1981lecture}, (\ref{smallN_exp}) can be rewritten as
\begin{align}\label{proofthm:weaksmall6}
\mathbb{E}(\xi_{j_{2},\leq N})^{p}=
\underset{(K,L)\in D_{p}}{\sum}C(K,L)\underset{\Gamma\in \mathrm{T}^{*}}{\sum}F_{\Gamma}(K,L,j_{2})
+\underset{(K,L)\in D_{p}}{\sum}C(K,L)\underset{\Gamma\in \mathrm{T}\backslash \mathrm{T}^{*}}{\sum}F_{\Gamma}(K,L,j_{2}),
\end{align}
where
\begin{align}\notag
&C(K,L)=\overset{p}{\underset{i=1}{\prod}}a_{k_{i}} \frac{C_{\sigma_{j_{1}},\ell_{i}}}{\sqrt{\ell_{i}!}},
\\\label{proofthm:weakK(J,L)}
&F_{\Gamma}(K,L,j_{2})=2^{\frac{pj_{2}}{2}}
\int_{\mathbb{R}^{p}}
\Big[
\overset{p}{\underset{i=1}{\prod}}
2^{-j_{2}}\psi(\frac{2^{j_{2}}t_{k_{i}}-s_{i}}{2^{j_{2}}})\Big]
\Big[
\underset{e\in E(\Gamma)}{\prod}R_{Y_{j_{1}}}(s_{d_{1}(e)}-s_{d_{2}(e)})
\Big]ds_{1}\cdots ds_{p}.
\end{align}
To prove (\ref{Markovmethod2}), by (\ref{proofthm:weaksmall6}), it suffices to verify the following two claims:
\begin{align*}
\left\{
\begin{array}{l}
\textup{(c1)}\ \underset{j_{2}\rightarrow \infty}{\textup{lim}}
\underset{(K,L)\in D_{p}}{\sum}C(K,L)\underset{\Gamma\in \mathrm{T}^{*}}{\sum}F_{\Gamma}(K,L,j_{2})
=(p-1)!!
\Big\{\mathbb{E}\Big[\big(\overset{M}{\underset{k=1}{\sum}}a_{k}V_{N}(t_{k})\big)^{2}\Big]\Big\}^{p/2},
\\
\textup{(c2)}\ \underset{j_{2}\rightarrow \infty}{\textup{lim}}\underset{(K,L)\in D_{p}}{\sum}C(K,L)\underset{\Gamma\in \mathrm{T}\backslash \mathrm{T}^{*}}{\sum}F_{\Gamma}(K,L,j_{2})
=0.
\end{array}\right.
\end{align*}

\noindent{\it Proof of (c1):} If $\Gamma$ is a regular diagram in $\mathrm{T}^*(\ell_{1},\cdots,\ell_{p})$,
then $\Gamma$ has a unique
decomposition
$\Gamma=(\Gamma_{1},\ldots,\Gamma_{p/2})$, where
$\Gamma_{1},\ldots,\Gamma_{p/2}$
cannot be further decomposed.
Accordingly, $F_{\Gamma}(K,L,j_{2})$
can be rewritten as the following $p/2$ products
\begin{align}\label{proofthm:weaksmall7}
F_{\Gamma}(K,L,j_{2})=
2^{\frac{pj_{2}}{2}}
\overset{p/2}{\underset{r=1}{\prod}}\int_{\mathbb{R}^{2}}
2^{-2j_{2}}
\psi(\frac{2^{j_{2}}t_{k_{d_{1}(\Gamma_{r})}}-s}{2^{j_{2}}}) \psi(\frac{2^{j_{2}}t_{k_{d_{2}(\Gamma_{r})}}-s'}{2^{j_{2}}})
R_{Y_{j_{1}}}^{\# E(\Gamma_{r})}(s-s')
\ ds\ ds'.
\end{align}
By
\begin{align*}
R_{Y_{j_{1}}}^{\# E(\Gamma_{r})}(s-s')
=\int_{\mathbb{R}}e^{i\lambda(s-s')}f_{Y_{j_{1}}}^{\star\# E(\Gamma_{r})}(\lambda)d\lambda
\end{align*}
for $r=1,\cdots,p/2$,
\begin{align*}
&\int_{\mathbb{R}}e^{i\lambda s}
2^{-j_{2}}\psi(\frac{2^{j_{2}}t_{k_{d_{1}(\Gamma_{r})}}-s}{2^{j_{2}}})
ds
=e^{i2^{j_{2}}\lambda t_{k_{d_{1}(\Gamma_{r})}}}
\Psi(2^{j_{2}}\lambda),
\end{align*}
and
\begin{align*}
&\int_{\mathbb{R}}e^{-i\lambda s'}
2^{-j_{2}}\psi(\frac{2^{j_{2}}t_{k_{d_{2}(\Gamma_{r})}}-s'}{2^{j_{2}}})
ds'
=e^{-i2^{j_{2}}\lambda t_{k_{d_{2}(\Gamma_{r})}}}
\overline{\Psi(2^{j_{2}}\lambda)},
\end{align*}
we rewrite (\ref{proofthm:weaksmall7})  as  follows:
\begin{align}\notag
F_{\Gamma}(K,L,j_{2})
=&
2^{\frac{pj_{2}}{2}}
\overset{p/2}{\underset{r=1}{\prod}}
\Big[\int_{\mathbb{R}}
e^{i2^{j_{2}}\lambda(t_{k_{d_{1}(\Gamma_{r})}}-t_{k_{d_{2}(\Gamma_{r})}})}
|\Psi(2^{j_{2}}\lambda)|^{2}
 f_{Y_{j_{1}}}^{\star\# E(\Gamma_{r})}(\lambda)d\lambda\Big]
 \\\notag=&
\overset{p/2}{\underset{r=1}{\prod}}
\Big[\int_{\mathbb{R}}
e^{i\lambda(t_{k_{d_{1}(\Gamma_{r})}}-t_{k_{d_{2}(\Gamma_{r})}})}
|\Psi(\lambda)|^{2}
 f_{Y_{j_{1}}}^{\star\# E(\Gamma_{r})}(2^{-j_{2}}\lambda)d\lambda\Big]\,,
\end{align}
which tends to
\begin{align}
\label{proofthm:weaksmall7s}
\overset{p/2}{\underset{r=1}{\prod}}\Big[f_{Y_{j_{1}}}^{\star\# E(\Gamma_{r})}(0)
\int_{\mathbb{R}}
e^{i\lambda(t_{k_{d_{1}(\Gamma_{r})}}-t_{k_{d_{2}(\Gamma_{r})}})}
|\Psi(\lambda)|^{2}
 d\lambda\Big]
\end{align}
when $j_{2}$ tends to infinity,
where $f_{Y_{j_{1}}}^{\star\# E(\Gamma_{r})}(0)<\infty$ follows from Lemma \ref{lemma:1} in the appendix because $\# E(\Gamma_{r})\geq\mathbf{r}$ and $(2\alpha+\beta)\mathbf{r}>1$.
Meanwhile, because $\Gamma$ is a  regular
diagram in $\mathrm{T}(L)$, $C(K,L)$ can be rewritten as follows:
\begin{equation}\label{proofthm:regularcoeff}
C(K,L)=\overset{p/2}{\underset{r=1}{\prod}}
a_{k_{d_{1}(\Gamma_{r})}} a_{k_{d_{2}(\Gamma_{r})}}
\frac{C^{2}_{\sigma_{j_{1}},\# E(\Gamma_{r})}}{\# E(\Gamma_{r})!}.
\end{equation}
Combining (\ref{proofthm:weaksmall7s}) and (\ref{proofthm:regularcoeff}) yields
\begin{align}\notag
&\underset{j_{2}\rightarrow \infty}{\textup{lim}}
\underset{(K,L)\in D_{p}}{\sum}C(K,L)\underset{\Gamma\in \mathrm{T}^{*}}{\sum}F_{\Gamma}(K,L,j_{2})
\\\label{proofthm:weaksmall8}
=&
\underset{(K,L)\in D_{p}}{\sum}
\ \underset{\Gamma\in \mathrm{T}^{*}}{\sum}
\Big[
\overset{p/2}{\underset{r=1}{\prod}}
a_{k_{d_{1}(\Gamma_{r})}} a_{k_{d_{2}(\Gamma_{r})}}
\int_{\mathbb{R}}
e^{i\lambda(t_{k_{d_{1}(\Gamma_{r})}}-t_{k_{d_{2}(\Gamma_{r})}})}
|\Psi(\lambda)|^{2}
d\lambda
\Big]
\Big[
\overset{p/2}{\underset{r=1}{\prod}}
f_{Y_{j_{1}}}^{\star\# E(\Gamma_{r})}(0)
\frac{C^{2}_{\sigma_{j_{1}},\# E(\Gamma_{r})}}{\# E(\Gamma_{r})!}
\Big].
\end{align}
By changing the order of summation,
$(\ref{proofthm:weaksmall8})$ can be rewritten as follows:
\begin{align}\notag
&\underset{j_{2}\rightarrow \infty}{\textup{lim}}
\underset{(K,L)\in D_{p}}{\sum}C(K,L)\underset{\Gamma\in \mathrm{T}^{*}}{\sum}F_{\Gamma}(K,L,j_{2})
\\\notag
=&
\underset{L}{\sum}
\underset{\Gamma\in \mathrm{T}^{*}}{\sum}
\Big[\underset{K}{\sum}
\overset{p/2}{\underset{r=1}{\prod}}
a_{k_{d_{1}(\Gamma_{r})}} a_{k_{d_{2}(\Gamma_{r})}}
\int_{\mathbb{R}}
e^{i\lambda(t_{k_{d_{1}(\Gamma_{r})}}-t_{k_{d_{2}(\Gamma_{r})}})}
|\Psi(\lambda)|^{2}
d\lambda
\Big]
\Big[
\overset{p/2}{\underset{r=1}{\prod}}
f_{Y_{j_{1}}}^{\star\# E(\Gamma_{i})}(0)
\frac{C^{2}_{\sigma_{j_{1}},\# E(\Gamma_{r})}}{\# E(\Gamma_{r})!}
\Big]
\\\label{proofthm:regualr1}
=&
\Big[
\overset{M}{\underset{k,h=1}{\sum}}a_{k} a_{h}
\int_{\mathbb{R}}
e^{i\lambda(t_{k}-t_{h})}
|\Psi(\lambda)|^{2}
d\lambda
\Big]^{p/2}
\underset{L}{\sum}
\underset{\Gamma\in \mathrm{T}^{*}}{\sum}
\Big[
\overset{p/2}{\underset{r=1}{\prod}}
f_{Y_{j_{1}}}^{\star\# E(\Gamma_{r})}(0)
\frac{C^{2}_{\sigma_{j_{1}},\# E(\Gamma_{r})}}{\# E(\Gamma_{r})!}
\Big].
\end{align}
To handle the summation term
$\underset{L}{\sum}
\underset{\Gamma\in \mathrm{T}^{*}}{\sum}[\cdots]$ in (\ref{proofthm:regualr1}), we note that $
\overset{p/2}{\underset{r=1}{\prod}} f_{Y_{j_{1}}}^{\star\#
E(\Gamma_{r})}(0) \frac{C^{2}_{\#
E(\Gamma_{r})}}{\# E(\Gamma_{r})!} $ only depends on
$\{\#
E(\Gamma_{r}),r=1,\ldots,p/2\}$, but not on the internal structures of
the sub-diagrams $\Gamma_{r},\ r=1,\ldots,p/2$.
Let $s$ be
the number of different integers $\overline{\ell}_{1},\ldots,\overline{\ell}_{s}$ in
$\{\ell_{1},\ldots,\ell_{p}\}$ with $\mathbf{r} \leq \overline{\ell}_{1}<\ldots <\overline{\ell}_{s}\leq
N$, where $1\leq s\leq p/2$. It
implies that the set $\{\ell_{1},\ldots,\ell_{p}\}$ can be split into
$s$ subsets $Q_{1},\ldots,Q_{s}$ and all elements within $Q_{i}$
have the common value $\overline{\ell}_{i},\ i=1,\ldots,s.$ Denote the number of
{\it pairs} within each subset $Q_{i}$ by $q_{i}$,
which satisfies $q_{i}\geq 1$ for $i\in\{1,\ldots,s\}$ and
$q_{1}+\cdots+q_{s}=p/2$.
Based on the notation introduced above,
\begin{align}\notag
&\underset{L}{\sum}
\underset{\Gamma\in \mathrm{T}^{*}}{\sum}
\Big[
\overset{p/2}{\underset{r=1}{\prod}}
f_{Y_{j_{1}}}^{\star\# E(\Gamma_{r})}(0)
\frac{C^{2}_{\sigma_{j_{1}},\# E(\Gamma_{r})}}{\# E(\Gamma_{r})!}
\Big]
\\\notag
=&\underset{1\leq s\leq p/2}{\sum} s! \ {\underset{\mathbf{r}\leq
\overline{\ell}_{1}<\cdots<\overline{\ell}_{s}\leq N}{\sum}} \
\underset{q_{1}+\cdots+q_{s}=p/2}{\sum}
\frac{p!}{2^{p/2}q_{1}!\cdots q_{s}!}(\overline{\ell}_{1}!)^{q_{1}}\cdots
(\overline{\ell}_{s}!)^{q_{s}} \Big[ \overset{s}{\underset{i=1}{\prod}} \big(
f_{Y_{j_{1}}}^{\star\overline{\ell}_{i}}(0) \frac{C^{2}_{\sigma_{j_{1}},\overline{\ell}_{i}}}{\overline{\ell}_{i}!} \big)^{q_{i}}
\Big]
\\\notag
=&(p-1)!! \underset{1\leq s\leq p/2}{\sum} s! \
{\underset{\mathbf{r}\leq \overline{\ell}_{1}<\cdots<\overline{\ell}_{s}\leq N}{\sum}} \
\underset{q_{1}+\cdots+q_{s}=p/2}{\sum} \frac{(p/2)!}{q_{1}!\cdots
q_{s}!} \Big[ \overset{s}{\underset{i=1}{\prod}} \big(
f_{Y_{j_{1}}}^{\star\overline{\ell}_{i}}(0) C^{2}_{\sigma_{j_{1}},\overline{\ell}_{i}} \big)^{q_{i}} \Big]
\\\label{proofthm:weaksmalla}
=&
(p-1)!!
\Big[
\overset{N}{\underset{\ell=\mathbf{r}}{\sum}}
f_{Y_{j_{1}}}^{\star\ell}(0)
C^{2}_{\sigma_{j_{1}},\ell}
\Big]^{p/2},
\end{align}
where $(p-1)!! = 1\times 3\times 5 \times\cdots \times(p-1).$
Inserting (\ref{proofthm:weaksmalla}) into (\ref{proofthm:regualr1})
yields
\begin{align}\notag
&\underset{j_{2}\rightarrow \infty}{\textup{lim}}
\underset{(K,L)\in D_{p}}{\sum}C(K,L)\underset{\Gamma\in \mathrm{T}^{*}}{\sum}F_{\Gamma}(K,L,j_{2})
\\\label{proofthm:regualr2}
=&
(p-1)!!
\Big[
\overset{M}{\underset{k,h=1}{\sum}}a_{k} a_{h}
\int_{\mathbb{R}}
e^{i\lambda(t_{k}-t_{h})}
|\Psi(\lambda)|^{2}
d\lambda
\Big]^{p/2}
\Big[
\overset{N}{\underset{\ell=\mathbf{r}}{\sum}}
f_{Y_{j_{1}}}^{\star\ell}(0)
C^{2}_{\sigma_{j_{1}},\ell}
\Big]^{p/2}.
\end{align}
By the orthogonal property of the standard Gaussian random measure $W$ (see (\ref{sample path represent}))
and the identity $f_{Y_{j_{1}}}(\eta) = \sigma^{-2}_{j_{1}}f_{X\star \psi_{j_{1}}}(\eta)$,
which implies that $f^{\star\ell}_{Y_{j_{1}}} = \sigma^{-2\ell}_{j_{1}}f_{X\star \psi_{j_{1}}}^{\star\ell}$,
the right hand side of (\ref{proofthm:regualr2}) is equal to
\begin{align}
(p-1)!!
\Big[
\mathbb{E}
\Big(
\overset{M}{\underset{k=1}{\sum}}a_{k}
\kappa_{N}\int_{\mathbb{R}}
e^{i\lambda t_{k}}
\Psi(\lambda)W(d\lambda)
\Big)^{2}
\Big]^{p/2},
\end{align}
where $\kappa_{N}$ is defined in (\ref{df:kappa_truncated}). The proof of (c1) is complete.\\

\noindent{\it Proof of (c2)}:
Because the number of elements in the summation
$$
\underset{(K,L)\in D_{p}}{\sum}C(K,L)
\underset{\Gamma\in \mathrm{T}\backslash \mathrm{T}^{*}}{\sum}F_{\Gamma}(K,L,j_{2})
$$
is finite and $C(K,L)$ is independent of $j_{2}$, it  suffices to
show that $\underset{j_{2}\rightarrow \infty}{\textup{lim}}
F_{\Gamma}(K,L,j_{2})=0$ for each $\Gamma\in \mathrm{T}(\ell_{1},\cdots,\ell_{p})\backslash
\mathrm{T}^{*}$ and $(K,L)\in D_{p}$.
To prove this,
we first rearrange the $p$ levels of $\Gamma$ such that $\ell_{1}\leq \ell_{2}\leq \ell_{3}\leq \cdots \leq \ell_{p}.$
For any $i,i'\in \{1,\ldots,p\}$ with $i<i'$, let $A_{i,i'}$ be the set of edges, which connect the $i$th and $i'$th layers;
that is,
\begin{equation}\label{notationAB}
A_{i,i'}=\big\{e\in E(\Gamma)\mid d_{1}(e)=i,\
d_{2}(e)=i'\big\}.
\end{equation}
Also, we set that $A_{i,i'}= \emptyset$ for $i\geq i'$.
Let $\# A_{i,i'}$ be the cardinality of $A_{i,i'}$, i.e., the number of edges in
$A_{i,i'}$ (see Figure \ref{diagram}).
Denote
\[
\ell^{*} := \max\{\ell\in \mathbb{N}\cup\{0\}\mid \ell(2\alpha+\beta)\leq 1\}
\]
and
\begin{equation*}
\tilde{N} := \left\{\begin{array}{ll}
0\ & \textup{if}\ \ell^{*}=0,\\
\overset{\ell^{*}}{\underset{\ell=1}{\sum}}\left[1-\ell(2\alpha+\beta)\right]\#\{(i,i'): \#A_{i,i'}=\ell\}\ & \textup{if}\  \ell^{*}(2\alpha+\beta)< 1,\\
\overset{\ell^{*}-1}{\underset{\ell=1}{\sum}}\left[1-\ell(2\alpha+\beta)\right]\#\{(i,i'): \#A_{i,i'}=\ell\}+\varepsilon\#\{(i,i'): \#A_{i,i'}=\ell^{*}\}\ & \textup{if}\  \ell^{*}(2\alpha+\beta)= 1,
\end{array}\right.
\end{equation*}
where the parameter $\varepsilon>0$ will be specified in (\ref{boundary_case_Gaussian_limit}).
By the Fourier transform relationship between $R_{Y_{j_{1}}}^{\# A_{i,i'}}$
and $f_{Y_{j_{1}}}^{\star\# A_{i,i'}}$ (see also (\ref{relation_power_cov_spectral})) and a direct expansion, we have
\begin{align}\notag
&2^{j_{2}\left(-\frac{p}{2}+\#\{(i,i'): A_{i,i'}\neq \emptyset\}\right)}2^{-\tilde{N}j_{2}}F_{\Gamma}(K,L,j_{2})
\\\notag
=&2^{j_{2}\#\{(i,i'): A_{i,i'}\neq \emptyset\}}2^{-\tilde{N}j_{2}}
\int_{\mathbb{R}^{p}}
\Big[\overset{p}{\underset{i=1}{\prod}}2^{-j_{2}}\psi(\frac{2^{j_{2}}t_{k_{i}}-s_{i}}{2^{j_{2}}})\Big]
\Big[
\underset{(i,i'): A_{i,i'}\neq \emptyset}{\prod}
R_{Y_{j_{1}}}^{\# A_{i,i'}}(s_{i}-s_{i'})
\Big]ds_{1}\cdots ds_{p}
\\\notag
=&
2^{j_{2}\#\{(i,i'): A_{i,i'}\neq \emptyset\}}2^{-\tilde{N}j_{2}}
\int_{\mathbb{R}^{p}}
\Big[\overset{p}{\underset{i=1}{\prod}}2^{-j_{2}}\psi(\frac{2^{j_{2}}t_{k_{i}}-s_{i}}{2^{j_{2}}})\Big]
\Big[
\hspace{-0.2cm}\underset{(i,i'): A_{i,i'}\neq \emptyset}{\prod}\hspace{-0.1cm}
\int_{\mathbb{R}}e^{i\lambda_{i,i'}(s_{i}-s_{i'})}f_{Y_{j_{1}}}^{\star\# A_{i,i'}}(\lambda_{i,i'})d\lambda_{i,i'}
\Big]ds_{1}\cdots ds_{p}
\\\notag
=&2^{-\tilde{N}j_{2}}\int_{\mathbb{R}^{p}}
\Big[\overset{p}{\underset{i=1}{\prod}}\psi(t_{k_{i}}-s_{i})\Big]
\Big[
\underset{(i,i'): A_{i,i'}\neq \emptyset}{\prod}
\int_{\mathbb{R}}e^{i\lambda_{i,i'}(s_{i}-s_{i'})}f_{Y_{j_{1}}}^{\star\# A_{i,i'}}(2^{-j_{2}}\lambda_{i,i'})d\lambda_{i,i'}
\Big]ds_{1}\cdots ds_{p}
\\\notag=&
2^{-\tilde{N}j_{2}}\int_{\mathbb{R}^{p}}
\Big[\overset{p}{\underset{i=1}{\prod}}\psi(t_{k_{i}}-s_{i})\Big]
\int_{\mathbb{R}}\cdots \int_{\mathbb{R}} \overset{p}{\underset{i=1}{\prod}}\textup{exp}\left\{i s_{i}\left[\left(\underset{i': A_{i,i'}\neq \emptyset}{\sum}\lambda_{i,i'}\right)-\left(\underset{i': A_{i',i}\neq \emptyset}{\sum}\lambda_{i',i}\right)\right]\right\}
\\\notag&\times\left[\underset{(i,i'): A_{i,i'}\neq \emptyset}{\prod}f_{Y_{j_{1}}}^{\star\# A_{i,i'}}(2^{-j_{2}}\lambda_{i,i'})\ d\lambda_{i,i'}\right]
ds_{1}\cdots ds_{p}
\\\notag=&\int_{\mathbb{R}}\cdots \int_{\mathbb{R}}
\overset{p}{\underset{i=1}{\prod}}\textup{exp}\left\{i t_{k_{i}} \left[\Big(\underset{i': A_{i,i'}\neq \emptyset}{\sum}\lambda_{i,i'}\Big)-\Big(\underset{i': A_{i',i}\neq \emptyset}{\sum}\lambda_{i',i}\Big)\right]\right\}\Psi\left(\Big(\underset{i': A_{i,i'}\neq \emptyset}{\sum}\lambda_{i,i'}\Big)-\Big(\underset{i': A_{i',i}\neq \emptyset}{\sum}\lambda_{i',i}\Big)\right)
\\\label{nonregular_diagram_est}&\times\left[2^{-\tilde{N}j_{2}}\underset{(i,i'): A_{i,i'}\neq \emptyset}{\prod}f_{Y_{j_{1}}}^{\star\# A_{i,i'}}(2^{-j_{2}}\lambda_{i,i'})\ d\lambda_{i,i'}\right].
\end{align}
In the following, we study the limit of
$$H(j_{2}):=
2^{-\tilde{N}j_{2}} \underset{(i,i'): A_{i,i'}\neq \emptyset}{\prod}f_{Y_{j_{1}}}^{\star\# A_{i,i'}}(2^{-j_{2}}\lambda_{i,i'})=H_{1}(j_{2})H_{2}(j_{2}),
$$
where
$
H_{1}(j_{2})= 2^{-\tilde{N}j_{2}}\underset{(i,i'): \# A_{i,i'}\leq \ell^{*}}{\prod}
f^{\star\# A_{i,i'}}_{Y_{j_{1}}}(2^{-j_{2}}\lambda_{i,i'})
$
and
$
H_{2}(j_{2}) = \underset{(i,i'): \# A_{i,i'}\geq \ell^{*}+1}{\prod}f_{Y_{j_{1}}}^{\star\# A_{i,i'}}(2^{-j_{2}}\lambda_{i,i'})$.
By Lemma \ref{lemma:1} in the appendix, we directly have
\begin{align}\label{ell_large_case}
\underset{j_{2}\rightarrow\infty}{\lim}H_{2}(j_{2}) = \underset{(i,i'): \# A_{i,i'}\geq \ell^{*}+1}{\prod}f_{Y_{j_{1}}}^{\star\# A_{i,i'}}(0).
\end{align}
Next, we show that $H_{1}(j_{2})$ also converges when $j_{2}\rightarrow \infty$.
Note that
\begin{align*}
f_{Y_{j_{1}}}(\lambda) = f_{X}(\lambda) |\Psi(2^{j_{1}}\lambda)|^{2} =
\frac{C_{X}(\lambda)|C_{\Psi}(2^{j_{1}}\lambda)|^{2}}{|\lambda|^{1-(2\alpha+\beta)}}.
\end{align*}
If $\ell(2\alpha+\beta)<1$, by Lemma \ref{lemma:1} in the appendix, there exists a bounded function $B_{\ell}$ such that
\begin{align*}
f^{\star \ell}_{Y_{j_{1}}}(\lambda) =
\frac{B_{\ell}(\lambda)}{|\lambda|^{1-\ell(2\alpha+\beta)}}.
\end{align*}
In this case, for all $\lambda \neq 0$,
\begin{align}\label{ell_interior_case}
\underset{j_{2}\rightarrow\infty}{\overline{\lim}}2^{-j_{2}(1-\ell(2\alpha+\beta))}f^{\star \ell}_{Y_{j_{1}}}(2^{-j_{2}}\lambda) \leq
\frac{\|B_{\ell}\|_{\infty}}{|\lambda|^{1-\ell(2\alpha+\beta)}}.
\end{align}

If $\ell^{*}(2\alpha+\beta)=1$, by Lemma \ref{lemma:1} in the appendix, there exists a bounded function $B_{\ell^{*}}$ such that
\begin{align*}
f^{\star \ell^{*}}_{Y_{j_{1}}}(\lambda) =
B_{\ell^{*}}(\lambda) \ln(1+|\lambda|^{-1}).
\end{align*}
In this case,  for all $\lambda \neq 0$,
\begin{align}\label{ell_boundary_case}
\underset{j_{2}\rightarrow\infty}{\lim}2^{-\varepsilon j_{2}}f^{\star \ell^{*}}_{Y_{j_{1}}}(2^{-j_{2}}\lambda) =
0.
\end{align}
By (\ref{ell_large_case}), (\ref{ell_interior_case}), and (\ref{ell_boundary_case}),
we get that
$H(j_{2})$ converges when $j_{2}\rightarrow \infty$.
Thus, (\ref{nonregular_diagram_est}) implies that there exists a constant $C>0$ and a threshold $J<\infty$ such that
\begin{align}\label{F_upperbound}
F_{\Gamma}(K,L,j_{2})\leq C 2^{j_{2}\left(\frac{p}{2}-\#\{(i,i'): A_{i,i'}\neq \emptyset\}\right)}2^{\tilde{N}j_{2}}
\end{align}
for all $j_{2}\geq J$.
Therefore, if $\frac{p}{2}-\#\{(i,i'): A_{i,i'}\neq \emptyset\}+\tilde{N}<0$,  then we have
\begin{align}\notag
\underset{j_{2}\rightarrow\infty}{\lim} F_{\Gamma}(K,L,j_{2})= 0
\end{align}
as required.

Now, we claim that $\frac{p}{2}-\#\{(i,i'): A_{i,i'}\neq \emptyset\}+\tilde{N}<0$ does hold if $\Gamma$ is a non-regular diagram.
We prove it by considering two cases as follows.
\begin{itemize}
\item Case $\ell^{*}(2\alpha+\beta)<1$:
\begin{align}
&\#\{(i,i'): A_{i,i'}\neq \emptyset\}-\tilde{N}
\\\notag=&\#\{(i,i'): A_{i,i'}\neq \emptyset\}- \overset{\ell^{*}}{\underset{\ell=1}{\sum}}\left[1-\ell\left(2\alpha+\beta\right)\right]\#\{(i,i'): \#A_{i,i'}=\ell\}
\\\notag=&
\#\{(i,i'): \# A_{i,i'}>\ell^{*}\}+(2\alpha+\beta)\overset{\ell^{*}}{\underset{\ell=1}{\sum}}\ell\#\{(i,i'): \#A_{i,i'}=\ell\}
\\\notag\geq &
\#\{(i,i'): \# A_{i,i'}>\ell^{*}\}+\frac{1}{\mathbf{r}}\overset{\ell^{*}}{\underset{\ell=1}{\sum}}\ell\#\{(i,i'): \#A_{i,i'}=\ell\}
\\\notag\geq&\overset{p}{\underset{i,i'= 1}{\sum}}\frac{\# A_{i,i'}}{\ell_{i}}1_{\{(j,j^{'})\mid \# A_{j,j^{'}}>\ell^{*}\}}(i,i')
+\frac{1}{\mathbf{r}}\overset{\ell^{*}}{\underset{\ell=1}{\sum}}\overset{p}{\underset{i,i'= 1}{\sum}}\#A_{i,i'}1_{\{(j,j^{'})\mid \# A_{j,j^{'}}=\ell\}}(i,i')
\\\notag\geq&\overset{p}{\underset{i,i'= 1}{\sum}}\frac{\# A_{i,i'}}{\ell_{i}}1_{\{(j,j^{'})\mid \# A_{j,j^{'}}>\ell^{*}\}}(i,i')
+\overset{\ell^{*}}{\underset{\ell=1}{\sum}}\overset{p}{\underset{i,i'= 1}{\sum}}\frac{\#A_{i,i'}}{\ell_{i}}1_{\{(j,j^{'})\mid \# A_{j,j^{'}}=\ell\}}(i,i')
\\\label{eq:p/2_v1}=&\overset{p}{\underset{i= 1}{\sum}}\frac{1}{\ell_{i}}\left[\overset{p}{\underset{i'= 1}{\sum}}\# A_{i,i'}\right],
\end{align}
where we used the condition $\mathbf{r}(2\alpha+\beta)>1$ to get the first inequality and the fact $\# A_{i,i'}\leq \ell_{i}$ to get the second inequality. The last inequality follows from the fact $\ell_{i}\geq \mathbf{r}$ for all $i=1,\ldots,p$.
Furthermore,
\begin{align}\notag
\overset{p}{\underset{i= 1}{\sum}}\frac{1}{\ell_{i}}\left[\overset{p}{\underset{i'= 1}{\sum}}\# A_{i,i'}\right]
=&\overset{p}{\underset{i= 1}{\sum}}\frac{1}{\ell_{i}}\left[\underset{e\in E(\Gamma)}{\sum}1_{\{\overline{e}\mid d_{1}(\overline{e})=i\}}(e)\right]=\underset{e\in E(\Gamma)}{\sum}\overset{p}{\underset{i= 1}{\sum}}\frac{1}{\ell_{i}}1_{\{\overline{e}\mid d_{1}(\overline{e})=i\}}(e)
\\\label{eq:p/2_v2}=&\underset{e\in E(\Gamma)}{\sum}\frac{1}{\ell_{d_{1}(e)}}\geq
\frac{1}{2}\underset{e\in E(\Gamma)}{\sum}\left[\frac{1}{\ell_{d_{1}(e)}}+\frac{1}{\ell_{d_{2}(e)}}\right]=\frac{p}{2},
\end{align}
where the strictly inequality holds if there exists an edge $e$ such that $\ell_{d_{1}(e)}<\ell_{d_{2}(e)}$.
For each $i \in\{1,2,\ldots,p\}$, there are $\ell_{i}$ edges whose right-hand or left-hand side connect to the $i$th level.
Due to the weighting $1/\ell_{i}$, they contribute one to the last sum
in (\ref{eq:p/2_v2}). Hence, the last equality holds.

If $\Gamma$ is a non-regular diagram, then one of the following situations happens.
The reason will be explained below. We first show that $\#\{(i,i'): A_{i,i'}\neq \emptyset\}-\tilde{N} > \frac{p}{2}$ holds under each situation.
\begin{description}
\item[Situation 1: $\Gamma$ contains an edge $e$ such that $\ell_{d_{1}(e)}<\ell_{d_{2}(e)}$.]
Under this situation, (\ref{eq:p/2_v2}) implies that
\begin{align}\notag
\overset{p}{\underset{i= 1}{\sum}}\frac{1}{\ell_{i}}\left[\overset{p}{\underset{i'= 1}{\sum}}\# A_{i,i'}\right]>\frac{p}{2}.
\end{align}
By substituting the strict inequality above into (\ref{eq:p/2_v1}), we get $\#\{(i,i'): A_{i,i'}\neq \emptyset\}-\tilde{N} > \frac{p}{2}$.

\item[Situation 2:  $0<\# A_{i,i'}\leq \ell^{*}$ for some $i,i'\in\{1,2,\ldots,p\}$.]
Under this situation, the first inequality in (\ref{eq:p/2_v1}) is strict because $(2\alpha+\beta)>1/\mathbf{r}$.
By (\ref{eq:p/2_v2}) again,
we get the same result $\#\{(i,i'): A_{i,i'}\neq \emptyset\}-\tilde{N} > \frac{p}{2}$.

\item[Situation 3: $\ell^{*}<\# A_{i,i'}<\ell_{i}$ for some $i,i'\in\{1,2,\ldots,p\}$.]
If $\# A_{i,i'}>\ell^{*}$, then the first indicator function $1_{\{(j,j^{'})\mid \# A_{j,j^{'}}>\ell^{*}\}}(i,i')$ on the right hand side of the second inequality in (\ref{eq:p/2_v1}) is equal to one.
The condition $\# A_{i,i'}<\ell_{i}$ implies that second inequality in (\ref{eq:p/2_v1}) is strict.
By (\ref{eq:p/2_v2}) again,
we get the same result $\#\{(i,i'): A_{i,i'}\neq \emptyset\}-\tilde{N} > \frac{p}{2}$.
\end{description}

In the following, we prove that if $\Gamma$ is a non-regular diagram, then one of the situations above happens.
If none of Situation 2 and Situation 3 occurs, then we first have $\# A_{i,i'} = \ell_{i}$
for all $i,i'\in\{1,2,\ldots,p\}$ with $\# A_{i,i'}>0$.
Because the $i$th level only contains $\ell_{i}$ vertices, $\# A_{i,i'} = \ell_{i}$ implies that
\begin{equation}\label{counterproof}
d_{2}(e)=i'\ \textup{and}\ \ell_{d_{1}(e)}=\ell_{i}
\end{equation}
for all edge $e$ satisfying  $d_{1}(e)=i$.
If Situation 1 does not occur at the same time, i.e.,
$\ell_{d_{1}(e)}=\ell_{d_{2}(e)}$ for all edge $e$ in $\Gamma$,
then the first equality in (\ref{counterproof}) implies that
\begin{equation}\label{counterproof2}
\ell_{d_{2}(e)}=\ell_{i'}.
\end{equation}
Combining (\ref{counterproof2}) and the second equality in (\ref{counterproof}) yields
$\ell_{i} = \ell_{i'}$.
The equality $\ell_{i} = \ell_{i'}$ and $\# A_{i,i'} = \ell_{i}$ imply that all vertices
in the $i'$th level are occupied by the $\ell_{i}$ edges incident from the $i$th level.
Hence, if none of Situations 1-3 holds, then $\Gamma$ is a regular diagram.

\item $\ell^{*}(2\alpha+\beta)=1$:
In this case, by (\ref{eq:p/2_v1}) and (\ref{eq:p/2_v2}),
\begin{align}\notag
\#\{(i,i'): A_{i,i'}\neq \emptyset\}-\tilde{N}
\geq& \overset{p}{\underset{i= 1}{\sum}}\frac{1}{\ell_{i}}\left[\overset{p}{\underset{i'= 1}{\sum}}\# A_{i,i'}\right]-\varepsilon\#\{(i,i'): \#A_{i,i'}=\ell^{*}\}
\\\label{boundary_case_Gaussian_limit} \geq& \frac{p}{2}-\varepsilon\#\{(i,i'): \#A_{i,i'}=\ell^{*}\}.
\end{align}
If $\Gamma$ is a non-regular diagram, then one of the Situation 1 and Situation 2 above happens.
It implies that one of the inequalities in (\ref{boundary_case_Gaussian_limit}) becomes strict.
Therefore, there exists a $\varepsilon>0$ such that the right hand side of (\ref{boundary_case_Gaussian_limit})
is strictly greater than $\frac{p}{2}$.

\end{itemize}
For both cases, we have $\#\{(i,i'): A_{i,i'}\neq \emptyset\}-\tilde{N} > \frac{p}{2}$, so (\ref{F_upperbound}) implies that
\begin{align*}
\underset{j_{2}\rightarrow\infty}{\lim}\ \underset{j_{2}\rightarrow\infty}{\lim} F_{\Gamma}(K,L,j_{2})=C\underset{j_{2}\rightarrow\infty}{\lim}\  2^{j_{2}\left(\frac{p}{2}-\#\{(i,i'): A_{i,i'}\neq \emptyset\}\right)}2^{\tilde{N}j_{2}}=0.
\end{align*}
The proof of (c2) is complete. We thus finish the proof of Theorem 1.
\qed

\subsection{Proof of Corollary \ref{corollary_delta_generalA}}
\begin{proof}

By Theorem \ref{thm1:A_gaussian} and the continuous mapping theorem, when $j\rightarrow\infty$,
the rescaled stationary process $2^{(j_{1}+j)/2}U[j_{1},j_{1}+j]X(2^{(j_{1}+j)}t)$
converges to $|V_{1}(t)|$ in the finite dimensional distribution sense and the proof was based on the method of moments.
Hence, all conditions in Lemma \ref{lemma:expect_converge} in the appendix are satisfied.
It implies that when $j\rightarrow\infty$,
\begin{equation}\label{proof:coro1}
2^{(j_{1}+j)/2}\mathbb E\left[U[j_{1},j_{1}+j]X\right]\rightarrow \mathbb E\left[\Big|\kappa\int_{\mathbb{R}}
e^{i\lambda t}
\Psi(\lambda)W(d\lambda)\Big|\right] = \sqrt{\frac{2}{\pi}}\kappa\|\Psi\|_{L^{2}},
\end{equation}
On the other hand, because $X\star\psi_{j_{1}}$ is a stationary Gaussian process with mean zero and variance
$\sigma_{j_{1}}^{2}$,
\begin{equation}\label{proof:coro2}
\mathbb E\left[U[j_{1}]X\right] = \mathbb E|X\star\psi_{j_{1}}| = \sqrt{\frac{2}{\pi}}\sigma_{j_{1}}.
\end{equation}
Combining (\ref{proof:coro1}) and (\ref{proof:coro2}) yields
\begin{align*}
2^{(j_{1}+j)/2}\widetilde{S}X(j_{1},j_{1}+j)
= \frac{2^{(j_{1}+j)/2}\mathbb E\left[U[j_{1},j_{1}+j]X\right]}{\mathbb E\left[U[j_{1}]X\right]}
\rightarrow \kappa\sigma_{j_{1}}^{-1}\|\Psi\|_{L^{2}},
\end{align*}
when $j\rightarrow\infty$. Therefore, (\ref{thm1.2.1}) is proved.

To prove (\ref{thm1.2.2}), we first observe that
\begin{align}\notag
\underset{j_{1}\rightarrow\infty}{\lim} \Theta_{1}(j_{1},\alpha,\beta)^{2} =\,& \|\Psi\|_{L^{2}}^{2}\underset{j_{1}\rightarrow\infty}{\lim}2^{-j_{1}} \kappa^{2}
\sigma_{j_{1}}^{-2}
\\\notag=\,&\|\Psi\|_{L^{2}}^{2}\underset{j_{1}\rightarrow\infty}{\lim} 2^{-j_{1}}\left[\overset{\infty}{\underset{\ell=2}{\sum}}
\sigma_{j_{1}}^{2(1-\ell)}f_{X\star \psi_{j_{1}}}^{\star\ell}(0)
C^{2}_{\ell}\right]\sigma_{j_{1}}^{-2}
\\\notag
=\,&
\|\Psi\|_{L^{2}}^{2}\underset{j_{1}\rightarrow\infty}{\lim} 2^{-j_{1}}\overset{\infty}{\underset{\ell=2}{\sum}}
\sigma_{j_{1}}^{-2\ell}f_{X\star \psi_{j_{1}}}^{\star\ell}(0)
C^{2}_{\ell}
\\\label{squareTheta}=\,&
\|\Psi\|_{L^{2}}^{2}\underset{j_{1}\rightarrow\infty}{\lim} \overset{\infty}{\underset{\ell=2}{\sum}}
\left[2^{j_{1}\beta }\sigma_{j_{1}}^{2}\right]^{-\ell}\left[2^{j_{1}(\ell\beta-1)}f_{X\star \psi_{j_{1}}}^{\star\ell}(0)\right]
C^{2}_{\ell}.
\end{align}
By (\ref{def:sigmaj1}), the equality
\begin{equation*}
f_{X\star \psi_{j_{1}}}(\lambda) = \frac{C_{X}(\lambda)}{|\lambda|^{1-\beta}} |\Psi(2^{j_{1}}\lambda)|^{2}, 
\end{equation*}
and change of variables, we obtain
\begin{equation}\label{sigma_0}
\underset{j_{1}\rightarrow\infty}{\lim}2^{j_{1}\beta}\sigma_{j_{1}}^{2} = \underset{j_{1}\rightarrow\infty}{\lim}2^{j_{1}\beta }\int_{\mathbb{R}}f_{X\star \psi_{j_{1}}}(\lambda)d\lambda = C_{X}(0)\|Q\|_{L^{1}},
\end{equation}
where $Q(\lambda) = |\Psi(\lambda)|^{2}/|\lambda|^{1-\beta}$.
On the other hand, by change of variables, for all integer $\ell\geq2$, by a direct expansion we obtain
\begin{equation}\label{f_0}
\underset{j_{1}\rightarrow\infty}{\lim}2^{j_{1}(\ell\beta-1)}f_{X\star \psi_{j_{1}}}^{\star\ell}(0) = C_{X}(0)^{\ell}Q^{\star \ell}(0).
\end{equation}
For instance, when $\ell=3$, we have
\begin{align*}
&\underset{j_{1}\rightarrow\infty}{\lim}2^{j_{1}(3\beta-1)}f_{X\star \psi_{j_{1}}}^{\star3}(0)\nonumber\\
=& \underset{j_{1}\rightarrow\infty}{\lim}2^{j_{1}(3\beta-1)}  \int_{\mathbb{R}}\int_{\mathbb{R}}
\frac{C_{X}(\lambda_{1})}{|\lambda_{1}|^{1-\beta}} |\Psi(2^{j_{1}}\lambda_{1})|^{2}
\frac{C_{X}(\lambda_{1}-\lambda_{2})}{|\lambda_{1}-\lambda_{2}|^{1-\beta}} \nonumber\\
&\qquad\qquad\qquad\times |\Psi(2^{j_{1}}(\lambda_{1}-\lambda_{2}))|^{2}
\frac{C_{X}(\lambda_{2})}{|\lambda_{2}|^{1-\beta}} |\Psi(2^{j_{1}}\lambda_{2})|^{2}d\lambda_{1}d\lambda_{2}
\\=& \underset{j_{1}\rightarrow\infty}{\lim} \int_{\mathbb{R}}\int_{\mathbb{R}}
\frac{C_{X}(2^{-j_{1}}\lambda_{1})}{|\lambda_{1}|^{1-\beta}} |\Psi(\lambda_{1})|^{2}
\frac{C_{X}(2^{-j_{1}}\lambda_{1}-2^{-j_{1}}\lambda_{2})}{|\lambda_{1}-\lambda_{2}|^{1-\beta}} \nonumber\\
&\qquad\qquad\qquad\times|\Psi(\lambda_{1}-\lambda_{2})|^{2}
\frac{C_{X}(2^{-j_{1}}\lambda_{2})}{|\lambda_{2}|^{1-\beta}} |\Psi(\lambda_{2})|^{2}d\lambda_{1}d\lambda_{2}
\\=& C_{X}(0)^{3}Q^{\star 3}(0)\,.
\end{align*}

Substituting (\ref{sigma_0}) and (\ref{f_0}) into (\ref{squareTheta}) yields
\begin{align}\notag
\underset{j_{1}\rightarrow\infty}{\lim} \Theta_{1}(j_{1},\alpha,\beta)^{2} =&
\|\Psi\|_{L^{2}}^{2}\overset{\infty}{\underset{\ell=2}{\sum}}
\left[C_{X}(0)\|Q\|_{L^{1}}\right]^{-\ell}
\left[C_{X}(0)^{\ell}Q^{\star \ell}(0)\right]
C^{2}_{\ell}
\\\label{Q_tilde}=&\|\Psi\|_{L^{2}}^{2}\overset{\infty}{\underset{\ell=2}{\sum}}
\|Q\|_{L^{1}}^{-\ell}
Q^{\star \ell}(0)
C^{2}_{\ell}.
\end{align}
By defining $\tilde{Q}(\lambda) = Q(\lambda)/\|Q\|_{L^{1}}$ for $\lambda\in \mathbb{R}$, we have
\begin{align*}
\|Q\|_{L^{1}}^{-\ell}Q^{\star \ell}(0)=\tilde{Q}^{\star \ell}(0) =  \int_{\mathbb{R}} \tilde{Q}^{\star2}(-\lambda)\tilde{Q}^{\star (\ell-2)}(\lambda)d\lambda
\leq \underset{\lambda\in \mathbb{R}}{\textup{sup}}\ \tilde{Q}^{\star 2}(\lambda)
\end{align*}
for all $\ell> 2$. The behavior of $\tilde{Q}(\lambda)$ near the origin looks like $|\lambda|^{2\alpha+\beta-1}$.
Hence, the condition $2\alpha+\beta>1/\mathbf{r}=1/2$ and Lemma \ref{lemma:1} in the appendix imply that $\tilde{Q}^{\star 2}$ is bounded.
The observation above and the fact $\overset{\infty}{\underset{\ell=2}{\sum}}
C^{2}_{\ell} = 1$ imply that
the last series in (\ref{Q_tilde}) converges.
\end{proof}


\subsection{Proof of Theorem \ref{thm:nonGaussian}}

By the Hermite polynomials expansion (\ref{eq:decom_2nd_scat}),
we have that
\begin{align}\notag
U^{A_{1}}[j_{1}]X\star \psi_{j_{2}}(2^{j_{2}}t)
=\sum_{\ell=\mathbf{r}}^{\infty}
\frac{C_{\sigma_{j_{1}},\ell}}{\sqrt{\ell!}}Z_{\ell}(2^{j_{2}}t),
\end{align}
where $Z_{\ell}$ is defined in (\ref{eq:decom_2nd_scatQ}).
By Slutsky's argument and the Cramer-Wold device \cite[p. 6.]{leonenko1999limit},
Theorem \ref{thm:nonGaussian} will be proved if we can show
that
\begin{equation}
\left\{\begin{array}{lr}
\textup{(a)}\ 2^{j_{2}(2\alpha+\beta)\mathbf{r}/2}
\frac{C_{\sigma_{j_{1}},\mathbf{r}}}{\sqrt{\mathbf{r}!}}Z_{\mathbf{r}}(2^{j_{2}}t) \overset{d}{\Rightarrow} V_{2}(t),
\\
\textup{(b)}\ 2^{j_{2}(2\alpha+\beta)\mathbf{r}/2}\overset{\infty}{\underset{\ell=\mathbf{r}+1}{\sum}}
\frac{C_{\sigma_{j_{1}},\ell}}{\sqrt{\ell!}}Z_{\ell}(2^{j_{2}}t) \rightarrow 0\ \textup{in probability}
\end{array}
\right.
\end{equation}
when $j_{2}\rightarrow\infty$, where the process $V_{2}$ is defined in (\ref{thm:limitprocess}).
\\

\noindent{\it Proof of (a):}
By (\ref{eq:decom_2nd_scatQ}) and the self-similar property; that is,
\begin{equation*}
W(2^{-j_{2}}d\lambda) \overset{d}{=} 2^{-\frac{j_{2}}{2}}W(d\lambda)\,,
\end{equation*}
we have
\begin{align}\notag
Z_{\mathbf{r}}(2^{j_{2}}t)
=&
\sigma_{j_{1}}^{-\mathbf{r}}\int^{'}_{\mathbb{R}^{\mathbf{r}}}e^{i2^{j_{2}}(\lambda_{1}+\cdots+\lambda_{\mathbf{r}})t}
\Psi(2^{j_{2}}(\lambda_{1}+\cdots+\lambda_{\mathbf{r}}))\left[\prod_{k=1}^{\mathbf{r}}{\Psi(2^{j_{1}}\lambda_{k})\sqrt{f_{X}(\lambda_{k})}}W(d\lambda_{k})\right]
\\\label{asym_Z2}\overset{d}{=}& 2^{-j_{2}\mathbf{r}/2}\sigma_{j_{1}}^{-\mathbf{r}}
\int^{'}_{\mathbb{R}^{\mathbf{r}}}e^{i(\lambda_{1}+\cdots+\lambda_{\mathbf{r}})t}
\Psi(\lambda_{1}+\cdots+\lambda_{\mathbf{r}})\nonumber\\
&\qquad\qquad\times \left[\prod_{k=1}^{\mathbf{r}}{\Psi(2^{j_{1}-j_{2}}\lambda_{k})\sqrt{f_{X}(2^{-j_{2}}\lambda_{k})}}W(d\lambda_{k})\right].
\end{align}
Applying Assumptions 1 and 2 to (\ref{asym_Z2}), we get
\begin{align}\label{asym_Z2_assumption}
2^{j_{2}(2\alpha+\beta)\mathbf{r}/2}Z_{\mathbf{r}}(2^{j_{2}}t)
\overset{d}{=}Z_{\mathbf{r},j_{2}}(t),
\end{align}
where
\begin{equation*}
Z_{\mathbf{r},j_{2}}(t) =  2^{j_{1}\alpha\mathbf{r}}\sigma_{j_{1}}^{-\mathbf{r}}
\int^{'}_{\mathbb{R}^{\mathbf{r}}}e^{i(\lambda_{1}+\cdots+\lambda_{\mathbf{r}})t}C_{j_{2}}(\lambda_{1:\mathbf{r}})
\frac{\Psi(\lambda_{1}+\cdots+\lambda_{\mathbf{r}})}{|\lambda_{1}\cdots\lambda_{\mathbf{r}}|^{\frac{1-2\alpha-\beta}{2}}}W(d\lambda_{1})\cdots W(d\lambda_{\mathbf{r}})
\end{equation*}
and
\begin{equation*}
C_{j_{2}}(\lambda_{1:\mathbf{r}}) = \overset{\mathbf{r}}{\underset{k=1}{\prod}} C_{\Psi}(2^{j_{1}-j_{2}}\lambda_{k})
\sqrt{C_{X}(2^{-j_{2}}\lambda_{k})}.
\end{equation*}
For any $M\in \mathbb{N}$ and any set of real numbers
$\{a_{1},a_{2},\ldots,a_{M},t_{1},t_{2},\ldots,t_{M}\}$, (\ref{asym_Z2_assumption}) implies that
\begin{align}\label{asym_Z2_assumption_v2}
\overset{M}{\underset{k=1}{\sum}}a_{k}2^{j_{2}(2\alpha+\beta)\mathbf{r}/2}
\frac{C_{\sigma_{j_{1}},\mathbf{r}}}{\sqrt{\mathbf{r}!}}Z_{\mathbf{r}}(2^{j_{2}}t_{k}) \overset{d}{=}
\overset{M}{\underset{k=1}{\sum}}a_{k}
\frac{C_{\sigma_{j_{1}},\mathbf{r}}}{\sqrt{\mathbf{r}!}}Z_{\mathbf{r},j_{2}}(t_{k}).
\end{align}
By the Minkowski inequality,
\begin{align}\notag
\left\{\mathbb E\Big|\overset{M}{\underset{k=1}{\sum}}a_{k}
\frac{C_{\sigma_{j_{1}},\mathbf{r}}}{\sqrt{\mathbf{r}!}}Z_{\mathbf{r},j_{2}}(t_{k})-\overset{M}{\underset{k=1}{\sum}}a_{k}
\frac{C_{\sigma_{j_{1}},\mathbf{r}}}{\sqrt{\mathbf{r}!}}V_{2}(t_{k})\Big|^{2}\right\}^{\frac{1}{2}}
\leq \frac{|C_{\sigma_{j_{1}},\mathbf{r}}|}{\sqrt{\mathbf{r}!}}\overset{M}{\underset{k=1}{\sum}}|a_{k}|
\left\{\mathbb E\Big|Z_{\mathbf{r},j_{2}}(t_{k})-V_{2}(t_{k})\Big|^{2}\right\}^{\frac{1}{2}}.
\end{align}
For each $k\in\{1,\ldots,M\}$, by the orthogonal property of the complex Gaussian random measure,
\begin{align}\notag
&\left[2^{j_{1}\alpha\mathbf{r}}\sigma_{j_{1}}^{-\mathbf{r}}\right]^{-2}\underset{j_{2}\rightarrow\infty}{\lim}\mathbb E\Big|Z_{\mathbf{r},j_{2}}(t_{k})-V_{2}(t_{k})\Big|^{2}
\\\notag=&\underset{j_{2}\rightarrow\infty}{\lim}\mathbb E\left\{\Big|\int^{'}_{\mathbb{R}^{\mathbf{r}}}e^{i(\lambda_{1}+\cdots+\lambda_{\mathbf{r}})t}\left[C_{j_{2}}(\lambda_{1:\mathbf{r}})-C_{\Psi}(0)^{\mathbf{r}}C_{X}(0)^{\frac{\mathbf{r}}{2}}\right]
\frac{\Psi(\lambda_{1}+\cdots+\lambda_{\mathbf{r}})}{|\lambda_{1}\cdots\lambda_{\mathbf{r}}|^{\frac{1-2\alpha-\beta}{2}}}W(d\lambda_{1})\cdots W(d\lambda_{\mathbf{r}})
\Big|^{2}\right\}
\\\label{asym_Z2_diff}=&\mathbf{r}!\ \underset{j_{2}\rightarrow\infty}{\lim}\int_{\mathbb{R}^{\mathbf{r}}}\big|C_{j_{2}}(\lambda_{1:\mathbf{r}})-C_{\Psi}(0)^{\mathbf{r}}C_{X}(0)^{\frac{\mathbf{r}}{2}}\big|^{2}
\frac{|\Psi(\lambda_{1}+\cdots+\lambda_{\mathbf{r}})|^{2}}{|\lambda_{1}\cdots \lambda_{\mathbf{r}}|^{1-2\alpha-\beta}}d\lambda_{1}\ \cdots\  d\lambda_{\mathbf{r}}
=0\,,
\end{align}
where the last equality follows from the Lebesgue dominated convergence theorem, because
if $\mathbf{r}(2\alpha+\beta)<1$, Riesz's composition formula \cite[p. 71]{du1970introduction} implies that
\begin{align*}
\int_{\mathbb{R}^{\mathbf{r}}}\frac{|\Psi(\lambda_{1}+\cdots+\lambda_{\mathbf{r}})|^{2}}{|\lambda_{1}\cdots\lambda_{\mathbf{r}}|^{1-2\alpha-\beta}}d\lambda_{1}\ \cdots\ d\lambda_{\mathbf{r}} = \pi^{\frac{\mathbf{r}-1}{2}} \frac{\left[\Gamma(\frac{2\alpha+\beta}{2})\right]^\mathbf{r}\Gamma(\frac{1-\mathbf{r}(2\alpha+\beta)}{2})}{\left[\Gamma(\frac{1-2\alpha-\beta}{2})\right]^\mathbf{r}\Gamma(\mathbf{r}(2\alpha+\beta))}\int_{\mathbb{R}} \frac{|\Psi(\lambda)|^{2}}{|\lambda|^{1-\mathbf{r}(2\alpha+\beta)}}d\lambda<\infty,
\end{align*}
and
$\big|C_{j_{2}}(\lambda_{1:\mathbf{r}})-C_{\Psi}(0)^{\mathbf{r}}C_{X}(0)^{\frac{\mathbf{r}}{2}}\big|$ is a bounded  and continuous function
with $\underset{j_{2}\rightarrow\infty}{\lim}\big|C_{j_{2}}(\lambda_{1:\mathbf{r}})-C_{\Psi}(0)^{\mathbf{r}}C_{X}(0)^{\frac{\mathbf{r}}{2}}\big|=0$.
It follows from (\ref{asym_Z2_diff}) that
\begin{align}\label{nonGaussian_slutsky}
\underset{j_{2}\rightarrow\infty}{\lim}\overset{M}{\underset{k=1}{\sum}}a_{k}
\frac{C_{\sigma_{j_{1}},\mathbf{r}}}{\sqrt{\mathbf{r}!}}Z_{\mathbf{r},j_{2}}(t_{k}) =  \overset{M}{\underset{k=1}{\sum}}a_{k}
\frac{C_{\sigma_{j_{1}},\mathbf{r}}}{\sqrt{\mathbf{r}!}}V_{2}(t_{k})
\end{align}
in the $L^{2}(dP)$ sense.
By applying Slutsky's argument to (\ref{asym_Z2_assumption_v2}) and (\ref{nonGaussian_slutsky}),
we observe
\begin{align}\label{nonGaussian_slutsky_v2}
\underset{j_{2}\rightarrow\infty}{\lim}\overset{M}{\underset{k=1}{\sum}}a_{k}2^{j_{2}(2\alpha+\beta)\mathbf{r}/2}
\frac{C_{\sigma_{j_{1}},\mathbf{r}}}{\sqrt{\mathbf{r}!}}Z_{\mathbf{r}}(2^{j_{2}}t_{k}) = \overset{M}{\underset{k=1}{\sum}}a_{k}
\frac{C_{\sigma_{j_{1}},\mathbf{r}}}{\sqrt{\mathbf{r}!}}V_{2}(t_{k})
\end{align}
in the distribution sense.
By the Cramer-Wold
device and (\ref{nonGaussian_slutsky_v2}),
\begin{align*}
2^{j_{2}(2\alpha+\beta)\mathbf{r}/2}\sigma_{j_{1}}
\frac{C_{\sigma_{j_{1}},\mathbf{r}}}{\sqrt{\mathbf{r}!}}Z_{\mathbf{r}}(2^{j_{2}}t)\overset{d}{\Rightarrow} \frac{C_{\sigma_{j_{1}},\mathbf{r}}}{\sqrt{\mathbf{r}!}}V_{2}(t)
\end{align*}
when $j_{2}\rightarrow\infty$, where the limiting process $V_2$ is defined in (\ref{thm:limitprocess}).
The claim (a) is proved.
\\

\noindent{\it Proof of (b):}
By the definition of the process $Z_{\ell}$ in
(\ref{eq:decom_2nd_scatQ}) and the orthogonal property
(\ref{expectionhermite}), we have
\begin{align}\notag
&\mathbb{E}\left\{\left[2^{j_{2}(2\alpha+\beta)\mathbf{r}/2}\overset{\infty}{\underset{\ell>\mathbf{r}}{\sum}}
\frac{C_{\sigma_{j_{1}},\ell}}{\sqrt{\ell!}}Z_{\ell}(2^{j_{2}}t)\right]^{2}\right\}
\\\notag
=&
2^{j_{2}(2\alpha+\beta)\mathbf{r}}\overset{\infty}{\underset{\ell>\mathbf{r}}{\sum}}\overset{\infty}{\underset{m>\mathbf{r}}{\sum}}
\frac{C_{\sigma_{j_{1}},\ell}}{\sqrt{\ell!}}\frac{C_{\sigma_{j_{1}},m}}{\sqrt{m!}}\int_{\mathbb{R}}\int_{\mathbb{R}}
\mathbb E\left[H_{\ell}(Y_{j_{1}}(s))H_{m}(Y_{j_{1}}(s'))\right]
\psi_{j_{2}}(2^{j_{2}}t-s)\psi_{j_{2}}(2^{j_{2}}t-s')ds\ ds'
\\\label{residual_term}=&
2^{j_{2}(2\alpha+\beta)\mathbf{r}}\overset{\infty}{\underset{\ell>\mathbf{r}}{\sum}}
C^{2}_{\sigma_{j_{1}},\ell}\int_{\mathbb{R}}\int_{\mathbb{R}}
R^{\ell}_{Y_{j_{1}}}(s-s')
\psi_{j_{2}}(2^{j_{2}}t-s)\psi_{j_{2}}(2^{j_{2}}t-s')ds\ ds'.
\end{align}
By the following Fourier transform relationship
\begin{equation*}
R^{\ell}_{Y_{j_{1}}}(s-s')= \int_{\mathbb{R}} e^{i\lambda (s-s')} f^{\star \ell}_{Y_{j_{1}}}(\lambda)d\lambda,\quad
\int_{\mathbb{R}}e^{i\lambda s}
\psi_{j_{2}}(2^{j_{2}}t-s)
ds
=e^{i\lambda 2^{j_{2}}t}
\Psi(2^{j_{2}}\lambda),
\end{equation*}
and
\begin{align*}
&\int_{\mathbb{R}}e^{-i\lambda s'}
\psi_{j_{2}}(2^{j_{2}}t-s')
ds'
=e^{-i\lambda 2^{j_{2}}t}
\overline{\Psi(2^{j_{2}}\lambda)},
\end{align*}
we can rewrite (\ref{residual_term}) as
\begin{align}
&\mathbb{E}\left\{\left[2^{j_{2}(2\alpha+\beta)\mathbf{r}/2}\overset{\infty}{\underset{\ell>\mathbf{r}}{\sum}}
\frac{C_{\sigma_{j_{1}},\ell}}{\sqrt{\ell!}}Z_{\ell}(2^{j_{2}}t)\right]^{2}\right\}\label{part0}\\
=&\notag
2^{j_{2}(2\alpha+\beta)\mathbf{r}}\overset{\infty}{\underset{\ell>\mathbf{r}}{\sum}}
C^{2}_{\sigma_{j_{1}},\ell}\int_{\mathbb{R}}
f^{\star\ell}_{Y_{j_{1}}}(\lambda)|\Psi(2^{j_{2}}\lambda)|^{2}
d\lambda
 2^{j_{2}((2\alpha+\beta)\mathbf{r}-1)}\overset{\infty}{\underset{\ell>\mathbf{r}}{\sum}}
C^{2}_{\sigma_{j_{1}},\ell}\int_{\mathbb{R}}
f^{\star\ell}_{Y_{j_{1}}}(2^{-j_{2}}\lambda)|\Psi(\lambda)|^{2}
d\lambda.
\end{align}
Under Assumptions 1 and 2, we know that
\begin{align*}
f_{Y_{j_{1}}}(\lambda) = \frac{1}{\sigma_{j_{1}}^{2}}|\Psi(2^{j_{1}}\lambda)|^{2}f_{X}(\lambda)
=\frac{2^{2\alpha j_{1}}}{\sigma_{j_{1}}^{2}}\frac{|C_{\Psi}(2^{j_{1}}\lambda)|^{2}C_{X}(\lambda)}{|\lambda|^{1-2\alpha-\beta}},\ \lambda\in \mathbb{R},
\end{align*}
where $2\alpha+\beta<1/\mathbf{r}$ under the assumption of Theorem 2.
Denote $\ell^{*} = \textup{max}\{\ell\in \mathbb{N}| \ell\geq \mathbf{r},\ \ell(2\alpha+\beta)<1\}$.
For $\mathbf{r}\leq\ell\leq \ell^{*}$, by Lemma \ref{lemma:1} in the appendix,
\begin{align*}
f_{Y_{j_{1}}}^{\star \ell}(\lambda) \leq\frac{B_{\ell}(\lambda)}{|\lambda|^{1-\ell(2\alpha+\beta)}}
\end{align*}
for some bounded function $B_{\ell}$.
Hence,
\begin{align}\notag
&2^{j_{2}((2\alpha+\beta)\mathbf{r}-1)}\overset{\ell^{*}}{\underset{\ell>\mathbf{r}}{\sum}}
C^{2}_{\sigma_{j_{1}},\ell}\int_{\mathbb{R}}
f^{\star\ell}_{Y_{j_{1}}}(2^{-j_{2}}\lambda)|\Psi(\lambda)|^{2}
d\lambda
\\\notag\leq& 2^{j_{2}((2\alpha+\beta)\mathbf{r}-1)}\overset{\ell^{*}}{\underset{\ell>\mathbf{r}}{\sum}}
C^{2}_{\sigma_{j_{1}},\ell}\int_{\mathbb{R}}
\frac{B_{\ell}(2^{-j_{2}}\lambda)}{|2^{-j_{2}}\lambda|^{1-\ell(2\alpha+\beta)}}|\Psi(\lambda)|^{2}
d\lambda
\\\label{part1}=&\overset{\ell^{*}}{\underset{\ell>\mathbf{r}}{\sum}}2^{j_{2}(2\alpha+\beta)(\mathbf{r}-\ell)}
C^{2}_{\sigma_{j_{1}},\ell}\int_{\mathbb{R}}
\frac{B_{\ell}(2^{-j_{2}}\lambda)}{|\lambda|^{1-\ell(2\alpha+\beta)}}|\Psi(\lambda)|^{2}
d\lambda\rightarrow0
\end{align}
when $j_{2}\rightarrow\infty$.

If $(\ell^{*}+1)(2\alpha+\beta)=1$, by Lemma \ref{lemma:1} in the appendix,
\begin{align*}
f_{Y_{j_{1}}}^{\star (\ell^{*}+1)}(\lambda) \leq B_{\ell^{*}+1}(\lambda)\textup{ln}(2+\frac{1}{|\lambda|})
\end{align*}
for some bounded function $B_{\ell^{*}+1}$. For any $\varepsilon>0$, there exists a constant $C_{\varepsilon}$ such that
\begin{equation*}
\textup{ln}(2+\frac{1}{|\lambda|})\leq C_{\varepsilon}|\lambda|^{-\varepsilon}\ \textup{for}\ \lambda\in[-1,1]\setminus \{0\}.
\end{equation*}
By choosing $\varepsilon=(1-(2\alpha+\beta)\mathbf{r})/2$, we have
\begin{align}\notag
&2^{j_{2}((2\alpha+\beta)\mathbf{r}-1)}
\int_{\mathbb{R}}
f^{\star(\ell^{*}+1)}_{Y_{j_{1}}}(2^{-j_{2}}\lambda)|\Psi(\lambda)|^{2}
d\lambda
\\\notag\leq& 2^{j_{2}((2\alpha+\beta)\mathbf{r}-1)}
\left[\underset{\lambda\in \mathbb{R}}{\textup{max}}B_{\ell^{*}+1}(\lambda)\right]\int_{\mathbb{R}}
\textup{ln}(2+\frac{1}{|2^{-j_{2}}\lambda|})|\Psi(\lambda)|^{2}
d\lambda
\\\notag\leq& 2^{j_{2}((2\alpha+\beta)\mathbf{r}-1)}
\left[\underset{\lambda\in \mathbb{R}}{\textup{max}}B_{\ell^{*}+1}(\lambda)\right]\left\{C_{\varepsilon}\int_{-2^{j_{2}}}^{2^{j_{2}}}
|2^{-j_{2}}\lambda|^{-\varepsilon}|\Psi(\lambda)|^{2}d\lambda+\textup{ln}9\int_{2^{j_{2}}}^{\infty}|\Psi(\lambda)|^{2}d\lambda\right\}
\\\label{part2}\leq&
\left[\underset{\lambda\in\mathbb{R}}{\textup{max}}B_{\ell^{*}+1}(\lambda)\right]
\left\{C_{\varepsilon}2^{j_{2}((2\alpha+\beta)\mathbf{r}-1)/2}
\int_{\mathbb{R}}|\lambda|^{-\varepsilon}|\Psi(\lambda)|^{2}d\lambda+ 2^{j_{2}((2\alpha+\beta)\mathbf{r}-1)}\textup{ln}9  \int_{2^{j_{2}}}^{\infty}|\Psi(\lambda)|^{2}d\lambda
\right\},
\end{align}
which converges to zero when $j_{2}\rightarrow\infty$ because $(2\alpha+\beta)\mathbf{r}-1<0$.


If $(\ell^{*}+1)(2\alpha+\beta)>1$, by Lemma \ref{lemma:1} in the appendix, for each $\ell\geq \ell^{*}+1$,
$f_{Y_{j_{1}}}^{\star \ell}$
is bounded. Moreover,
\[
\underset{\lambda\in \mathbb{R}}{\textup{max}}f^{\star(\ell+1)}(\lambda)\leq \underset{\lambda\in \mathbb{R}}{\textup{max}}f^{\star\ell}(\lambda)
\]
for all $\ell\geq \ell^{*}+1$. Hence,
\begin{align}\notag
&2^{j_{2}((2\alpha+\beta)\mathbf{r}-1)}\overset{\infty}{\underset{\ell=\ell^{*}+1}{\sum}}
C^{2}_{\sigma_{j_{1}},\ell}\int_{\mathbb{R}}
f^{\star\ell}_{Y_{j_{1}}}(2^{-j_{2}}\lambda)|\Psi(\lambda)|^{2}
d\lambda
\\\label{part3}\leq& \left[\underset{\eta\in \mathbb{R}}{\textup{max}} f^{\star(\ell^{*}+1)}(\eta)\right] 2^{j_{2}((2\alpha+\beta)\mathbf{r}-1)}\left[\overset{\infty}{\underset{\ell=\ell^{*}+1}{\sum}}
C^{2}_{\sigma_{j_{1}},\ell}\right]
\|\Psi\|_{L^2}^{2}
\rightarrow0
\end{align}
when $j_{2}\rightarrow\infty$.
Substituting (\ref{part1})-(\ref{part3}) into (\ref{part0}) yields
\begin{align}\notag
\underset{j_{2}\rightarrow\infty}{\lim}\mathbb{E}\left\{\left[2^{j_{2}(2\alpha+\beta)\mathbf{r}/2}\overset{\infty}{\underset{\ell>\mathbf{r}}{\sum}}
\frac{C_{\sigma_{j_{1}},\ell}}{\sqrt{\ell!}}Z_{\ell}(2^{j_{2}}t)\right]^{2}\right\}=0.
\end{align}
The claim (b) follows
by the Markov inequality. We have thus proved Theorem \ref{thm:nonGaussian}.
\qed
%

\subsection{Proof of Corollary \ref{corollary_nonGaussian_v1}}
\begin{proof}
First of all, we note that $\mathbf{r} = \textup{rank}(A_{1}(\sigma_{j_{1}}\cdot))=2$ when $A_{1}(\cdot) = |\cdot|$. From the proof of Theorem \ref{thm:nonGaussian}, we have that
\begin{equation}
\mathbb E\Big[\Big|2^{j_{2}(2\alpha+\beta)} U[j_{1},j_{2}]X\Big|^{2} \Big]\rightarrow \mathbb E\left[|V_{2}|^{2}\right]
\end{equation}
when $j_{2}\rightarrow\infty$. Hence, the conditions of Lemma \ref{lemma:expect_converge} in the appendix hold.
It implies that
\begin{equation}\label{proof:coro2.1}
\underset{j\rightarrow\infty}{\lim}2^{(j_{1}+j)(2\alpha+\beta)}E\big[U[j_{1},j_{1}+j]X\big]=\mathbb  E\left[\big|V_{2}\big|\right]
\end{equation}
Because $X\star\psi_{j_{1}}$ is a stationary Gaussian process with mean zero and variance
$\sigma_{j_{1}}^{2}$,
\begin{equation}\label{proof:coro2.2}
\mathbb E\left[U[j_{1}]X\right] = \mathbb E|X\star\psi_{j_{1}}| = \sqrt{\frac{2}{\pi}}\sigma_{j_{1}}.
\end{equation}
Combining (\ref{proof:coro2.1}) and (\ref{proof:coro2.2}) yields
\begin{align*}
2^{j(2\alpha+\beta)}\widetilde{S}X(j_{1},j_{1}+j)
\rightarrow& 2^{-j_{1}(2\alpha+\beta)}\sqrt{\frac{\pi}{2}}\sigma_{j_{1}}^{-1}\nu \mathbb E\Big|\int^{'}_{\mathbb{R}^{2}}
\frac{\Psi(\lambda_{1}+\lambda_{2})}{|\lambda_{1}\lambda_{2}|^{\frac{1-2\alpha-\beta}{2}}}W(d\lambda_{1})W(d\lambda_{2})\Big|
\\=&
2^{-j_{1}\beta}\sigma_{j_{1}}^{-2}\frac{\sqrt{\pi}}{2} C_{\Psi}(0)^{2}C_{X}(0)C_{2}\mathbb E\Big|\int^{'}_{\mathbb{R}^{2}}
\frac{\Psi(\lambda_{1}+\lambda_{2})}{|\lambda_{1}\lambda_{2}|^{\frac{1-2\alpha-\beta}{2}}}W(d\lambda_{1})W(d\lambda_{2})\Big|
\end{align*}
when $j\rightarrow\infty$. Hence, (\ref{thm2.2.1}) is proved.
(\ref{thm2.2.2}) comes from the observation (\ref{sigma_0}), i.e.,
\begin{equation*}
\underset{j_{1}\rightarrow\infty}{\lim}2^{j_{1}\beta}\sigma_{j_{1}}^{2} = C_{X}(0)\|Q\|_{L^{1}}.
\end{equation*}
\end{proof}


\vskip 20 pt

%
%


\bigskip

\noindent {\bf Reference}
\bibliographystyle{elsarticle-num}
\bibliography{reference}

\bigskip


\appendix

\section{Proof of Lemma \ref{thm:nonexpansive}}\label{sec:proof_nonexpansive_UJA}

Because $U^{A_{i}}[j]X=  A_{i}(W[j]X)$, $L_{A_{i}}\leq1$, and $A_{i}(0)=0$ for $i\in\{1,2,...,n\}$,
we have
\begin{align}\notag
\mathbb{E}\left[\|U^{A_{i}}_{J}X\|^{2}\right]=&\mathbb{E}\left[|P_{J}X|^{2}\right]
+\underset{j<J}{\sum}\mathbb{E}\left[\big|A_{i}(W[j]X)\big|^{2}\right]
\\\notag
\leq&
\mathbb{E}\left[|P_{J}X|^{2}\right]
+L_{A_{i}}\underset{j<J}{\sum}\mathbb{E}\left[\big|W[j]X\big|^{2}\right]
\\\label{proof:nonexpansive1}
\leq&
\mathbb{E}\left[|P_{J}X|^{2}\right]
+\underset{j<J}{\sum}\mathbb{E}\left[\big|W[j]X\big|^{2}\right].
\end{align}
Because
$P_{J}X$ has the spectral density $f_{X}(\cdot)|\Phi(2^{J}\cdot)|^{2}$ and
the condition $\int_{\mathbb{R}} \phi(x)dx = 1$ implies that $\mathbb{E}[P_{J}X] = \mathbb{E}[X]$,
\begin{align}\label{EPX}
\mathbb{E}\left[|P_{J}X|^{2}\right]= \int_{\mathbb{R}}f_{X}(\lambda)|\Phi(2^{J}\lambda)|^{2}d\lambda + \left(\mathbb{E}[X]\right)^{2}.
\end{align}
Because
$W[j]X$ has the spectral density $f_{X}(\cdot)|\Psi(2^{j}\cdot)|^{2}$ and
the condition $\int_{\mathbb{R}} \psi(x)dx = 0$ implies that $\mathbb{E}[W[j]X] = 0$,
\begin{align}\label{EWX}
\mathbb{E}\left[|W[j]X|^{2}\right]= \int_{\mathbb{R}}f_{X}(\lambda)|\Psi(2^{j}\lambda)|^{2}d\lambda.
\end{align}
By substituting (\ref{EPX}) and (\ref{EWX}) into the right hand side of (\ref{proof:nonexpansive1}),
we get
\begin{align}\notag
\mathbb{E}\left[\|U^{A_{i}}_{J}X\|^{2}\right]
\leq&
\int_{\mathbb{R}}f_{X}(\lambda)|\Phi(2^{J}\lambda)|^{2}d\lambda + \left(\mathbb{E}[X]\right)^{2}
+\underset{j<J}{\sum}\int_{\mathbb{R}}f_{X}(\lambda)|\Psi(2^{j}\lambda)|^{2}d\lambda.
\\\label{UA<EX2}\leq&
\int_{\mathbb{R}}f_{X}(\lambda)d\lambda + \left(\mathbb{E}[X]\right)^{2}
= E\left[|X|^{2}\right],
\end{align}
where the second inequality follows from the Littlewood-Paley condition (\ref{Littlewood_condition}).

By (\ref{UA<EX2}) and the definition $U^{A_{n}}_{J} = \{P_{J},U^{A_{n}}[\Lambda_{J}]\}$,
\begin{align}\notag
\mathbb{E}\left[\|U^{A_{1:n-1}}[\Lambda_{J}^{n-1}]X\|^{2}\right] \geq& \mathbb{E}\left[\|U^{A_{n}}_{J}U^{A_{1:n-1}}[\Lambda_{J}^{n-1}]X\|^{2}\right]
\\\label{nonincreasing_U_app}=&\mathbb{E}\left[\|S^{A_{1:n-1}}_{J}[\Lambda_{J}^{n-1}]X\|^{2}\right]
+\mathbb{E}\left[\|U^{A_{1:n}}[\Lambda_{J}^{n}]X\|^{2}\right]
\end{align}
for $n\in \mathbb{N}.$
By using the inequality above iteratively, we get
\begin{align}\notag
\mathbb{E}\left[|X|^{2}\right]\geq\mathbb{E}\left[\|U_{J}^{A_{1}}X\|^{2}\right]\geq
\mathbb{E}\left[\|U^{A_{1}}[\Lambda_{J}]X\|^{2}\right]\geq\cdots
\geq\mathbb{E}\left[\|U^{A_{1:n}}[\Lambda_{J}^{n}]X\|^{2}\right],
\end{align}
which proves (\ref{nonincreasing_U}).
For any $n\in \mathbb{N}\cup\{0\}$, by (\ref{nonincreasing_U_app}),
\begin{align}\notag
\mathbb{E}\left[\overset{n}{\underset{m=0}{\sum}}\|S^{A_{1:m}}_{J}[\Lambda_{J}^{m}]X\|^{2}\right]
\leq
\mathbb{E}\left[\|U^{A_{1:0}}[\Lambda_{J}^{0}]X\|^{2}\right]-\mathbb{E}\left[\|U^{A_{1:n+1}}[\Lambda_{J}^{n+1}]X\|^{2}\right]
\leq \mathbb{E}\left[|X|^{2}\right],
\end{align}
which proves (\ref{nonexpansive_S}).

\section{Proof of Lemma \ref{thm:deformation}}\label{main_section_stability}


First of all, we have
\begin{align}\notag
&\mathbb{E}\left[\|S^{A_{1:n}}_{J}[\Lambda_{J}^{0:n}]L_{\tau}X-S^{A_{1:n}}_{J}[\Lambda_{J}^{0:n}]X\|^{2}\right]
\\\notag\leq& 2\mathbb{E}\left[\|S^{A_{1:n}}_{J}[\Lambda_{J}^{0:n}]L_{\tau}X-L_{\tau}S^{A_{1:n}}_{J}[\Lambda_{J}^{0:n}]X\|^{2}\right]
+2\mathbb{E}\left[\|L_{\tau}S^{A_{1:n}}_{J}[\Lambda_{J}^{0:n}]X-S^{A_{1:n}}_{J}[\Lambda_{J}^{0:n}]X\|^{2}\right].
\\\label{deformation_stability_estimate1}=&2\mathbb{E}\left[\|[S^{A_{1:n}}_{J}[\Lambda_{J}^{0:n}],L_{\tau}]X\|^{2}\right]
+2\mathbb{E}\left[\|L_{\tau}S^{A_{1:n}}_{J}[\Lambda_{J}^{0:n}]X-S^{A_{1:n}}_{J}[\Lambda_{J}^{0:n}]X\|^{2}\right],
\end{align}
where $[S^{A_{1:n}}_{J}[\Lambda_{J}^{0:n}],L_{\tau}] = S^{A_{1:n}}_{J}[\Lambda_{J}^{0:n}]L_{\tau}-L_{\tau}S^{A_{1:n}}_{J}[\Lambda_{J}^{0:n}]$ is the commutator of $S^{A_{1:n}}_{J}[\Lambda_{J}^{0:n}]$ and $L_{\tau}$.
For the first component $\mathbb{E}\left[\|[S^{A_{1:n}}_{J}[\Lambda_{J}^{0:n}],L_{\tau}]X\|^{2}\right]$
on the right hand side of (\ref{deformation_stability_estimate1}), we have the following estimate
\begin{equation}\label{eq:norm[S,L]}
\mathbb{E}\left[\|\left[S_{J}^{A_{1:n}}\left[\Lambda_{J}^{0:n}\right],L_{\tau}\right]X\|^{2}\right]
\leq (n+1)^{2}
\mathbb{E}\left[\|\left[W_{J},L_{\tau}\right]\|^{2}\right]
\mathbb{E}\left[|X|^{2}\right],
\end{equation}
where $W_{J} = \{P_{J},\left(W[j]\right)_{j\in \Lambda_{J}}\}$ and
$\left[W_{J},L_{\tau}\right] = \{[P_{J},L_{\tau}],([W[j],L_{\tau}])_{j\in \Lambda_{J}}\}$.
The proof of (\ref{eq:norm[S,L]}) is in \ref{sec:proof_lemma:norm[S,L]}.
We note that (\ref{eq:norm[S,L]}) coincides with the equation (H.3) in \cite[Appnedix H]{mallat2012group},
in which $A_{1}(\cdot) = A_{2}(\cdot) = \cdots = A_{n}(\cdot) = |\cdot|$.
Note that $\|[W_{J},L_{\tau}]\| = 0$ if $\tau$ is a constant function.
For the case $\|\tau^{'}\|_{\infty}>0$, \cite[Lemma 2.14]{mallat2012group} proved that
if $\tau\in \mathbf{C}^{2}(\mathbb{R})$ and $\|\tau^{'}\|_{\infty}\leq\frac{1}{2}$ almost surely,
then there exists $C>0$ such that
\begin{equation}\label{norm_[W,L]}
\mathbb{E}\left[\|[W_{J},L_{\tau}]\|^{2}\right] \leq C\ K(\tau)
\end{equation}
for all $J\in \mathbb{Z}$.

Next, we estimate the second component
$\mathbb{E}\left[\|L_{\tau}S^{A_{1:n}}_{J}[\Lambda_{J}^{0:n}]X-S^{A_{1:n}}_{J}[\Lambda_{J}^{0:n}]X\|^{2}\right]$
on the right hand side of (\ref{deformation_stability_estimate1}).
Because
\begin{equation*}
L_{\tau}S_{J}^{A_{1:n}}[\Lambda_{J}^{0:n}]X-S^{A_{1:n}}_{J}[\Lambda_{J}^{0:n}]
=
\left(L_{\tau}P_{J}-P_{J}\right) U^{A_{1:n}}[\Lambda_{J}^{0:n}]X,
\end{equation*}
we have
\begin{align}\notag
\mathbb{E}\left[\|L_{\tau}S_{J}^{A_{1:n}}[\Lambda_{J}^{0:n}]X-S_{J}^{A_{1:n}}[\Lambda_{J}^{0:n}]X\|^{2}\right]
\leq&
\mathbb{E}\left[\|L_{\tau}P_{J}-P_{J}\|^{2}\right]
\mathbb{E}\left[\|U^{A_{1:n}}[\Lambda_{J}^{0:n}]X\|^{2}\right]
\\\label{thm:secondpart:proof1}\leq&
(n+1)\mathbb{E}\left[\|L_{\tau}P_{J}-P_{J}\|^{2}\right]
\mathbb{E}\left[|X|^{2}\right],
\end{align}
where the first inequality follows from Lemma \ref{theorem_cite1} in \ref{sec:theorem_cite1} (see also \cite[Lemma 4.8 and the proof of Theorem 4.7]{mallat2012group})
and the second inequality follows from the non-expansiveness of $\{U^{A_{1},...,A_{k}}[\Lambda_{J}^{k}]\}_{k=0,1,...,n}$,
i.e., (\ref{nonincreasing_U}).
From Lemma \ref{theorem_cite3} in \ref{sec:theorem_cite3} (see also \cite[Lemma 2.11]{mallat2012group}),
we know that there exists a constant $C>0$ such that
\begin{align}\label{thm:secondpart:proof2}
\mathbb{E}\left[\|L_{\tau}P_{J}-P_{J}\|^{2}\right]
\leq C^{2}2^{-2J}E\left[\|\tau\|_{\infty}^{2}\right]
\end{align}
for all $J\in \mathbb{Z}$.
Combining (\ref{thm:secondpart:proof1}) and (\ref{thm:secondpart:proof2}) yields
\begin{equation}\label{thm:secondpart:eq}
\mathbb{E}\left[\|L_{\tau}S_{J}^{A_{1:n}}[\Lambda_{J}^{0:n}]X-S_{J}^{A_{1:n}}[\Lambda_{J}^{0:n}]X\|^{2}\right]
\leq
C^{2}(n+1)2^{-2J}\mathbb{E}\left[\|\tau\|_{\infty}^{2}\right]
\mathbb{E}\left[|X|^{2}\right].
\end{equation}

By substituting (\ref{eq:norm[S,L]}), (\ref{norm_[W,L]}) and (\ref{thm:secondpart:eq})
into (\ref{deformation_stability_estimate1}), we obtain
\begin{align}\notag
\mathbb{E}\left[\|S_{J}^{A_{1:n}}[\Lambda_{J}^{0:n}]L_{\tau}X-S_{J}^{A_{1:n}}[\Lambda_{J}^{0:n}]X\|^{2}\right]
\leq
C\left\{(n+1)^{2}K(\tau)+2^{-2J}(n+1)\mathbb{E}\left[\|\tau\|_{\infty}^{2}\right]\right\}\mathbb{E}\left[|X|^{2}\right]
\end{align}
for some constant $C>0$.

\section{Proof of (\ref{eq:norm[S,L]})}\label{sec:proof_lemma:norm[S,L]}

We follow the idea and symbols used for deriving of equation (H.3) in \cite[Appnedix H]{mallat2012group}
to decompose the NAST coefficients within the first $n$ layers, which are defined in (\ref{def:pooled_NAST_first_m}),
as follows
\begin{align}\notag
&S_{J}^{A_{1:n}}\left[\Lambda_{J}^{0:n}\right]L_{\tau}
\\\notag=& \left\{P_{J}L_{\tau},\left\{P_{J}U^{A_{1:m}}[\Lambda_{J}^{m}]L_{\tau}\right\}_{1\leq m\leq n}\right\}
\\\notag=& \left\{L_{\tau}P_{J}+\left[P_{J},L_{\tau}\right],\left\{P_{J}U^{A_{2:m}}[\Lambda_{J}^{m-1}]L_{\tau}U^{A_{1}}[\Lambda_{J}]
+P_{J}U^{A_{2:m}}[\Lambda_{J}^{m-1}]\left[U^{A_{1}}[\Lambda_{J}],L_{\tau}\right]\right\}_{1\leq m\leq n}\right\}
\\\notag=&\left\{L_{\tau}P_{J},\left\{P_{J}U^{A_{2:m}}[\Lambda_{J}^{m-1}]L_{\tau}U^{A_{1}}[\Lambda_{J}]
\right\}_{1\leq m\leq n}\right\}
\\\label{iterative1}+&\left\{\left[P_{J},L_{\tau}\right],\left\{P_{J}U^{A_{2:m}}[\Lambda_{J}^{m-1}]\left[U^{A_{1}}[\Lambda_{J}],L_{\tau}\right]\right\}_{1\leq m\leq n}\right\}.
\end{align}
Note that some elements within the curly bracket above may also be a set of NAST coefficients, such as  $P_{J}U^{A_{2:m}}[\Lambda_{J}^{m-1}]\left[U^{A_{1}}[\Lambda_{J}],L_{\tau}\right]$, and
the summation is performed independently for different layer index $[j_{1},j_{2},...,j_{m}]$, where $0\leq m\leq n$.
The sequence of operators $\{P_{J}U^{A_{2:m}}[\Lambda_{J}^{m-1}]L_{\tau}\}_{1\leq m\leq n}$ on the right hand side of the third equality
is equivalent to $S_{J}^{A_{2:n}}\left[\Lambda_{J}^{0:n-1}\right]L_{\tau}$.
A substitution of the sequence of operators $\{P_{J}U^{A_{2:m}}[\Lambda_{J}^{m-1}]L_{\tau}\}_{1\leq m\leq n}$ in (\ref{iterative1}) by
the formula (\ref{iterative1}) for $S_{J}^{A_{2:n}}\left[\Lambda_{J}^{0:n-1}\right]L_{\tau}$ gives
\begin{align}\notag
&S_{J}^{A_{1:n}}\left[\Lambda_{J}^{0:n}\right]L_{\tau}
\\\notag=&\left\{L_{\tau}P_{J},L_{\tau}P_{J}U^{A_{1}}[\Lambda_{J}], \left\{P_{J}U^{A_{3:m}}[\Lambda_{J}^{m-2}]L_{\tau}U^{A_{1:2}}[\Lambda_{J}^{2}]
\right\}_{2\leq m\leq n}\right\}
\\\notag+&\left\{\left[P_{J},L_{\tau}\right]U^{A_{1}}[\Lambda_{J}]
,\left\{P_{J}U^{A_{3:m}}[\Lambda_{J}^{m-2}]\left[U^{A_{2}}[\Lambda_{J}],L_{\tau}\right]U^{A_{1}}[\Lambda_{J}]\right\}_{2\leq m\leq n}\right\}
\\\label{iterative2}+&\left\{\left[P_{J},L_{\tau}\right],
\left\{P_{J}U^{A_{2:m}}[\Lambda_{J}^{m-1}]\left[U^{A_{1}}[\Lambda_{J}],L_{\tau}\right]\right\}_{1\leq m\leq n}\right\}.
\end{align}
Note that the sequence of operators $\{P_{J}U^{A_{3:m}}[\Lambda_{J}^{m-2}]L_{\tau}\}_{2\leq m\leq n}$ on the right hand side of the equality
are equivalent to $S_{J}^{A_{3:n}}\left[\Lambda_{J}^{0:n-2}\right]L_{\tau}$.
A substitution of the sequence of operators $\{P_{J}U^{A_{3:m}}[\Lambda_{J}^{m-2}]L_{\tau}\}_{2\leq m\leq n}$ in (\ref{iterative2})
by the formula (\ref{iterative1}) for $S_{J}^{A_{3:n}}\left[\Lambda_{J}^{0:n-2}\right]L_{\tau}$ gives
\begin{align}\notag
&S_{J}^{A_{1:n}}\left[\Lambda_{J}^{0:n}\right]L_{\tau}
\\\notag=&\left\{L_{\tau}P_{J},L_{\tau}P_{J}U^{A_{1}}[\Lambda_{J}],L_{\tau}P_{J}U^{A_{1:2}}[\Lambda_{J}^{2}],
\left\{P_{J}U^{A_{4:m}}[\Lambda_{J}^{m-3}]L_{\tau}U^{A_{1:3}}[\Lambda_{J}^{3}]
\right\}_{3\leq m\leq n}\right\}
\\\notag+&\left\{\left[P_{J},L_{\tau}\right]U^{A_{1:2}}[\Lambda_{J}^{2}],
\left\{P_{J}U^{A_{4:m}}[\Lambda_{J}^{m-3}]\left[U^{A_{3}}[\Lambda_{J}],L_{\tau}\right]U^{A_{1:2}}[\Lambda_{J}^{2}]\right\}_{3\leq m\leq n}\right\}
\\\notag+&\left\{\left[P_{J},L_{\tau}\right]U^{A_{1}}[\Lambda_{J}]
,\left\{P_{J}U^{A_{3:m}}[\Lambda_{J}^{m-2}]\left[U^{A_{2}}[\Lambda_{J}],L_{\tau}\right]U^{A_{1}}[\Lambda_{J}]\right\}_{2\leq m\leq n}\right\}
\\\label{iterative3}+&\left\{\left[P_{J},L_{\tau}\right],
\left\{P_{J}U^{A_{2:m}}[\Lambda_{J}^{m-1}]\left[U^{A_{1}}[\Lambda_{J}],L_{\tau}\right]\right\}_{1\leq m\leq n}\right\}.
\end{align}
After $n$ substitutions, we get
\begin{align}\notag
&S_{J}^{A_{1:n}}\left[\Lambda_{J}^{0:n}\right]L_{\tau}
\\\notag=&\left\{L_{\tau}P_{J},L_{\tau}P_{J}U^{A_{1}}[\Lambda_{J}],L_{\tau}P_{J}U^{A_{1:2}}[\Lambda_{J}^{2}],\ldots,
L_{\tau}P_{J}U^{A_{1:n}}[\Lambda_{J}^{n}]\right\}
\\\notag+&\left\{\left[P_{J},L_{\tau}\right]U^{A_{1:n}}[\Lambda_{J}^{n}]\right\}
\\\notag+&\left\{\left[P_{J},L_{\tau}\right]U^{A_{1:(n-1)}}[\Lambda_{J}^{n-1}],
\left\{P_{J}\left[U^{A_{n}}[\Lambda_{J}],L_{\tau}\right]U^{A_{1:(n-1)}}[\Lambda_{J}^{n-1}]\right\}\right\}
\\\notag+&\cdots
\\\notag+&\left\{\left[P_{J},L_{\tau}\right]U^{A_{1:2}}[\Lambda_{J}^{2}],
\left\{P_{J}U^{A_{4:m}}[\Lambda_{J}^{m-3}]\left[U^{A_{3}}[\Lambda_{J}],L_{\tau}\right]U^{A_{1:2}}[\Lambda_{J}^{2}]\right\}_{3\leq m\leq n}\right\}
\\\notag+&\left\{\left[P_{J},L_{\tau}\right]U^{A_{1}}[\Lambda_{J}]
,\left\{P_{J}U^{A_{3:m}}[\Lambda_{J}^{m-2}]\left[U^{A_{2}}[\Lambda_{J}],L_{\tau}\right]U^{A_{1}}[\Lambda_{J}]\right\}_{2\leq m\leq n}\right\}
\\\notag+&\left\{\left[P_{J},L_{\tau}\right],
\left\{P_{J}U^{A_{2:m}}[\Lambda_{J}^{m-1}]\left[U^{A_{1}}[\Lambda_{J}],L_{\tau}\right]\right\}_{1\leq m\leq n}\right\}
\\\label{iterative3}=&L_{\tau}S_{J}^{A_{1:n}}\left[\Lambda_{J}^{0:n}\right]+\overset{n}{\underset{\ell=0}{\sum}} K_{\ell},
\end{align}
where
\begin{align}\notag
K_{0} = \left\{\left[P_{J},L_{\tau}\right]U^{A_{1:n}}[\Lambda_{J}^{n}]\right\},
\end{align}
and
\begin{align}\notag
K_{\ell} = \left\{\left[P_{J},L_{\tau}\right]U^{A_{1:(n-\ell)}}[\Lambda_{J}^{n-\ell}],
S_{J}^{A_{(n+2-\ell):n}}[\Lambda_{J}^{0:\ell-1}]\left[U^{A_{n+1-\ell}}[\Lambda_{J}],L_{\tau}\right]U^{A_{1:(n-\ell)}}[\Lambda_{J}^{n-\ell}]\right\}
\end{align}
for $\ell = 1,2,...,n.$
By default,
$S_{J}^{A_{(n+1):n}}[\Lambda_{J}^{0}] = P_{J}$ and $U^{A_{1:0}}[\Lambda_{J}^{0}] = \textup{Id}$.
From (\ref{iterative3}), we get
\begin{align}\notag
\left[S_{J}^{A_{1:n}}\left[\Lambda_{J}^{0:n}\right],L_{\tau}\right] X
= S_{J}^{A_{1:n}}\left[\Lambda_{J}^{0:n}\right]L_{\tau}X
-L_{\tau}S_{J}^{A_{1:n}}\left[\Lambda_{J}^{0:n}\right]X
=\overset{n}{\underset{\ell=0}{\sum}} K_{\ell}X.
\end{align}
By the inequality $\|a_{0}+a_{1}+\cdots+a_{n}\|^{2}\leq (n+1)(\|a_{0}\|^{2}+\|a_{1}\|^{2}+\cdots+\|a_{n}\|^{2})$
for vectors $a_{0},a_{1},...,a_{n}$ with the same dimension,
where $\|\cdot\|$ is the Euclidean norm,
\begin{align}\label{sum_norm_K}
\mathbb{E}\left[\|\left[S_{J}^{A_{1:n}}\left[\Lambda_{J}^{0,n}\right],L_{\tau}\right]X\|^{2}\right]
= \mathbb{E}\left[\|\overset{n}{\underset{\ell=0}{\sum}} K_{\ell}X\|^{2}\right]
\leq (n+1)\overset{n}{\underset{\ell=0}{\sum}}\mathbb{E}\left[\|K_{\ell}X\|^{2}\right].
\end{align}
From the definition of $K_{0}$, we have
\begin{align}\notag
\mathbb{E}\left[\|K_{0}X\|^{2}\right] \leq& \mathbb{E}\left[\|\left[P_{J},L_{\tau}\right]\|^{2}\right]
\mathbb{E}\left[\|U^{A_{1:n}}[\Lambda_{J}^{n}]X\|^{2}\right]
\\\notag\leq& \mathbb{E}\left[\|\left[P_{J},L_{\tau}\right]\|^{2}\right]  \mathbb{E}\left[|X|^{2}\right]
\\\label{EK_ellX=0} \leq& \mathbb{E}\left[\|\left[W_{J},L_{\tau}\right]\|^{2}\right]  \mathbb{E}\left[|X|^{2}\right],
\end{align}
where the first inequality follows from Lemma \ref{theorem_cite1} in \ref{sec:theorem_cite1} (see also \cite[Lemma 4.8 and the proof of Theorem 4.7]{mallat2012group})
and the second inequality is obtained by the non-expansiveness of the operator $U^{A_{1:n}}[\Lambda_{J}^{n}]$ (see (\ref{nonincreasing_U})).
The third inequality comes from the fact
$\left[W_{J},L_{\tau}\right] = \left\{\left[P_{J},L_{\tau}\right], ([W[j],L_{\tau}])_{j\in \Lambda_{J}}\right\}.$

For $\ell = 1,...,n$,
\begin{align}\notag
\mathbb{E}\left[\|K_{\ell}X\|^{2}\right] \leq&
\mathbb{E}\left[\|\left[P_{J},L_{\tau}\right]U^{A_{1:(n-\ell)}}[\Lambda_{J}^{n-\ell}]X\|^{2}\right]
\\\label{EK_ellX>0}&+\mathbb{E}\left[\|S_{J}^{A_{(n+2-\ell):n}}[\Lambda_{J}^{0:\ell-1}]
\left[U^{A_{n+1-\ell}}[\Lambda_{J}],L_{\tau}\right]U^{A_{1:(n-\ell)}}[\Lambda_{J}^{n-\ell}]X\|^{2}
\right].
\end{align}
For the first part in (\ref{EK_ellX>0}), Lemma \ref{theorem_cite1} in \ref{sec:theorem_cite1} implies that
\begin{align}\notag
\mathbb{E}\left[\|\left[P_{J},L_{\tau}\right]U^{A_{1:(n-\ell)}}[\Lambda_{J}^{n-\ell}]X\|^{2}\right]
\leq&
\mathbb{E}\left[\|\left[P_{J},L_{\tau}\right]\|^{2}\right]
\mathbb{E}\left[\|U^{A_{1:(n-\ell)}}[\Lambda_{J}^{n-\ell}]X\|^{2}\right]
\\\label{proof:S_firstpart}
\leq&
\mathbb{E}\left[\|\left[P_{J},L_{\tau}\right]\|^{2}\right]
\mathbb{E}\left[|X|^{2}\right],
\end{align}
where the last inequality follows from the non-expansiveness of $U^{A_{1:(n-\ell)}}[\Lambda_{J}^{n-\ell}]$ (see (\ref{nonincreasing_U})).
For the second part in (\ref{EK_ellX>0}),
\begin{align}\notag
&\mathbb{E}\left[\|S_{J}^{A_{(n+2-\ell):n}}[\Lambda_{J}^{0:\ell-1}]
\left[U^{A_{n+1-\ell}}[\Lambda_{J}],L_{\tau}\right]U^{A_{1:(n-\ell)}}[\Lambda_{J}^{n-\ell}]X\|^{2}
\right]
\\\notag\leq&
\mathbb{E}\left[\|
\left[U^{A_{n+1-\ell}}[\Lambda_{J}],L_{\tau}\right]U^{A_{1:(n-\ell)}}[\Lambda_{J}^{n-\ell}]X\|^{2}
\right]
\\\notag
\leq& \mathbb{E}\left[\|
\left[W[\Lambda_{J}],L_{\tau}\right]U^{A_{1:(n-\ell)}}[\Lambda_{J}^{n-\ell}]X\|^{2}
\right]
\\\notag\leq& \mathbb{E}\left[\|
\left[W[\Lambda_{J}],L_{\tau}\right]\|^{2}\right]
\mathbb{E}\left[\|U^{A_{1:(n-\ell)}}[\Lambda_{J}^{n-\ell}]X\|^{2}
\right]
\\\label{proof:S_secondpart}\leq& \mathbb{E}\left[\|
\left[W[\Lambda_{J}],L_{\tau}\right]\|^{2}\right]
\mathbb{E}\left[|X|^{2}
\right],
\end{align}
where the first and the fourth inequalities follow from the non-expansiveness of $S_{J}^{A_{(n+2-\ell):n}}[\Lambda_{J}^{0:\ell-1}]$
and $U^{A_{1:(n-\ell)}}[\Lambda_{J}^{n-\ell}]$ (see (\ref{nonexpansive_S}) and (\ref{nonincreasing_U})).
The second inequality follows from Lemma \ref{lemma:U<W} in \ref{sec:proof_lemma:U<W}.
The third inequality follows from Lemma \ref{theorem_cite1} in \ref{sec:theorem_cite1}.
Substituting (\ref{proof:S_firstpart}) and (\ref{proof:S_secondpart}) into (\ref{EK_ellX>0}) yields
\begin{align}\label{EK_ellX>0_v2}
\mathbb{E}\left[\|K_{\ell}X\|^{2}\right] \leq
\mathbb{E}\left[\|\left[W_{J},L_{\tau}\right]\|^{2}\right]
\mathbb{E}\left[|X|^{2}
\right],\ \ell=1,2,...,n.
\end{align}
By (\ref{EK_ellX=0}) and (\ref{EK_ellX>0_v2}), we can rewrite (\ref{sum_norm_K}) as follows
\begin{align*}
\mathbb{E}\left[\|\left[S_{J}^{A_{1:n}}\left[\Lambda_{J}^{0,n}\right],L_{\tau}\right]X\|^{2}\right]
\leq (n+1)^{2}
\mathbb{E}\left[\left[W_{J},L_{\tau}\right]\|^{2}\right]
\mathbb{E}\left[|X|^{2}\right].
\end{align*}
\qed

\section{}\label{sec:theorem_cite1}

\begin{Lemma}\label{theorem_cite1}(\cite[Lemma 4.8 and the proof of Theorem 4.7]{mallat2012group})
Let $K_{\tau}$ be an integral operator with a kernel $k_{\tau}(x,u)$
that depends upon a random process $\tau$. If the following two conditions are satisfied:
\begin{equation}\label{cond1}
E\left[k_{\tau}(x,u)k_{\tau}(x,u^{'})\right] = \overline{k}_{\tau}(x-u,x-u^{'})
\end{equation}
and
\begin{equation}\label{cond2}
\iint|\overline{k}_{\tau}(v,v^{'})|\ |v-v^{'}| dv\ dv^{'}<\infty,
\end{equation}
then for any stationary process $Y$ independent of $\tau$, $\mathbb{E}\left[|K_{\tau}Y(t)|^{2}\right]$ does not depend upon $t$ and
\begin{equation}
\mathbb{E}\left[|K_{\tau}Y|^{2}\right]\leq \mathbb{E}\left[\|K_{\tau}\|^{2}\right]\mathbb{E}\left[|Y|^{2}\right],
\end{equation}
where $\|K_{\tau}\|$ is the operator norm in $L^{2}(\mathbb{R})$ for each realization of $\tau$.
The conditions (\ref{cond1}) and (\ref{cond2}) are satisfied by the kernels of the commutators
$L_{\tau}P_{J}-P_{J}$, $[P_{J},L_{\tau}]$, $\{[W[j],L_{\tau}]\}_{j\in \Lambda_{J}}$, and $[W_{J},L_{\tau}]$.
\end{Lemma}

\section{}\label{sec:theorem_cite3}
\begin{Lemma}\label{theorem_cite3}\cite[Lemma 2.11]{mallat2012group}
There exists $C>0$ such that for all $\tau\in \mathbf{C}^{2}(\mathbb{R})$
with $\|\tau^{'}\|_{\infty}\leq\frac{1}{2}$
we have
\begin{equation}
\|L_{\tau}P_{J}f-P_{J}f\|\leq C\|f\|_{L^{2}}2^{-J}\|\tau\|_{\infty}
\end{equation}
for all $f\in L^{2}(\mathbb{R}).$
\end{Lemma}

\section{}\label{sec:proof_lemma:U<W}
\begin{Lemma}\label{lemma:U<W}
If the activation function/operator $A$ is Lipschitz continuous with Lipschitz constant $L_{A}\leq 1$ and $A(0)=0$, then for any strictly stationary process $X$ with $\mathbb{E}[|X|^{2}]<\infty$,
\begin{equation*}
\mathbb{E}\left[\|[U^{A}_{J},L_{\tau}]X\|^{2}\right] \leq \mathbb{E}\left[\|[W_{J},L_{\tau}]X\|^{2}\right],
\end{equation*}
where $W_{J}X = \{P_{J}X,\left(W[j]X\right)_{j\in \Lambda_{J}}\}$
and the expectations above are defined by
\begin{equation*}
\mathbb{E}\left[\|[U^{A}_{J},L_{\tau}]X\|^{2}\right] = \mathbb{E}\left[|[P_{J},L_{\tau}]X|^{2}\right]+\underset{j\in \Lambda_{J}}{\sum}\mathbb{E}\left[|[AW[j],L_{\tau}]X|^{2}\right]
\end{equation*}
and
\begin{equation*}
\mathbb{E}\left[\|[W_{J},L_{\tau}]X\|^{2}\right] = \mathbb{E}\left[|[P_{J},L_{\tau}]X|^{2}\right]+\underset{j\in \Lambda_{J}}{\sum}\mathbb{E}\left[|[W[j],L_{\tau}]X|^{2}\right].
\end{equation*}
\end{Lemma}
\begin{proof}
For any stationary process $X$,
\begin{align}\notag
[U^{A}_{J},L_{\tau}]X =& U^{A}_{J}L_{\tau}X-L_{\tau}U^{A}_{J}X
\\\notag=&
\left\{P_{J}L_{\tau}X,\left(AW[j]L_{\tau}X\right)_{j\in\Lambda_{J}}\right\}-
L_{\tau}\left\{P_{J}X,\left(AW[j]X\right)_{j\in\Lambda_{J}}\right\}
\\\notag=&
\left\{\left[P_{J},L_{\tau}\right]X,A\left(W[j]L_{\tau}X-L_{\tau}W[j]X\right)_{j\in\Lambda_{J}}
\right\},
\end{align}
where the last equality follows from $L_{\tau}A = AL_{\tau}$. Hence,
\begin{align}\notag
\mathbb{E}\left[\|[U^{A}_{J},L_{\tau}]X\|^{2}\right] =& \mathbb{E}\left[|\left[P_{J},L_{\tau}\right]X|^{2}\right]
+\underset{j\in\Lambda_{J}}{\sum}\mathbb{E}\left[|A\left(W[j]L_{\tau}X-L_{\tau}W[j]X\right)|^{2}\right]
\\\notag\leq&
\mathbb{E}\left[|\left[P_{J},L_{\tau}\right]X|^{2}\right]
+\underset{j\in\Lambda_{J}}{\sum}\mathbb{E}\left[|W[j]L_{\tau}X-L_{\tau}W[j]X|^{2}\right]
\\\notag=&\mathbb{E}\left[\|[W_{J},L_{\tau}]X\|^{2}\right],
\end{align}
where the first inequality follows from the Lipschitz continuity of $A$ with $L_{A}\leq1$ and $A(0)=0$.
\end{proof}

\section{Delta Method}\label{Lemma:Delta_method}
\begin{Lemma} For any $M\in \mathbb{N}$,
let $\mathbf{U}_{j} = (U^{(1)}_{j},\cdots, U^{(M)}_{j})$ be a sequence of random vectors, where $j\in \mathbb{N}$.
Suppose that there exists a normalizer $N_{j}$ and a random vector $\mathbf{Z}$ such that
\begin{equation}
N_{j}\mathbf{U}_{j}\overset{d}{\Rightarrow} \mathbf{Z}
\end{equation}
when $j\rightarrow\infty$. Denote $\mathcal{A}_{2}(\mathbf{x}) = (A_{2}(x_{1}),\cdots, A_{2}(x_{M}))$ for $\mathbf{x}= (x_{1},\ldots,x_{M})\in \mathbb{R}^{M}.$ If $A_{2}$ satisfies Assumption \ref{Assumption:4.1:A2:differentiable}, then we have
\begin{equation}
N_{j}\left(\mathcal{A}_{2}(\mathbf{U}_{j})-\mathcal{A}_{2}(\mathbf{0})\right) \overset{d}{\Rightarrow} A_{2}^{'}(0)\mathbf{Z}.
\end{equation}
when $j\rightarrow\infty$.
\end{Lemma}

\section{}
\begin{Lemma}\label{lemma:1}
Let $f: \mathbb{R}\rightarrow [0,\infty)$ be the spectral density function of a second-order stationary random process.
If it has the form
\begin{equation*}
f(\lambda)=\frac{B(\lambda)}{|\lambda|^{1-\gamma}},\ \gamma\in (0,1),
\end{equation*}
for some non-negative bounded function $B$,
then for any $\ell\geq 2$,
there exists a bounded function $B_{\ell}$ such that the
$\ell$-fold convolution $f^{\star \ell}$ of $f$ can be expresses as
\begin{equation}\label{singularconvolution1}
f^{\star \ell}(\lambda)=
\left\{\begin{array}{lr}
B_{\ell}(\lambda)|\lambda|^{\ell\gamma-1},\ \ & \textup{if } \ell\gamma<1,
\\
B_{\ell}(\lambda)\textup{ln}(2+\frac{1}{|\lambda|}),\ \ & \textup{if } \ell\gamma=1,
\\
B_{\ell}(\lambda),\ \ & \textup{if } \ell\gamma> 1.
\end{array}
\right.
\end{equation}
Moreover, for the case $\ell\gamma> 1$, the function $B_{\ell}$ is continuous.
\end{Lemma}

\begin{proof}
We consider three cases.

\noindent{\bf Case 1: $\ell\gamma<1$.}
The result for this case follows from using the following observation iteratively.
Consider two spectral densities $f_{1}$ and $f_{2}$, which have the form
\begin{equation}\label{proofsingularconvolution2}
f_{j}(\lambda)=\frac{K_{j}(\lambda)}{|\lambda|^{1-\gamma_{j}}},\
\gamma_{j}\in(0,1),\ j=1,2,
\end{equation}
where $\gamma_{1}+\gamma_{2}<1$ and $K_{1}(\lambda)$ and $K_{2}(\lambda)$ are nonnegative and bounded functions so that
\begin{equation}\label{appendix:L1_spectral}
f_{1},\ f_{2}\in L^{1}(\mathbb{R}).
\end{equation}
Riesz's composition formula \cite[p. 71]{du1970introduction} implies that for any $\lambda\neq0$,
\begin{align}\notag
f_{1}\star f_{2}(\lambda)=&\,
\int_{\mathbb{R}}
\frac{K_{1}(\lambda-\eta)}{|\lambda-\eta|^{1-\gamma_{1}}}\frac{K_{2}(\eta)}{|\eta|^{1-\gamma_{2}}}d\eta
\leq \| K_{1}\|_{\infty}
\| K_{2}\|_{\infty}
\int_{\mathbb{R}}
\frac{1}{|\lambda-\eta|^{1-\gamma_{1}}}\frac{1}{|\eta|^{1-\gamma_{2}}}d\eta
\\\label{proofsingularconvolution_part3}=&\,
\| K_{1}\|_{\infty}
\| K_{2}\|_{\infty}
\frac{\Gamma(\frac{\gamma_{1}}{2})\Gamma(\frac{\gamma_{2}}{2})\Gamma(\frac{1-\gamma_{1}-\gamma_{2}}{2})}{\Gamma(\frac{1-\gamma_{1}}{2})\Gamma(\frac{1-\gamma_{2}}{2})\Gamma(\frac{\gamma_{1}+\gamma_{2}}{2})}
|\lambda|^{\gamma_{1}+\gamma_{2}-1}.
\end{align}

\noindent{\bf Case 2: $\ell\gamma=1.$}
The idea of the following proof comes from \cite[p.115, Theorem
3]{smirnov2014course} and \cite[p.160, Theorem 8.8]{krasnosel1976integral}.
Let $f_{1}=f^{\star (\ell-1)}$ and $f_{2}=f$. From the result of Case 1, both $f_{1}$ and $f_{2}$ have the form (\ref{proofsingularconvolution2}) with
$\gamma_{1}=(\ell-1)\gamma\in(0,1)$, $\gamma_{2}=\gamma\in(0,1)$, and $\gamma_{1}+\gamma_{2}=1$.
For any $\lambda\neq0$, denote $\widehat{\lambda} = \lambda/|\lambda|$. Note that by a change of variable we have
\begin{align}\notag
f^{\star\ell}(\lambda) = f_{1}\star f_{2}(\lambda)
=&\,
\int_{\mathbb{R}}\frac{K_{1}(\lambda-\eta)K_{2}(\eta)}
{|\lambda-\eta|^{1-\gamma_{1}}|\eta|^{1-\gamma_{2}}}d\eta
=
\big(\int_{\{|\eta|<2\}}+\int_{|\eta|>2}\big)
\frac{K_{1}(|\lambda|(\widehat{\lambda}-\eta))K_{2}(|\lambda|\eta)}
{|\widehat{\lambda}-\eta|^{1-\gamma_{1}}|\eta|^{1-\gamma_{2}}}d\eta
\\\notag=:&\,J_{1}(\lambda)+J_{2}(\lambda).
\end{align}
By a direct calculation, we have
\begin{equation}\label{appendixlog1}
J_{1}(\lambda)\leq
\int_{\{|\eta|<2\}}
\frac{\parallel \hspace{-0.15cm}K_{1}\hspace{-0.15cm}\parallel_{\infty}
\parallel \hspace{-0.15cm}K_{2}\hspace{-0.15cm}\parallel_{\infty}
}
{|\widehat{\lambda}-\eta|^{1-\gamma_{1}}|\eta|^{1-\gamma_{2}}}d\eta
\leq \max_{|x|=1}
\int_{\{|\eta|<2\}}
\frac{\parallel \hspace{-0.15cm}K_{1}\hspace{-0.15cm}\parallel_{\infty}
\parallel \hspace{-0.15cm}K_{2}\hspace{-0.15cm}\parallel_{\infty}
}
{|x-\eta|^{1-\gamma_{1}}|\eta|^{1-\gamma_{2}}}d\eta
<\infty,
\end{equation}
where the last integral is independent of $\lambda$. For $J_2$, we have
\begin{align*}
J_{2}(\lambda)=
\big(\hspace{-0.2cm}
\underset{{\{2<|\eta|<2(2+\frac{1}{|\lambda|})\}}}{\int}
\hspace{-0.2cm}+\hspace{-0.2cm}\underset{\{|\eta|>2(2+\frac{1}{|\lambda|})\}}{\int}
\hspace{-0.2cm}\big)
\frac{K_{1}(|\lambda|(\widehat{\lambda}-\eta))K_{2}(|\lambda|\eta)}
{|\widehat{\lambda}-\eta|^{1-\gamma_{1}}|\eta|^{1-\gamma_{2}}}d\eta
=J_{2,1}(\lambda)+J_{2,2}(\lambda).
\end{align*}
Because $|\widehat{\lambda}-\eta|\geq |\eta|-1$ and $\gamma_{1}+\gamma_{2}=1$,
\begin{align}\notag
J_{2,1}(\lambda_{0})\leq
\underset{{\{2<|\eta|<2(2+\frac{1}{|\lambda|})\}}}{\int}
\frac{\parallel \hspace{-0.15cm}K_{1}\hspace{-0.15cm}\parallel_{\infty}
\parallel \hspace{-0.15cm}K_{2}\hspace{-0.15cm}\parallel_{\infty}
d\eta}{(|\eta|-1)^{1-\gamma_{1}}|\eta|^{1-\gamma_{2}}}
=
2\int_{2}^{2(2+\frac{1}{|\lambda|})}
\frac{\parallel \hspace{-0.15cm}K_{1}\hspace{-0.15cm}\parallel_{\infty}
\parallel \hspace{-0.15cm}K_{2}\hspace{-0.15cm}\parallel_{\infty}
dr}{(r-1)^{1-\gamma_{1}}r^{1-\gamma_{2}}}
\\\label{appendixlog2}
=2
\parallel \hspace{-0.15cm}K_{1}\hspace{-0.15cm}\parallel_{\infty}
\parallel \hspace{-0.15cm}K_{2}\hspace{-0.15cm}\parallel_{\infty}
\int_{2}^{2(2+\frac{1}{|\lambda|})}
\hspace{-0.2cm}\frac{dr}{(1-\frac{1}{r})^{1-\gamma_{1}}r}
\leq
2^{2-\gamma_{1}}
\parallel \hspace{-0.15cm}K_{1}\hspace{-0.15cm}\parallel_{\infty}
\parallel \hspace{-0.15cm}K_{2}\hspace{-0.15cm}\parallel_{\infty}
\textup{ln}(2+\frac{1}{|\lambda|}).
\end{align}
On the other hand,
by substituting  $\eta=\tau/|\lambda|$
and using the inequality
$|\lambda-\tau|\geq |\tau|-|\lambda|\geq 2(1+2|\lambda|)-|\lambda|\geq 2$
for $|\tau|>2(1+2|\lambda|)$, we have
\begin{equation*}
J_{2,2}(\lambda)=\hspace{-0.15cm}
\underset{\{|\tau|>2(1+2|\lambda|)\}}{\int}
\hspace{-0.2cm}\frac{K_{1}(\lambda-\tau)K_{2}(\tau)}
{|\lambda-\tau|^{1-\gamma_{1}}|\tau|^{1-\gamma_{2}}}d\tau
\leq
\frac{\parallel\hspace{-0.15cm}K_{1}\hspace{-0.15cm}\parallel_{\infty}}{2^{1-\gamma_{1}}}
\hspace{-0.22cm}\underset{\{|\tau|>2(1+2|\lambda|)\}}{\int}\hspace{-0.2cm}
\frac{K_{2}(\tau)}
{|\tau|^{1-\gamma_{2}}}d\tau
\leq
\frac{\parallel\hspace{-0.15cm}K_{1}\hspace{-0.15cm}\parallel_{\infty}}{2^{1-\gamma_{1}}}\|f_{2}\|_{L^{1}}.
\end{equation*}
The last estimation, together with (\ref{appendixlog1}) and (\ref{appendixlog2}),
implies that
there exists a bounded function $B$
such that
$f^{\star\ell}(\lambda)=
B(\lambda)\textup{ln}(2+\frac{1}{|\lambda|})$.
\\
{\bf Case 3: $\ell\gamma>1.$}
We divide this case into two subcases as follows.

{\it Case 3.1: $(\ell-1)\gamma<1$ and $\ell\gamma>1$.}
Let $f_{1}=f^{\star (\ell-1)}$ and $f_{2}=f$.
Under this situation, by Case 1, $f_{1}$ and $f_{2}$ have the form (\ref{proofsingularconvolution2})
with $\gamma_{1}=(\ell-1)\gamma\in(0,1)$, $\gamma_{2}=\gamma\in(0,1)$, and
$\gamma_{1}+\gamma_{2}=\ell\gamma>1$.
Define $p_{1} = \frac{(1-\gamma_{1})+(1-\gamma_{2})}{(1-\gamma_{1})}\in(1,\infty)$ and
$p_{2} = \frac{(1-\gamma_{1})+(1-\gamma_{2})}{(1-\gamma_{2})}\in(1,\infty)$, which satisfy
$p_{1}(1-\gamma_{1})= p_{2}(1-\gamma_{2})<1$ and $\frac{1}{p_{1}}+\frac{1}{p_{2}}=1$.
Because $f_{1}\in L^{p_{1}}$ and $f_{2}\in L^{p_{2}}$,
by H$\ddot{\textup{o}}$lder's inequality,
$f_{1}\star f_{2}(\lambda)\leq \parallel\hspace{-0.12cm}f_{1}\hspace{-0.12cm}\parallel_{L^{p_{1}}}
\parallel\hspace{-0.12cm}f_{2}\hspace{-0.12cm}\parallel_{L^{p_{2}}}$ for any $\lambda\in \mathbb{R}$.
It implies the boundedness of $f^{\star\ell}$. Moreover, by \cite[Proposition 8.8]{folland1999real},
$f^{\star\ell}$ is uniformly continuous.

{\it Case 3.2: $(\ell-1)\gamma=1$ and $\ell\gamma>1$.}
Under this situation, the result of Case 2 implies that there exists a positive and bounded function $K$ such that
\begin{equation}\label{lemma_case3.2}
f^{\star (\ell-1)}(\cdot) = K(\cdot)\textup{ln}(2+\frac{1}{|\cdot|})\in L^{1}(\mathbb{R})
\end{equation}
Note that we can select $p_{1}\in(1,\infty)$ such that $f\in L^{p_{1}}$. Define $p_{2}$ by $\frac{1}{p_{1}}+\frac{1}{p_{2}}=1$.
From (\ref{lemma_case3.2}), we know that the behavior of $f^{\star (\ell-1)}(\lambda)$ as $\lambda\rightarrow0$ looks like
$\textup{ln}(1/|\lambda|)$. Hence, $f^{\star (\ell-1)} \in L^{p_{2}}$. By H$\ddot{\textup{o}}$lder's inequality,
$f^{\star (\ell-1)}\star f(\lambda)\leq \parallel\hspace{-0.12cm}f^{\star (\ell-1)}\hspace{-0.12cm}\parallel_{L^{p_{2}}}
\parallel\hspace{-0.12cm}f\hspace{-0.12cm}\parallel_{L^{p_{1}}}$ for any $\lambda\in \mathbb{R}$.
It implies the boundedness and continuity of $f^{\star\ell}$ \cite[Proposition 8.8]{folland1999real}.

\end{proof}

\section{}
\begin{Lemma}\label{lemma:2}
The series representation (\ref{thm1:kappa}) for the constant $\kappa$ in Theorem \ref{thm1:A_gaussian} can be expressed in terms of the integral of the covariance function
of $A_{1}(X\star \psi_{j_{1}})$ as follows
\begin{equation*}
\kappa = \left[\overset{\infty}{\underset{\ell=\mathbf{r}}{\sum}}C_{\sigma_{j_{1}},\ell}^{2}\sigma_{j_{1}}^{-2\ell}
f_{X\star\psi_{j_{1}}}^{\star\ell}(0)\right]^{\frac{1}{2}} = \left[\frac{1}{2\pi}\int_{\mathbb{R}}R_{A_{1}(X\star \psi_{j_{1}})}(t)dt\right]^{\frac{1}{2}}.
\end{equation*}
\end{Lemma}

\begin{proof}
First of all, we note that
\begin{equation}\label{integral_cov}
A_{1}\left(X\star \psi_{j_{1}}(t)\right) = A_{1}\left(\sigma_{j_{1}} \frac{X\star \psi_{j_{1}}(t)}{\sigma_{j_{1}}}\right) = A_{1}(\sigma_{j_{1}} Y_{j_{1}}(t)).
\end{equation}
From (\ref{hermiteexpansion}), we know that (\ref{integral_cov}) can be rewritten as
\begin{equation}\label{lamma2:expansion}
A_{1}\left(X\star \psi_{j_{1}}(t)\right) = C_{\sigma_{j_{1}},0}+\overset{\infty}{\underset{\ell=\mathbf{r}}{\sum}}\frac{C_{\sigma_{j_{1}},\ell}}{\sqrt{\ell!}}H_{\ell}(Y_{j_{1}}(t)).
\end{equation}
By (\ref{lamma2:expansion}) and (\ref{expectionhermite}),
\begin{align}\notag
R_{A_{1}(X\star \psi_{j_{1}})}(t) =& \mathbb E\left[\left(A_{1}\left(X\star \psi_{j_{1}}(t)\right)-C_{\sigma_{j_{1}},0}\right)\left(A_{1}\left(X\star \psi_{j_{1}}(0)\right)-C_{\sigma_{j_{1}},0}\right)\right]
\\=& \overset{\infty}{\underset{\ell=\mathbf{r}}{\sum}}C_{\sigma_{j_{1}},\ell}^{2} R_{Y_{j}}^{\ell}(t)
\notag=\overset{\infty}{\underset{\ell=\mathbf{r}}{\sum}}C_{\sigma_{j_{1}},\ell}^{2} \left[\frac{1}{\sigma_{j_{1}}^{2}}R_{X\star\psi_{j_{1}}}(t)\right]^{\ell}.
\end{align}
It implies that
\begin{align}\label{int_power_R}
\int_{\mathbb{R}}R_{A_{1}(X\star \psi_{j_{1}})}(t)dt =& \overset{\infty}{\underset{\ell=\mathbf{r}}{\sum}}C_{\sigma_{j_{1}},\ell}^{2}\sigma_{j_{1}}^{-2\ell}
\int_{\mathbb{R}}\left[R_{X\star\psi_{j_{1}}}(t)\right]^{\ell}dt.
\end{align}
Because $\left[R_{X\star\psi_{j_{1}}}(\cdot)\right]^{\ell}\in L^{1}\cap L^{2}$ for all $\ell\geq \mathbf{r}$, the inverse Fourier transform
implies that
\begin{equation}\label{inv_power_R}
f_{X\star\psi_{j_{1}}}^{\star \ell}(\lambda) = \frac{1}{2\pi}\int_{\mathbb{R}} e^{-i\lambda t} \left[R_{X\star\psi_{j_{1}}}(t)\right]^{\ell}dt
\end{equation}
for all $\lambda\in \mathbb{R}$.
By substituting (\ref{inv_power_R}) with $\lambda=0$ into the right hand side of (\ref{int_power_R}), we achieve the desired equality
\begin{align*}
\int_{\mathbb{R}}R_{A_{1}(X\star \psi_{j_{1}})}(t)dt =2\pi\overset{\infty}{\underset{\ell=\mathbf{r}}{\sum}}C_{\sigma_{j_{1}},\ell}^{2}\sigma_{j_{1}}^{-2\ell}
f_{X\star\psi_{j_{1}}}^{\star\ell}(0).
\end{align*}

\end{proof}

\section{}
\begin{Lemma}\label{lemma:expect_converge}(\cite[Theorem 6.1-6.2]{dasgupta2008asymptotic})
Let $\{X_{k}\}_{k\in \mathbb{N}}$ be a sequence of random variables. If $X_{k}$ converges in distribution to $X_{\infty}$ and
\begin{equation}
\underset{k}{\textup{sup}}\,\mathbb E\left[|X_{k}|^{1+\delta}\right]<\infty\ \textup{for some}\ \delta>0\,,\notag
\end{equation}
then
\begin{equation}
\underset{k\rightarrow\infty}{\textup{lim}}\mathbb E\left[X_{k}\right] = \mathbb E[X_{\infty}]\,.\notag
\end{equation}
\end{Lemma}

\section{}
\begin{Lemma}\label{lemma:slustky}(a slight modification of Slustky's argument)
If the claims (a)-(c) in (\ref{proof:thm1:3claim}) hold, that is,
\begin{equation*}
\begin{array}{ll}
\textup{(a)}\ \underset{N\rightarrow\infty}{\lim}\ \underset{j_{2}\rightarrow\infty}{\lim} \mathbb E[\xi_{j_{2},>N}^2]=0\,;\\
\textup{(b)}\ \textup{The limit of}\ \kappa_{N}\ \textup{exists when}\ N\rightarrow\infty\,;\\
\textup{(c)}\ \textup{For any fixed}\ N\in \mathbb{N}, \xi_{j_{2},\leq N}\ \textup{converges in distribution to}\  \overset{M}{\underset{k=1}{\sum}}a_{k}V_{\leq N}(t_{k})\ \textup{as}\ j_{2}\rightarrow\infty\,,
\end{array}
\end{equation*}
we have $\xi_{j_{2}}$ converges in distribution to $\overset{M}{\underset{k=1}{\sum}}a_{k}V_{1}(t_{k})$ when $j_{2}\rightarrow\infty$,
where $V_{1}$ is the limiting Gaussian process (\ref{thm1_V1}).
\end{Lemma}

\begin{proof}
For any $x\in \mathbb{R}$ and $\varepsilon>0$,
\begin{align}\notag
P(\xi_{j_{2}}\leq x) =& P(\xi_{j_{2}}\leq x,\ |\xi_{j_{2},>N}|>\varepsilon)+P(\xi_{j_{2}}\leq x,\ |\xi_{j_{2},>N}|\leq\varepsilon)\\\label{proof:S:1}
\leq& P(|\xi_{j_{2},>N}|>\varepsilon) + P(\xi_{j_{2},\leq N}\leq x+\varepsilon).
\end{align}
Taking limsup on both sides in the inequality (\ref{proof:S:1}) as $j_{2}\rightarrow\infty$, we get
\begin{align}\notag
\underset{j_{2}\rightarrow\infty}{\overline{\lim}} P(\xi_{j_{2}}\leq x)
\leq& \underset{j_{2}\rightarrow\infty}{\overline{\lim}} P(|\xi_{j_{2},>N}|>\varepsilon) + \underset{j_{2}\rightarrow\infty}{\overline{\lim}}P(\xi_{j_{2},\leq N}\leq x+\varepsilon)
\\\label{proof:S:2}=& \underset{j_{2}\rightarrow\infty}{\overline{\lim}} P(|\xi_{j_{2},>N}|>\varepsilon) + P(\overset{M}{\underset{k=1}{\sum}}a_{k}V_{\leq N}(t_{k})\leq x+\varepsilon),
\end{align}
where the last equality follows from (c).
By (a) and the Markov inequality, $\underset{N\rightarrow\infty}{\lim}\underset{j_{2}\rightarrow\infty}{\overline{\lim}} P(|\xi_{j_{2},>N}|>\varepsilon)=0$.
By (b), we know that the variance of the normal random variable $\overset{M}{\underset{k=1}{\sum}}a_{k}V_{\leq N}(t_{k})$ converges to
that of the normal random variable $\overset{M}{\underset{k=1}{\sum}}a_{k}V_{1}(t_{k})$.
Hence, $\overset{M}{\underset{k=1}{\sum}}a_{k}V_{\leq N}(t_{k})$ converges in distribution to $\overset{M}{\underset{k=1}{\sum}}a_{k}V_{1}(t_{k})$
as $N\rightarrow\infty$. The observations made above lead to
\begin{align}\notag
\underset{j_{2}\rightarrow\infty}{\overline{\lim}} P(\xi_{j_{2}}\leq x)
\leq& \underset{N\rightarrow\infty}{\lim}\underset{j_{2}\rightarrow\infty}{\overline{\lim}} P(|\xi_{j_{2},>N}|>\varepsilon) + \underset{N\rightarrow\infty}{\lim}P(\overset{M}{\underset{k=1}{\sum}}a_{k}V_{\leq N}(t_{k})\leq x+\varepsilon)
\\\label{proof:S:3}=&P(\overset{M}{\underset{k=1}{\sum}}a_{k}V_{1}(t_{k})\leq x+\varepsilon).
\end{align}
On the other hand, for any $x\in \mathbb{R}$ and $\varepsilon>0$,
\begin{align}\notag
P(\xi_{j_{2}}> x) =&\, P(\xi_{j_{2}}>x,\ |\xi_{j_{2},>N}|>\varepsilon)+P(\xi_{j_{2}}> x,\ |\xi_{j_{2},>N}|\leq\varepsilon)\\\notag
\leq&\, P(|\xi_{j_{2},>N}|>\varepsilon) + P(\xi_{j_{2},\leq N}> x-\varepsilon),
\end{align}
which is equivalent to
\begin{align}\label{proof:S:4}
P(\xi_{j_{2}}\leq x)
\geq& -P(|\xi_{j_{2},>N}|>\varepsilon) + P(\xi_{j_{2},\leq N}\leq x-\varepsilon).
\end{align}
Taking liminf on both sides in the inequality (\ref{proof:S:4}) as $j_{2}\rightarrow\infty$, we get
\begin{align}\notag
\underset{j_{2}\rightarrow\infty}{\underline{\lim}} P(\xi_{j_{2}}\leq x)
\geq&\, -\underset{j_{2}\rightarrow\infty}{\overline{\lim}} P(|\xi_{j_{2},>N}|>\varepsilon) + \underset{j_{2}\rightarrow\infty}{\underline{\lim}}P(\xi_{j_{2},\leq N}\leq x-\varepsilon)
\\\label{proof:S:5}=&\, 0 + P(\overset{M}{\underset{k=1}{\sum}}a_{k}V_{\leq N}(t_{k})\leq x-\varepsilon)\,,
\end{align}
where the last equality follows from (a) and (c).
By taking $N\rightarrow\infty$, (b) implies that
\begin{align}\label{proof:S:6}
\underset{j_{2}\rightarrow\infty}{\underline{\lim}} P(\xi_{j_{2}}\leq x)
\geq P(\overset{M}{\underset{k=1}{\sum}}a_{k}V_{1}(t_{k})\leq x-\varepsilon).
\end{align}
Combining (\ref{proof:S:5}) and (\ref{proof:S:6}) yields
\begin{align*}
P(\overset{M}{\underset{k=1}{\sum}}a_{k}V_{1}(t_{k})\leq x-\varepsilon)\leq \underset{j_{2}\rightarrow\infty}{\underline{\lim}} P(\xi_{j_{2}}\leq x) \leq \underset{j_{2}\rightarrow\infty}{\overline{\lim}} P(\xi_{j_{2}}\leq x) \leq P(\overset{M}{\underset{k=1}{\sum}}a_{k}V_{1}(t_{k})\leq x+\varepsilon)
\end{align*}
Because  $\varepsilon$ can be arbitrarily small, the proof of Lemma \ref{lemma:slustky} is complete.
\end{proof}

\end{document}